\documentclass[twoside,11pt]{article}

\usepackage[preprint]{jmlr2e}

\usepackage[utf8]{inputenc} 
\usepackage[T1]{fontenc}    
\usepackage[toc,page,header]{appendix}
\usepackage{minitoc}
\usepackage{url}            
\usepackage{booktabs}       
\usepackage{nicefrac}       
\usepackage{microtype}      
\usepackage[dvipsnames]{xcolor}     
\usepackage{listings}
\usepackage{stmaryrd}
\usepackage{graphicx}
\usepackage{subfigure}
\usepackage{appendix}

\definecolor{codegreen}{rgb}{0,0.6,0}
\definecolor{codegray}{rgb}{0.5,0.5,0.5}
\definecolor{codepurple}{rgb}{0.58,0,0.82}
\definecolor{backcolour}{rgb}{0.95,0.95,0.92}

\usepackage{xr}
\usepackage{booktabs} 
\usepackage{caption}
\usepackage{lipsum}
\usepackage{hyperref}       
\usepackage{url}            

\RequirePackage{amsmath,amsfonts}
\usepackage{dsfont}

\usepackage{algorithm}
\usepackage{algorithmic}
\DeclareCaptionFormat{myformat}{#3}
\def\rset{\mathbb{R}}

\def\nset{\mathbb{N}}
\def\mcf{\mathcal{F}}
\def\argmax{\mathop{\rm arg\, max}}

\newtheorem{my_definition}{Definition}

\newtheorem{my_remark}{Remark}
\newtheorem{my_proposition}{Proposition}
\newtheorem{assumption}{Assumption}

\DeclareMathOperator{\Diag}{Diag}

\DeclareMathOperator{\spann}{Span}
\DeclareMathOperator{\prob}{\mathbb{P}}
\DeclareMathOperator{\expec}{\mathbb{E}}

\newcommand{\qed}{$\blacksquare$}

\newcommand{\mdf}[1]{\bgroup\color{Sepia}{#1}\egroup}

\usepackage{lastpage}
\jmlrheading{23}{2022}{1-\pageref{LastPage}}{10/21; Revised 6/22}{10/22}{21-1240}{R\'{e}mi Leluc and Fran\c{c}ois Portier}

\ShortHeadings{SGD with Coordinate Sampling: Theory and Practice}{Leluc and Portier}
\firstpageno{1}

\begin{document}

\doparttoc 
\faketableofcontents 


\title{SGD with Coordinate Sampling: Theory and Practice}

\author{\name R\'{e}mi Leluc \email remi.leluc@gmail.com \\
       \addr Department of Statistics, LTCI\\
       Télécom Paris, Institut Polytechnique de Paris \\
       91120 Palaiseau, France
       \AND 
       \name Fran\c{c}ois Portier \email francois.portier@gmail.com \\
       \addr Department of Statistics, CREST\\
       Ecole Nationale de la Statistique et de l'Analyse de l'Information (ENSAI)\\
       35170 Bruz, France}

\editor{Karthik Sridharan}

\maketitle

\begin{abstract}
While classical forms of stochastic gradient descent algorithm treat the different coordinates in the same way, a framework allowing for adaptive (\textit{non uniform}) coordinate sampling is developed to leverage structure in data. In a non-convex setting and including zeroth-order gradient estimate, almost sure convergence as well as non-asymptotic bounds are established. Within the proposed framework, we develop an algorithm, MUSKETEER, based on a reinforcement strategy: after collecting information on the noisy gradients, it samples the most promising coordinate (\textit{all for one}); then it moves along the one direction yielding an important decrease of the objective (\textit{one for all}). Numerical experiments on both synthetic and real data examples confirm the effectiveness of MUSKETEER in large scale problems.
\end{abstract}

\begin{keywords}
stochastic optimization, stochastic gradient algorithms, zeroth-order optimization, coordinate descent, adaptive methods.
\end{keywords}

\section{Introduction}
 
Coordinate Descent (CD) algorithms have become unavoidable in modern machine learning because they are tractable \citep{nesterov2012efficiency} and competitive to other methods when dealing with key problems such as support vector machines, logistic regression, LASSO regression and other $\ell_1$-regularized learning problems \citep{wu2008coordinate,friedman2010regularization}. They are applied in a wide variety of problems ranging from linear systems \citep{lee2013efficient,beck2013convergence} to finite sum optimization \citep{necoara2014random,lu2015complexity} and composite functions \citep{richtarik2014iteration} with parallel \citep{fercoq2015accelerated,richtarik2016parallel}, distributed \citep{fercoq2014fast,qu2015quartz} and dual \citep{shalev2013stochastic,csiba2015stochastic,perekrestenko2017faster} variants. In many contributions \citep{loshchilov2011adaptive,richtarik2016optimal,glasmachers2013accelerated,qu2016coordinate,allen2016even,namkoong2017adaptive}, 
the choice of the coordinate sampling policy is conducted through some optimality criterion estimated along the algorithm.
On the one hand, efficient forms of CD methods rely on a deterministic procedure \citep{nutini2015coordinate} which adapts to the underlying structure in data at the expense of higher calculation and thus, may be costly. On the other hand, stochastic gradient descent (SGD) methods are computationally efficient but often treat all coordinates equally and thus, may be sub-optimal. In the spirit of adaptive schemes, we tend to bridge the gap between the best of both worlds by developing, within a noisy gradient framework, a general stochastic coordinate descent method with a particular selection strategy. 

We are interested in solving unconstrained optimization problems of the form $\min_{\theta \in \rset^p} f(\theta)$, where the objective function $f$ may be either known exactly or accessed through noisy observations. When $f$ is differentiable, a common appproach is to rely on the gradient of $f$. However, in many scenarios and particularly in large-scale learning, the gradient may be hard to evaluate or even intractable. Hence, one usually approximates the gradient using zeroth or first order estimates \citep{ghadimi2013stochastic,lian2016comprehensive}. The former constructs pseudo-gradients by sampling some perturbed points or using finite differences \citep{flaxman2005,duchi2012randomized,nesterov2017random,shamir2017optimal} (see \citet{liu2020primer} for a recent survey and numerous references) leading to biased gradient estimates while the latter often relies on data sampling techniques \citep{needell2014stochastic,papa2015adaptive} to obtain unbiased gradient estimates. 
In both cases, a random gradient estimate is available at a cheap computing cost and the method consists in moving along this estimate at each iteration. Early seminal works on such stochastic algorithms include \citet{robbins1951stochastic,kiefer1952stochastic} and a recent review dealing with large scale learning problems is given in \citet{bottou2018optimization}. 

Starting from an initial point $\theta_0 \in \rset^p$, the SGD algorithm is defined by the update rule
\begin{equation*} 
\forall t \geq 0, \quad \theta_{t+1} = \theta_{t} - \gamma_{t+1} g_t
\end{equation*}
where $g_t\in \mathbb R^p $ is a gradient estimate at $\theta_t$ (possibly biased) and $(\gamma_t)_{t\geq 1}$ is some learning rate sequence that should decrease throughout the algorithm. While the computation of $g_t$ may be cheap, it still requires the computation of a vector of size $p$ which may be a critical issue in high-dimensional problems. To address this difficulty, we rely on sampling well-chosen coordinates of the gradient estimate at each iteration. 

We consider the framework of stochastic coordinate gradient descent (SCGD) which modifies standard stochastic gradient descent methods by adding a selection step to perform random coordinate descent. The SCGD algorithm is defined by the following iteration
\begin{align*}
\left\{\begin{array}{ll}
\theta_{t+1}^{(k)} = \theta_{t}^{(k)} &\text { if } k \neq \zeta_{t+1} \\
\theta_{t+1}^{(k)} = \theta_{t}^{(k)} - \gamma_{t+1}  g_t^{(k)} &\text{ if } k=\zeta_{t+1} 
\end{array}\right.
\end{align*} 
where $ \zeta_{t+1}$ is a random variable valued in $\llbracket 1,p \rrbracket$ which selects a coordinate of the gradient estimate. The distribution of $\zeta_t$ is called the \textit{coordinate sampling policy}. Note that the SCGD framework is very general as it contains as many methods as there are ways to generate both the gradient estimate $g_t$ and the random variables $\zeta_t$. 

\medskip
\noindent \textbf{Contributions.} The main contributions are as follows

\textit{(i)(Theory)} 
We show the almost-sure convergence of the SCGD iterates $(\theta_t)_{t \in \nset}$ towards stationary points in the sense that $\nabla f(\theta_t) \to 0$ almost surely as well as non-asymptotic bounds on the optimality gap $\expec[f(\theta_t)-f^\star]$ where $f^\star$ is a lower bound of $f$. The working conditions are relatively weak as the function $f$ is only required to be $L$-smooth (classical in non-convex problems) and the stochastic gradients are possibly biased  with unbounded variance, using a growth condition related to \textit{expected smoothness} \citep{gower2019sgd}. 

 \textit{(ii)(Practice)} We develop  a new algorithm, called MUSKETEER, for \textit{MUltivariate Stochastic Knowledge Extraction Through Exploration Exploitation Reinforcement}. In the image of the motto 'all for one and one for all', this procedure belongs to the SCGD framework with a particular design for the \textit{coordinate sampling policy}. It compares the value of all past gradient estimates $g_t$ to select a descent direction (\textit{all for one}) and then moves the current iterate according to the chosen direction (\textit{one for all}). The heuristic is the one of reinforcement learning in the sense that large gradient coordinates represent large decrease of the objective and can be seen as high rewards. The resulting directions should be favored compared to the path associated to small gradient coordinates. By updating the \textit{coordinate sampling policy}, the algorithm is able to detect when a direction becomes rewarding and when another one stops being engaging. 

\medskip
\noindent \textbf{Related work.} 
The authors of \citep{nutini2015coordinate} investigate the deterministic Gauss-Southwell rule which consists of picking the coordinate with maximum gradient value. In trusting large gradients, this rule looks like the one of MUSKETEER except that no stochastic noise -neither in the gradient evaluation nor in the coordinate selection- is present in their algorithm. In that aspect, our method differs from all the previous CD studies \citep{loshchilov2011adaptive,richtarik2016optimal,glasmachers2013accelerated,qu2016coordinate,allen2016even,namkoong2017adaptive}
 which rely on $\nabla f$. Among the SGD literature, compression and sparsification methods \citep{alistarh2017qsgd,wangni2018gradient} were developed  for communication efficiency. The former use compression operators to select a few components of the gradient estimates at the cost of full gradient computation and coordinate sorting. The latter use a gradient estimate $g$ which is sparsified using probability weights to reach an unbiased estimate of the gradient. In contrast, the SCGD framework allows the gradient to be biased as no importance re-weighting is performed. Note also that, to cover zeroth-order methods, the gradient estimate itself $g_t$ is allowed to be biased as for instance in the recent study of \citet{ajalloeian2020analysis}. The proofs of the asymptotic convergence results are based on ideas from \cite{bertsekas2000gradient} with particular extensions in the framework of biased gradient estimates. Finally, the non-asymptotic bounds are inspired from \citet{moulines2011non} where the authors provide a non-asymptotic analysis for standard SGD.

\medskip
\noindent \textbf{Outline.} Section \ref{sec:math_background} introduces the mathematical framework with the different sampling strategies and Section \ref{sec:main_results} contains our main theoretical results. Section \ref{sec:main_algo} is dedicated to MUSKETEER algorithm and a numerical analysis is performed in Section \ref{sec:simus}. Proofs, technical details and additional experiments may be found in the appendix.

\section{Mathematical Background}\label{sec:math_background}

\subsection{Notation and problem set-up}

\textbf{Notation.} Denote by $(e_1,\ldots,e_p)$ the canonical basis of $\rset^p$ and for $k \in \llbracket 1,p \rrbracket$, $D(k)=e_k^{} e_k^{T} \in \{0,1\}^{p \times p}$ is a diagonal matrix with a $1$ in position $k$. $\|\cdot\|_2$ and $\|\cdot\|_\infty$ are respectively the Euclidian and infinity norm. For any $u \in \rset^p$, $u^{(k)}$ is the k-th coordinate of $u$; $\mathds{1}_A$ is the indicator function of the event $A$, \textit{i.e.}, $\mathds{1}_A=1$ is $A$ is true and $\mathds{1}_A=0$ otherwise. Denote by $\mathcal{U(}\llbracket 1,p \rrbracket)$ the uniform distribution over $\llbracket 1,p \rrbracket$. For a vector of probability weights $d=(d^{(1)},\ldots,d^{(p)})$ with $\sum_{k=1}^p d^{(k)}=1$, denote by $Q(d)$ the associated categorical distribution.

\medskip
\noindent \textbf{Problem set-up.} Consider the classical stochastic optimization problem 
\begin{align*}
\min_{\theta \in \mathbb R^p} \left\{f(\theta) = \expec_{\xi}[f(\theta,\xi)] \right\},
\end{align*}
where $\xi$ is a random variable. In many scenarios, \textit{e.g.} empirical risk minimization or reinforcement learning, the gradient $\nabla f$ cannot be computed in a reasonable time and only a stochastic version, possibly biased, is available. 
The distribution of $\xi$ is called the \textit{data sampling policy} as it refers to the sampling mechanism in the empirical risk minimization (ERM) framework.
This running example is presented below and shall be considered throughout the paper. 
 Other classical optimization problems where stochastic gradients are available include adaptive importance sampling \citep{delyon+p:2018}, policy gradient methods \citep{hanna2019importance} and optimal transport \citep{genevay2016stochastic}.

\medskip
\noindent \textbf{Running Example (ERM).} Given some observed data $z_1 ,\ldots,z_n\subset \mathcal Z$ and a loss function $\ell: \rset^p \times \mathcal Z \to \rset$, the objective function $f$ approximates the risk $\expec_{z}[\ell(\theta,z)]$ by the so-called empirical risk defined as
$$ \forall \theta \in \rset^p, \quad  f(\theta) =  \frac{1}{n} \sum_{i=1}^n \ell(\theta,z_i). $$
Evaluating $f$ or its gradient is prohibitive in large scale machine learning as it requires seeing all the samples in the dataset. 
Instead, after picking at random an index $ j = \xi$, uniformly distributed over $\llbracket 1, n\rrbracket$, the $k$-th coordinate of the gradient estimate may be computed as $  (\ell(\theta + h e_k,z_j)-\ell(\theta,z_j)) / {h}$. When differentiation is possible, another gradient estimate is offered by $\nabla_\theta  \ell(\theta , z_j) $. These two gradient estimates are of a different nature: the first one, often referred to as zeroth-order estimate, is biased whereas the second one, often referred to as first order estimate, is unbiased.

\subsection{Gradient estimates}
Throughout the paper, the gradient generator is denoted by $g_h(\cdot, \xi)$ where the parameter $h\geq 0$ represents the underlying bias as claimed in the next assumption. This level of generality allows to include zeroth-order estimate as discussed right after the assumption.

\begin{assumption} [Biased gradient]\label{cond:biased_grad}
There exists a constant $c\geq 0$ such that: $$\forall h>0 ,\, \forall \theta\in \mathbb R^p,\quad  \|\expec_{\xi}[ g_h(\theta,\xi ) ] - \nabla f(\theta)\|_2 \leq c h.$$ 
\end{assumption}

\noindent This general assumption enables to work with classical unbiased gradient in the framework of first order estimates by taking $c=0$. Furthermore, Assumption \ref{cond:biased_grad} is satisfied for the following well-spread zeroth-order estimates.

\medskip
\noindent \textbf{Example 1 (smoothing).} 
The smoothed gradient estimate \citep{nesterov2017random} is given for all $\theta\in \mathbb R^p$ by $g_{h}(\theta,\xi) = h^{-1} [f(\theta + h U,\xi)-f(\theta,\xi)]U$ where $U $ is a standard Gaussian vector (independent from $\xi$). An alternative version consists in taking  $U$ uniformly distributed over the unit sphere. 

\medskip
\noindent \textbf{Example 2 (finite differences).} The finite differences gradient estimate is given for all $\theta\in \mathbb R^p$ by $g_{h}(\theta,\xi) = \sum_{k=1}^p g_{h}(\theta,\xi)^{(k)}e_k$ where for all $k=1,\ldots,p$ the coordinates are $g_{h}(\theta,\xi)^{(k)} = h^{-1}[f(\theta + h e_k,\xi)-f(\theta,\xi)] $.

\medskip
\noindent Both previous examples share the following general property. There exists a probability measure $\nu$ satisfying $\int_{\rset^p} x x^\top \nu (\mathrm{d}x) = I_p$ such that, 
\begin{align}\label{cond:grad_NEW2}
\forall h>0,\theta\in \mathbb R^p, \quad \expec_{\xi}[ g_h (\theta,\xi ) ] = \int_{\rset^p} x   \left\{  \frac{ f(\theta + h x ) - f(\theta ) } {h}   \right\}   \nu (\mathrm{d}x).
\end{align}
The smoothed gradient estimate is recovered when $\nu $ is the standard Gaussian measure and taking $\nu = \sum_{k=1}^p \delta_{e_k}/p $ covers the finite differences estimate. As detailed in the next subsection, an interesting framework is to use a measure $\nu$ that evolves through time and put different weights on the different directions. As stated in the following proposition, when the function $f$ is $L$-smooth, i.e., $\nabla f$ is $L$-Lipschitz, the bias of the gradient estimate \eqref{cond:grad_NEW2} is of order $ h$ and thus satisfies Assumption \ref{cond:biased_grad}. 

\begin{my_proposition}\label{useful_bound}
Under Eq. \eqref{cond:grad_NEW2}, if $f$ is $L$-smooth, then Assumption \ref{cond:biased_grad} holds true with 
$c =  \sqrt C L / 2$ where $ C = \int_{\rset^p} \|x\|_2 ^6 \nu(\mathrm{d}x) <\infty $.
\end{my_proposition}

The previous proposition allows us to cover the two methods: smoothing and finite difference. Note that for the latter, the constant $C$ is equal to $1$. 

\subsection{Coordinate Sampling Policy}\label{sec:policy}

Let $(\xi_t)_{t\geq 1} $ be a sequence of independent and identically distributed random variables. Let $(\gamma_t)_{t\geq 1} $ be a sequence of positive numbers called \textit{learning rates}. Let $(h_t)_{t\geq 1} $ be a sequence of positive numbers called \textit{smoothing parameters}. Denote by $ g_t = g _{h_{t+1} }  ( \theta_t,  \xi_{t+1} )$ the gradient estimate at time $t$. The classical SGD update rule is given by
\begin{align}\label{eq:sgd_iteration}
\theta_{t+1}  =  \theta_{t}  - \gamma_{t+1} g_t  ,\quad t\geq 0,
\end{align}
For any $t \in \nset, \mcf_t   = \sigma( \theta_0, \theta_1,\ldots, \theta_t)$ is the $\sigma$-field associated to the sequence of iterates $(\theta_t)_{t \in \nset}$.

The framework of SCGD is introduced thanks to random coordinate sampling. At each step, only one coordinate of the parameter of interest is updated. This coordinate is selected at random according to a distribution valued in $\llbracket 1,p \rrbracket$ which is allowed to evolve during the algorithm. The iteration of the coordinate sampling algorithm is given coordinate-wise by
\begin{align}\label{eq:cond_sgd}
\left\{\begin{array}{ll}
\theta_{t+1}^{(k)} = \theta_{t}^{(k)} &\text { if } k \neq \zeta_{t+1} \\
\theta_{t+1}^{(k)} = \theta_{t}^{(k)} - \gamma_{t+1}  g_t^{(k)} &\text{ if } k=\zeta_{t+1} 
\end{array}\right.
\end{align} 
where $ \zeta_{t+1}$ is a random variable valued in $\llbracket 1,p \rrbracket$. Hence $\zeta_{t+1}$ selects the coordinate along which the $t$-th descent shall proceed.  The distribution of $\zeta_{t+1}$  is called the \textit{coordinate sampling policy} as opposed to the \textit{data sampling policy} governed by the random variable $\xi_{t+1}$. The distribution of $\zeta_{t+1}$ is characterized by the probability weights vector $d_{t} = (d_{t}^{(1)},\ldots,d_{t}^{(p)})$ defined by 
\begin{equation*}
d_{t}^{(k)}  = \mathbb P ( \zeta_{t+1} = k |\mcf_{t}), \quad k\in \llbracket 1,p \rrbracket.
\end{equation*}
The categorical distribution on $\llbracket 1,p \rrbracket$ associated to $d_t$ is denoted by $Q(d_t)$, \textit{i.e.}, conditionally to $\mathcal F_t$, we have: $$\forall t \geq 0, \quad  \zeta_{t+1}  \sim Q (d_t) \quad \text{ with } \quad d_{t} = (d_{t}^{(1)},\ldots,d_{t}^{(p)}).$$

\noindent \textbf{Running Example (ERM).} The CD algorithm defined by Equation \eqref{eq:cond_sgd} can easily be applied in the ERM framework. The \textit{coordinate sampling} strategy $\zeta \sim Q(d_t)$ combined with the uniform \textit{data sampling} $\xi \sim \mathcal{U}(\llbracket 1,n \rrbracket)$ leads to $\theta_{t+1}^{(\zeta)}  = \theta_{t}^{(\zeta)} - (\gamma_{t+1}  / h_{t+1} )   (\ell(\theta_t + h_{t+1} e_\zeta , z_\xi)-\ell(\theta_t  , z_\xi))  $ (zeroth-order) and $\theta_{t+1}^{(\zeta)}  = \theta_{t}^{(\zeta)} - \gamma_{t+1}  \partial_{\theta_{\zeta}}  \ell(\theta_t, z_\xi)$ (first order).

Given the past, the \textit{data sampling} and \textit{coordinate sampling} draws should not be related.

\begin{assumption}[Conditional Independence]\label{ass:indep_random}
$\zeta_{t+1}$ is independent from $\xi_{t+1} $ conditionally on $\mathcal F_t$.
\end{assumption}

This assumption is natural in the  ERM context as in most cases there is no particular link between the sample indexes and the coordinates. Futhermore, the independence property plays an important role in our proofs. The SCGD algorithm defined in \eqref{eq:cond_sgd} is simply written with matrix notation as
\begin{equation*} 
\theta _ {t+1}  = \theta _ {t} -  \gamma_{t+1}  D(\zeta_{t+1})  g _t,
\end{equation*} 
where $D(k) = e_{k}^{} e_{k}^\top \in \mathbb R^{p\times p}$ has its entries equal to $0$ except the $(k,k)$ which is $1$. Observe that the distribution of the random matrix $D(\zeta_{t+1})$  is fully characterized by the matrix 
\begin{align*}
D_t = \expec [ D(\zeta_{t+1})|\mathcal F_{t}] = \Diag (d_{t}^{(1)},\ldots,d_{t}^{(p)}).
\end{align*}
Note that under Assumptions \ref{cond:biased_grad} and \ref{ass:indep_random}, the average move of SCGD follows a biased gradient direction. For instance, when $c=0$, the average move of SCGD is given by $\expec [\theta _ {t+1}  - \theta _ {t}  |\mathcal F_t] = -\gamma_{t+1} D_t \nabla f (\theta_t ) $ which bears resemblance to the Conditioned-SGD iteration  \citep[Section 6.2]{bottou2018optimization}. Such preprocessing is meant to refine the gradient direction through a matrix mulitplication for a better understanding of the underlying structure of the data. A natural question rises on the choice of the matrix $D_t$ among all the possible coordinate sampling distributions. 

The SCGD framework is efficient as soon as one can compute each coordinate of the gradient estimate. This is the case for zeroth-order (ZO) optimization with finite differences where the full gradient estimate uses $p$ partial derivatives, each of them requiring two queries of the objective function. SCGD reduces this cost to a single coordinate update.

\begin{my_remark}[Batch coordinates] A natural extension is to consider subsets of coordinates, \textit{a.k.a.} block-coordinate descent. Note that this framework is covered by our approach as the proofs can be extended by summing different matrices $D(\zeta)$. Similarly to mini-batching \citep{gower2019sgd}, one can consider multiple draws for the coordinates that are to be updated. The selecting random matrix $D(\zeta_{t+1})$ may be replaced by a diagonal matrix with $m(<p)$ non-zero coefficients. For that matter, it is enough to have multiple draws from the categorical distribution $Q(d_t)$.
\end{my_remark}

\begin{my_remark}[Parallelization]
Several families of communication-reduction methods such as quantization \citep{alistarh2017qsgd}, gradient sparsification \citep{wangni2018gradient,alistarh2018convergence} or local-SGD \citep{patel2019communication} have been proposed to reduce the overheads of distribution. The SCGD framework can benefit from such data parallelization techniques. When a fixed number $m$ of machines is available, it is then possible to gain computational acceleration by drawing $m$ times the coordinate distribution $Q(d_t)$ on the different machines and then transmit the batch of selected coordinates to the workers.
\end{my_remark}

\subsection{Adaptive and Unbiased Policies}

To understand more clearly the differences between SGD and SCGD, we shall rely on a more general iteration scheme. This framework is useful to compare different algorithms in terms of adaptive policies and unbiased estimates. Consider the following general update rule
\begin{equation} \label{eq:general_rule}
\theta _ {t+1}  = \theta _ {t} -  \gamma_{t+1} h  ( \theta_t,  \omega_{t+1} ), \quad t \geq 0
\end{equation} 
where $h$ is a gradient generator and $(\omega_t)_{t\geq 1}$ is a sequence of random variables which are not necessarily independent nor identically distributed. Observe that both frameworks, SGD and SCGD, are instances of \eqref{eq:general_rule}. For example, the randomness of SCGD can be expressed through $\omega_t = (\xi_t,\zeta_t)$.

\begin{my_definition}[Policy]
Denote by $P_t$ the distribution of $\omega_{t+1}$ given $\mathcal F_t$. The sequence $(P_{t})_{t \geq 0} $ is  called the policy of the stochastic algorithm. 
\end{my_definition}

The policy of a stochastic algorithm is an important tool as it determines the randomness introduced over time. On the one hand, it provides insights on the expected behavior of the algorithm. On the other hand, it measures the ability to adapt through the iterations.

\begin{my_definition}[Unbiased and Adaptive] A policy $(P_{t})_{t \geq 0} $ is called "unbiased" if: $\forall \theta \in  \rset^p,t\geq 0$, $\int  h ( \theta,  \omega  ) P_t(\mathrm{d} \omega) \propto \nabla f (\theta) $. It is called "naive" if $P_t$ does not change with $t$, otherwise it is adaptive.
\end{my_definition}

With these definitions in mind, it is clear that the SGD policy \eqref{eq:sgd_iteration} under Assumption \ref{cond:biased_grad}  with $c=0$ is unbiased and naive, and so does the policy induced by first order gradient in ERM.

Within the framework of SCGD, a policy cannot be unbiased and adaptive as claimed in the next proposition.
 
\begin{my_proposition}[Unbiased coordinate policy] \label{prop:unb_pol}
Suppose that Assumption \ref{cond:biased_grad} is fulfilled with $c=0$ and that $ {\spann} \{\nabla f(\theta) \, : \, \theta \in \mathbb R^p \} $ is dense in $\mathbb R^p $, then the only unbiased coordinate sampling policy is $D_t = I_p / p$. It corresponds to uniform coordinate sampling. 
\end{my_proposition}
When working under Assumption \ref{cond:biased_grad} with $c=0$, SCGD with uniform coordinate sampling is unbiased and hence similar to SGD. This is confirmed in the numerical experiments (Appendix \ref{sec:simus_first} and \ref{sec:more_num}). 
However, a uniform sampling does not use any available information to favor coordinates among others. Thus, the approach promoted in the paper is different: past gradient values are used to update the probability weights of $D_t$. The resulting method is an adaptive algorithm which is biased.

\begin{my_remark}[Importance Coordinate Sampling]\label{rk:IS}
Note that the general framework defined above includes the particular case where the coordinates are selected according to $\zeta$ then reweighted as proposed in \citep{wangni2018gradient}. This corresponds to the choice $h(\theta, \omega_{t+1} ) = D^{-1}_t D(\zeta_{t+1}) g(\theta , \xi_{t+1}).$ Even though such a policy is adaptive and unbiased, it turns out -from our numerical experiments (Appendix \ref{sec:more_zo_simus})- that it behaves similarly to the uniform version and is therefore sub-optimal.
\end{my_remark}

\section{Main Theoretical Results} \label{sec:main_results}

In a general non-convex setting, we investigate the almost sure convergence of SCGD algorithms as well as non-asymptotic bounds. The following assumptions on the objective function $f$ are classical among the SGD literature.

\begin{assumption}[Lower bound] \label{ass:lower_bound} There exists $f^\star \in \rset$ such that: $\forall \theta \in \rset^p, f(\theta) \geq f^\star$.
\end{assumption}

\begin{assumption}[Smoothness] \label{ass:smooth} The objective $f:\mathbb{R}^{p} \rightarrow \mathbb{R}$ is twicely continuously differentiable and $L$-smooth: $\forall \theta,\eta \in \rset^p, \quad \|\nabla f(\theta) - \nabla f(\eta)\|_2 \leq L \|\theta - \eta\|_2$.
\end{assumption}

\begin{my_remark}[Coordinate smoothness]\label{rem:coord_smooth}
Note that this assumption may be refined using the notion of coordinate smoothness with parameters $(L_1,\ldots,L_p)$ where for all $k=1,\ldots,p$, $\partial_k f(\cdot)$ is $L_k$-Lipschitz, \textit{i.e.}, for all $\theta \in \rset^p, \delta \in \rset, |\partial_k f(\theta + \delta e_k)-\partial_k f(\theta)| \leq L_k |\delta|$. Within this framework, small values of $L_k$ are associated to a high degree of smoothness in the $k$-th direction. Conversely, large values of $L_k$ are associated to more difficult minimization problems along that direction. Intuitively, it requires more energy to minimize $f$ along these directions and one should assign more sampling probability on coordinates with larger $L_k$ (see Proposition \ref{prop:regret} in the appendix).
\end{my_remark}

When dealing with stochastic algorithms, the stochastic noise associated to the gradient estimates is the keystone for the theoretical analysis. To treat this term, we consider a weak growth condition, related to the notion of \textit{expected smoothness} as introduced in \citet{gower2019sgd} (see also \citet{gazagnadou2019optimal,gower2021stochastic}). 

\begin{assumption}[Growth condition] \label{ass:exp_smooth} With probability $1$, there exist  $0 \leq \mathcal{L},\sigma^2 < \infty$ such that for all $\theta \in \rset^p$ and $h> 0$, we have: $\expec\left[\|  g_h(\theta,\xi)\|_{2}^2 \right] \leq 2 \mathcal{L} \left( f(\theta) - f^\star\right) + \sigma^2.$
\end{assumption}

This bound on the stochastic noise $ \expec\left[ \| g (\theta,\xi)  \|_{2}^2  \right] $ is the key to prove the almost sure convergence of the algorithm.  Note that Assumption \ref{ass:exp_smooth} is weak as it allows the noise to be large when the iterate is far away from the optimal point. In that aspect, it contrasts with uniform bounds of the form  $\expec\left[\| g(\theta ,\xi)\|_2^2 \right] \leq \sigma^2$ for some deterministic $\sigma^2 >0$ \citep{nemirovsky1983problem,nemirovski2009robust,shalev2011pegasos}. Observe that such uniform bound is recovered by taking $\mathcal{L}=0$ in Assumption \ref{ass:exp_smooth} but cannot hold when the objective function $f$ is strongly convex \citep{nguyen2018sgd}. The standard Robbins-Monro condition,  $\sum_{t \geq 1} \gamma_{t} = +\infty$ and $\sum_{t \geq 1} \gamma_t ^2 < +\infty$ is required in the next theorem which serves as a starting point for a comparison between SGD and SCGD methods.

\begin{theorem}[Almost sure convergence of biased SGD] \label{th:convergence_as_sgd} 
Suppose that Assumptions \ref{cond:grad_NEW2} to \ref{ass:exp_smooth} are fulfilled and let $(\theta_t)_{t \in \nset}$ be the sequence of iterates defined by \eqref{eq:sgd_iteration}. If the learning rates satisfy the Robbins-Monro condition and $h_t^2 = O(  \gamma_t) $ then $\nabla f(\theta_t) \to 0$ a.s. when $t \to +\infty$.
\end{theorem}

The SCGD framework is very general in the sense that it covers as many algorithms as there are ways to generate both the gradient estimate $g_t$ and the random variables $\zeta_t$ that select the coordinates. The next theorem provides the almost sure convergence of particular instances of SCGD algorithms where the true gradient is known and used to define the \textit{coordinate sampling} policy. It recovers the deterministic Gauss-Southwell rule \citep{nutini2015coordinate} and extends it to the case where the coordinate weights are proportional to any norm of the current gradient $\nabla f(\theta_t)$. 

\begin{theorem}[Almost sure convergence of particular SCGD] \label{th:convergence_as_csgd_grad} 
Suppose that Assumptions \ref{cond:grad_NEW2} to \ref{ass:exp_smooth} are fulfilled  and let $(\theta_t)_{t \in \nset}$ be the sequence of iterates defined by \eqref{eq:cond_sgd}, i.e., $\theta _ {t+1}  = \theta _ {t} -  \gamma_{t+1}  D(\zeta_{t+1})  g _t $.  If the learning rates satisfy the standard Robbins-Monro and $h_t^2 = O(  \gamma_t) $, then the two following results hold: 
\begin{itemize}
\item (a) (maximum gradient) if the selected coordinate follows the maximum coordinate of the gradient $\zeta_{t+1} = \argmax_{k=1,\ldots,p} |\partial_k f(\theta_t)|$ then $\nabla f(\theta_t) \to 0$ almost surely as $t \to +\infty$.
\item (b) (gradient weights) if the selection weights are proportional to the gradient norm $D_t \propto (|\nabla_k f(\theta_t)|^q)_{1 \leq k \leq p}$ with $q>0$ then $\nabla f(\theta_t) \to 0$ almost surely as $t \to +\infty$.
\end{itemize}
\end{theorem}

\medskip
\begin{my_remark}[Sparse Gradient]\label{rem:sparsity_func} In light of the sparsity assumption used in \citet{pmlr-v84-wang18e}(Assumption A5), note that SCGD methods with weights proportional to the gradient coordinates can outperform uniform coordinate sampling as they only select the relevant directions throughout the procedure. Such sparsity framework happens for instance in hyper-parameter tuning problems of learning systems: usually the performance of the system is insensitive to some hyper-parameters which implies the sparsity of the gradients.
\end{my_remark}

In the general case, one may not have access to the true gradient and can only rely on the estimate $g_t$. Another assumption is therefore needed on the weights of the \textit{coordinate sampling} policy to ensure that all the coordinates of interest are selected throughout the algorithm. The success of the proposed approach relies on the following restrictions between the \textit{learning rates} sequence $(\gamma_t)_{t \in \nset}$ and the weights of the \textit{coordinate policy}. This is formally stated in the following assumption, referred to as the extended Robbins-Monro condition. Denote by $\beta_{t+1}$ the smallest probability weight at time $t$, \textit{i.e.}, $\beta_{t+1} = \min_{1 \leq k \leq p}  d_{t}^{(k)} .$

\begin{assumption}[Extended Robbins-Monro condition] \label{ass:lr} $(\gamma_t)_{t\geq 1}$, $(\beta_t)_{t\geq 1}$ are positive sequences such that 
$ \sum_{t \geq 1} \gamma_{t}  \beta_{t}  = +\infty$ and $ \sum_{t \geq 1} \gamma_{t}^2 < +\infty.$
\end{assumption}
From a practical point of view, those are not restrictive as they can always be implemented by the user. In the case $D_t = I_p$,  this is simply the standard Robbins-Monro condition.

\begin{theorem}[Almost sure convergence of general SCGD] \label{th:convergence_as_csgd} 
Suppose that Assumptions \ref{cond:grad_NEW2} to \ref{ass:exp_smooth} are fulfilled  and let $(\theta_t)_{t \in \nset}$ be the sequence of iterates defined by \eqref{eq:cond_sgd}. Assume moreover that the learning rates satisfy Assumption \ref{ass:lr}, $h_t^2 = O(  \gamma_t) $ and that 
$(\beta_t)$ has a positive lower bound, then $\nabla f(\theta_t) \to 0$ almost surely as $t \to +\infty$.
\end{theorem}

\begin{my_remark}[Global convergence] \label{rem:global}
Other convergence results concerning the sequence of iterates towards global minimizers may be obtained by considering stronger assumptions including that $f$ is coercive and the level sets of stationary points $\{\theta,\nabla f(\theta)=0\} \cap \{\theta,f(\theta)=y\}$ are locally finite for every $y \in \rset^d$ (see \citet{gadat2018stochastic} or Appendix \ref{subsec:cv_stronger_assumptions}).
\end{my_remark}

For a non-asymptotic analysis, we place ourselves under the Polyak–Łojasiewicz (PL) condition \citep{polyak1963gradient} which does not assume convexity of $f$ but retains many properties of strong convexity, \textit{e.g.} the fact that every stationary point is a global minimum.
\begin{assumption}[PL inequality] \label{ass:pl_ineq} There exists a constant $\mu >0$ such that:
\begin{align*}
\forall \theta \in \rset^p, \| \nabla f(\theta) \|_2^2 \geq 2\mu\left( f(\theta) - f^\star \right).
\end{align*}
\end{assumption}

Similarly to \citep{moulines2011non}, we introduce $\varphi_\alpha:\rset_+^\star \to \rset, \varphi_\alpha(t) = \alpha^{-1}(t^\alpha - 1)$ if $\alpha \neq 0$ and $\varphi_\alpha(t) =\log(t)$ if $\alpha=0$. Denoting $\delta_t = \expec\left [ f(\theta_t) - f^\star \right]$ and assuming that $\beta_{t+1} \geq \beta >0$, one can obtain the recursion equation: $\delta_t \leq \left( 1 - 2\mu \beta \gamma_t + L \mathcal{L} \gamma_t^2 \right) \delta_{t-1} + \gamma_{t}^2 (\sigma^2 L + c^2)/2$, leading to the following theorem on non-asymptotic bounds for SCGD methods.

\begin{theorem}[Non-asymptotic bounds] \label{th:non_as_bound}
Suppose that Assumptions \ref{cond:biased_grad} to \ref{ass:pl_ineq} are fulfilled and let $(\theta_t)_{t \in \nset}$ defined in \eqref{eq:cond_sgd} with $\gamma_t = \gamma t^{-\alpha}$ and $h_t = \sqrt{\gamma_t}$. Denote by $\delta_t = \expec\left [ f(\theta_t) - f^\star \right]$ and assume that there exists $\beta >0$ such that $\beta_{t+1} \geq \beta >0$. We have for $\alpha \in [0,1]$: \\
\textbullet \ If $0 \leq \alpha < 1$ then 
\begin{align*}
\delta_t \leq 2 \exp\left( 2 L \mathcal{L} \gamma^2 \varphi_{1-2\alpha}(t)\right) \exp\left(-\frac{\mu \beta \gamma}{4} t^{1-\alpha}\right) \left( \delta_0 + \frac{\sigma^2+2c^2}{2 \mathcal{L}} \right) + \frac{\gamma (\sigma^2 L + 2c^2) }{\mu \beta} t^{-\alpha} 
\end{align*}
\textbullet \ If $\alpha = 1$ then 
\begin{align*}
\delta_t \leq 2 \exp\left( L \mathcal{L} \gamma^2\right)  \left( \delta_0 + \frac{\sigma^2 + 2c^2}{2 \mathcal{L}} \right) t^{-\mu \beta \gamma} + \left(\frac{\sigma^2L}{2}+c^2\right) \gamma^2 \varphi_{\mu \beta \gamma/2 - 1}(t) t^{-\mu \beta \gamma/2}
\end{align*}
\end{theorem}

\begin{my_remark}[Importance weights]\label{rh:corollay_weights}
The conclusion of Theorem \ref{th:convergence_as_csgd} remains valid for the update rule $\theta _ {t+1}  = \theta _ {t} -  \gamma_{t+1} W_t D(\zeta_{t+1})  g _t$  where $W_t $ is a diagonal matrix with coefficients $ (w_{t}^{(1)},\ldots,  w_{t}^{(p)})$ such that $\beta_{t+1} = \min_{1 \leq k \leq p}  w_{t}^{(k)}d_{t}^{(k)} $.
\end{my_remark}

\begin{my_remark}[Norms and constants] \label{rem:const}
 A quick inspection of the proof reveals that Assumptions \ref{cond:biased_grad} and \ref{ass:exp_smooth} may be replaced respectively by: $\forall \theta\in \mathbb R^p,h>0$, $\| \mathbb E _\xi [ g_h (\theta, \xi) ] - \nabla f(\theta) \|_\infty\leq c h$ and $\max_{k= 1,\ldots, p} \expec[  g_h^{(k)} (\theta,\xi) ^2] \leq 2 \mathcal{L} \left( f(\theta) - f(\theta^\star)\right) + \sigma^2$. Since $\|\cdot\|_{\infty}\leq \|\cdot\|_2 \leq \sqrt{p} \|\cdot\|_{\infty}$, the above constant scales more efficiently with the dimension.
\end{my_remark}

\begin{my_remark}[Rates] \label{rem:const}
The optimal convergence rate in Theorem \ref{th:non_as_bound} is of order $O(1/t)$, obtained with $\alpha=1$ under the condition $ \mu \beta \gamma>2$. Such rate matches optimal asymptotic minimax rate for stochastic approximation \citep{agarwal2012information} and recovers the rate of \citep{ajalloeian2020analysis} for SGD with biased gradients. 
\end{my_remark}

\section{MUSKETEER Algorithm} \label{sec:main_algo}

This section is dedicated to the algorithm MUSKETEER which performs an adaptive reweighting of the coordinate sampling probabilities to leverage the data structure. Note that this procedure is general and may be applied on top of any stochastic optimization algorithm as soon as one has acces to coordinates of a gradient estimate. In view of Theorem \ref{th:convergence_as_csgd_grad} and Remark \ref{rem:sparsity_func}, the main idea is to rely on a stochastic version of the Gauss-southwell rule where the coordinates of the gradients are only available through random estimates.
The algorithm of interest alternates between two elementary blocks: one for the \textit{exploration} phase and another one for the \textit{exploitation} phase. 

\medskip
\noindent\textbf{Exploration phase.} The goal of this phase is twofold: perform stochastic coordinate gradient descent and collect information about the noisy directions of the gradient. The former task is done using the current coordinate sampling distribution $Q(d_n)$ which is fixed during this phase whereas the latter is computed through cumulative gains.

\medskip
\noindent \textbf{Exploitation phase.} This phase is the cornerstone of the probability updates since it exploits the knowledge of the cumulative gains to update the coordinate sampling probability vector $d_{n}$ in order to sample more often the relevant directions of the optimization problem.

\begin{algorithm}[H]

\caption{MUSKETEER}
\algsetup{linenodelimiter=.}
\begin{algorithmic}[1]
\REQUIRE $\theta_0 \in \rset^p$, \ $N,T \in \nset$, \ $(\gamma_t)_{t \geq 0}, \ (\lambda_n)_{n \geq 0}, \ \eta >0$.
\STATE Initialize probability weights $d_0 = (1/p,\ldots,1/p)$ \ 	 \textcolor{blue}{// start with uniform sampling}
\STATE Initialize cumulative gains $G_0 = (0,\ldots,0)$
\FOR{$n=0,\ldots,N-1$} 
\STATE Initialize current gain $\widetilde{G}_0 = (0,\ldots,0)$
\STATE Run  \textbf{Explore}$(T,d_n)$ \hspace{1.6in} \textcolor{blue}{// to compute current gain $\widetilde{G}_T$}
\STATE Run \textbf{Exploit}$(G_n, \widetilde{G}_T,\lambda_n ,\eta)$ \hspace{1.05in} \ \textcolor{blue}{// to update weights $d_{n+1}$ }
\ENDFOR
\STATE Return final point $\theta_N$
\end{algorithmic}
\label{algo1}
\end{algorithm}

Consider a fixed iteration $n \in \nset$ of MUSKETEER's main loop. The \textit{exploration} phase may be seen as a multi-armed bandit problem \citep{auer2002finite} where the arms are the gradient coordinates for $k\in \llbracket 1,p \rrbracket$. At each time step $t\in \llbracket 1,T \rrbracket$, a coordinate $\zeta$ is drawn according to $Q(d_n)$ and the relative gradient $g_t^{(\zeta)}/d_{n}^{(\zeta)}$, representing the reward, is observed. 
Note that an importance sampling strategy is used to produce an unbiased estimate of the gradient when dealing with first order methods.
The rewards are then used to build cumulative gains $\widetilde{G}_{T}$ which can be written in a vectorized form as an empirical sum of the visited gradients during the \textit{exploration} phase
\begin{align} \label{eq:gains}
\forall k \in \llbracket 1,p \rrbracket, \quad \widetilde{G}_{T}^{(k)} = \frac{1}{T} \sum_{t=1}^{T} \frac{g_t^{(k)}}{d_{n}^{(k)}} \mathds{1}_{\{\zeta_{t+1} = k \}}, \quad i.e. \quad \widetilde{G}_T = \frac{1}{T} \sum_{t=1}^{T} D_n^{-1} D(\zeta_{t+1})g(\theta_t,\xi_{t+1}).
\end{align}
This average reduces the noise induced by the gradient estimates but may be sign-dependent. Thus, one may rely on the following cumulative gains which are also considered in the experiments,
\begin{align} \label{eq:gains_variants}
\widetilde{G}_T = \frac{1}{T} \sum_{t=1}^{T} D_n^{-1} D(\zeta_{t+1}) |g(\theta_t,\xi_{t+1})| \quad \text{or} \quad \widetilde{G}_T = \frac{1}{T} \sum_{t=1}^{T} D_n^{-1} D(\zeta_{t+1})g(\theta_t,\xi_{t+1})^2.
\end{align}

Starting from $G_0=(0,\ldots,0)$, the total gain $G_n$ is updated in a online manner during the \textit{exploitation} phase using the update rule $G_{n+1} = G_{n} + (\widetilde{G}_T-G_n)/(n+1)$. Once the average cumulative gains are computed, one needs to normalize them  to obtain probability weights. Such normalization can be done by a natural $\ell_1$-reweighting or a softmax operator with a parameter $\eta >0$. To cover both cases, consider the normalizing function $\varphi:\rset^p \to \rset^p$ defined by $\varphi(x)^{(k)} = |x^{(k)}|/\sum_{j=1}^p |x^{(j)}|$ or $\varphi(x)^{(k)} = \exp(\eta x^{(k)})/\sum_{j=1}^p \exp(\eta x^{(j)})$. Following the sequential approach of the EXP3 algorithm \citep{auer2002finite,auer2002nonstochastic}, the probability weights are updated through a mixture between the normalized average cumulative gains $\varphi(G_n)$ and a uniform distribution. The former term takes into account the knowledge of the gains by exploiting the rewards while the latter ensures exploration. Given a sequence $(\lambda_n) \in [0,1]^{\nset}$, we have for all $k\in \llbracket 1,p \rrbracket$,
\begin{equation} \label{eq:probas_update}
d_{n+1}^{(k)} = (1-\lambda_n) \varphi(G_n)^{(k)} + \lambda_n \frac{1}{p}\cdot
\end{equation} 

\begin{minipage}[t]{0.52\textwidth}
\vspace*{-0.3in}
\centering
\begin{algorithm}[H]
\caption{\textbf{Explore}$(T,d_n)$}
\algsetup{linenodelimiter=.}
\begin{algorithmic}[1]
\FOR{$t=1,\ldots,T$} 
\STATE Sample coordinate $\zeta \sim Q(d_n)$ and data  $\xi$
\STATE Move iterate: $\theta_{t+1}^{(\zeta)} = \theta_{t}^{(\zeta)} - \gamma_{t+1} g_h^{(\zeta)}(\theta_t,\xi)$
\STATE Update gain $\widetilde{G}_{t+1}^{(\zeta)}$ using \eqref{eq:gains} or \eqref{eq:gains_variants} 
\ENDFOR
\STATE Return vector of gains $\widetilde{G}_{T}$
\end{algorithmic}
\label{alg:explore}
\end{algorithm}
\end{minipage}
\hfill
\begin{minipage}[t]{0.4\textwidth}
\vspace*{-0.3in}
\centering
\begin{algorithm}[H]

\caption{\textbf{Exploit}$(G_n,\widetilde{G}_{T},\lambda_n, \eta)$}
\algsetup{linenodelimiter=.}
\begin{algorithmic}[1]
\STATE Update total average gain $G_n$ in an online manner
\STATE Compute normalized gains $\varphi(G_n)$ with $\ell_1$-weights or softmax
\STATE Update probability weights $d_{n+1}$ with the mixture of Eq.\eqref{eq:probas_update}
\end{algorithmic}
\label{alg:exploit}
\end{algorithm}
\end{minipage}

\medskip
\noindent
In view of Theorem \ref{th:convergence_as_csgd}, the convergence of the sequence of iterates  $(\theta_t)_{t \in \nset}$ obtained by MUSKETEER relies on the extended Robbins-Monro condition $\sum_{t \geq 1} \beta_t \gamma_t = +\infty$ which is implied by the weaker condition $\sum_{t\geq 1} \lambda_t \gamma_t = +\infty$ for both $\ell_1$ and softmax weights. Observe that such a constraint is easily verified with either a fixed value $\lambda_t \equiv \lambda$ in the mixture update or more generally a slowly decreasing sequence, e.g. $\lambda_t = 1/\log(t)$. Since the gradients $\nabla f(\theta_t)$ get smaller through the iterations, the softmax weights get closer to $1/p$. Thus, in the asymptotic regime, there is no favorable directions among all the possible gradient directions. Hence, near the optimum, the \textit{coordinate sampling policy} of MUSKETEER with softmax weights is likely to treat all the coordinates equally. 

\begin{theorem} \label{th:convergence_law} (Weak convergence)
Suppose that Assumptions \ref{cond:biased_grad} to \ref{ass:exp_smooth} are fulfilled and that the learning rates satisfy the standard Robbins-Monro condition. Then  MUSKETEER's coordinate policy $(Q(d_n))_{n \in \nset}$ with softmax normalization converges weakly to the uniform distribution, i.e., $Q(d_n) \leadsto \mathcal{U}(\llbracket 1,p \rrbracket)$ as $n \to +\infty$.
\end{theorem}

\begin{my_remark}(On the choice of $\lambda_n$ and $\eta$) The uniform term in Equation \eqref{eq:probas_update} ensures that all coordinates are eventually visited. Taking $\lambda_n \to 0$ at a specific rate (which can be derived from the proof) gives more importance to the cumulative gains. The parameter $\eta$ is fixed during the algorithm and may be tuned through an analysis of the regret \citep{auer2002finite}.
\end{my_remark}

\begin{my_remark}(Choice of Exploration Size $T$) Choosing the value of $T$ is a central question known as the exploration-exploitation dilemma in reinforcement learning. As $T$ gets large, the exploration phase gathers more information leading to fewer but more accurate updates. Conversely, with a small value of $T$, the probabilities get updated more often, at the price of less collected information. Setting $T=p$ ensures that, in average, all the coordinates are visited once during the exploration phase. Nevertheless, a smaller value $T=\lfloor \sqrt{p} \rfloor$ is taken in the experiments and lead to great performance.
\end{my_remark}

\begin{my_remark} (Asymptotic behavior) \label{rem:asymp_conjecture}
The previous results highlight two main features of MUSKETEER: the sequence of iterates converges almost surely and the coordinate policy converges weakly. The latter point suggests that, in the long run, MUSKETEER is similar to the uniform coordinate version of SCGD. However, the weak convergence of the rescaled process $(\theta_t - \theta^\star) /  \sqrt { \gamma_t}$ remains an open question. In light of the link between SCGD and Conditioned-SGD, discussed in Section \ref{sec:policy}, we conjecture that the behavior of MUSKETEER with softmax weights is asymptotically equivalent to SCGD with uniform policy. This is in line with the continuity property  obtained in \citet{leluc2020towards} within the Conditioned-SGD framework and relates to the convergence of stochastic Newton algorithms \citep{boyer2020asymptotic}.
\end{my_remark}

\section{Numerical Experiments} \label{sec:simus}

In this section, we empirically validate the SCGD framework by running MUSKETEER and competitors on synthetic and real datasets. First, we focus on regularized regression problems adopting the data generation process of \citep{namkoong2017adaptive} in which the covariates exhibit a certain block structure. Second, MUSKETEER is employed to  train different neural networks models on real datasets for multi-label classification task. For ease of reproducibility, the code is available online\footnote{https://github.com/RemiLELUC/SCGD-Musketeer}. Technical details and additional results (with different data settings, normalization and hyperparameters) are available in the appendix.

\medskip
\noindent \textbf{Methods in competition.} 
The set of methods is restricted to zeroth-order methods. This choice leads to an honest comparison based on the number of function queries. MUSKETEER is implemented according to Section \ref{sec:main_algo} with $T = \lfloor \sqrt{p} \rfloor $, softmax and $\ell_1$ normalization for the simulated and real data respectively. The different cumulative gains of Eq. \eqref{eq:gains_variants} are considered, namely AVG, SQR and ABS for the gradients, their squares or their absolute value respectively. The method FULL is the finite difference gradient estimate computed over all coordinates and UNIFORM stands for the uniform coordinate sampling policy. NESTEROV implements the gaussian smoothing of \citep{nesterov2017random}. In all cases, the initial parameter is set to $\theta_0 = (0,\ldots, 0)^\top \in \rset^p$ and the optimal SGD learning rate of the form $\gamma_k = \gamma/(k+k_0)$ is used. 

\medskip
\noindent \textbf{Regularized linear models.} We apply the Empirical Risk Minimization paradigm to regularized linear problems. Given a data matrix $X=(x_{i,j}) \in \rset^{n \times p}$, labels $y \in \rset^n$ or $\{-1,+1\}^n$ and a regularization parameter $\mu>0$, the \textit{Ridge regression} objective is defined by 
\begin{align*}
f(\theta)=\frac{1}{2n}   \sum_{i=1}^n (y_i - \sum_{j=1}^p x_{i,j} \theta_j)^2 +  \frac{\mu}{2} \|\theta\|_2^2
\end{align*}
and the $\ell_2$\textit{-regularized logistic regression} is given by
\begin{align*}
f(\theta)= \frac{1}{n}  \sum_{i=1}^n \log(1+\exp(-y_i \sum_{j=1}^p x_{i,j} \theta_j)) + \mu \|\theta\|_2^2.
\end{align*}
Similarly to \citep{namkoong2017adaptive}, we endow the data matrix $X$ with a block structure. The columns are drawn as $X[:,k] \sim \mathcal{N}(0,\sigma_k^2 I_n)$ with $\sigma_k^2 = k^{-\alpha}$ for all $k\in \llbracket 1,p \rrbracket$. The parameters are set to $n=10,000$ samples in dimension $p=250$ with an exploration size equal to $T = \lfloor \sqrt{p} \rfloor = 15$. The regularization parameter is set to the classical value $\mu=1/n$. Figure \ref{fig:zo_losses_linear} provides the graphs of the optimaliy gap $t\mapsto f(\theta_t)-f(\theta^\star)$ averaged over $20$ independent simulations for different values of $\alpha \in \{2;5;10\}$. First, note that the uniform sampling strategy shows similar performance to the classical full gradient estimate. Besides, MUSKETEER with average or absolute gains shows the best performance in all configurations. Greater values of $\alpha$, \textit{i.e.} stronger block structure, improve our relative performance with respect to the other methods as shown by Figures \ref{fig:ridge_zo_10} and \ref{fig:log_zo_10}.  

\begin{figure}[h]
  \centering
  \subfigure[Ridge $\alpha=5$]{
  \includegraphics[scale=0.3]{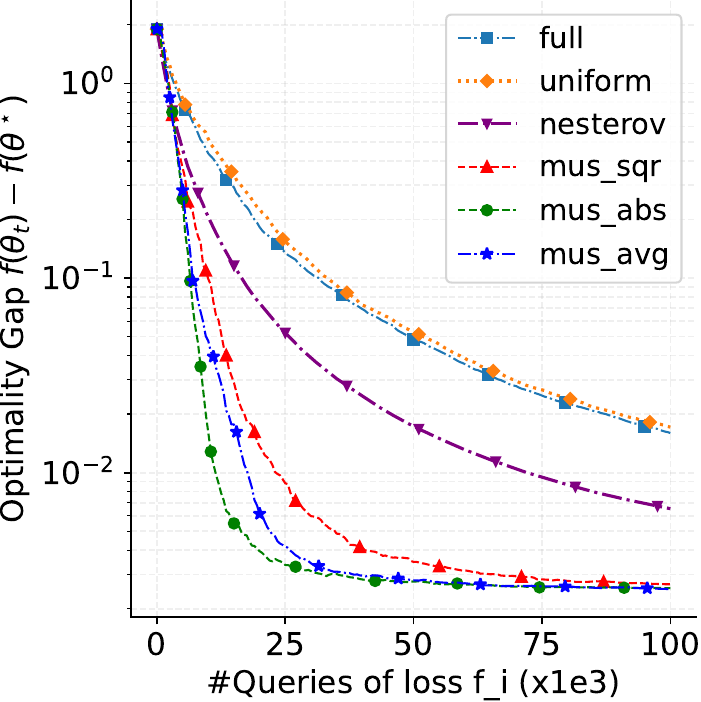}\label{fig:ridge_zo_5}}
  \subfigure[Ridge $\alpha=10$]{
  \includegraphics[scale=0.3]{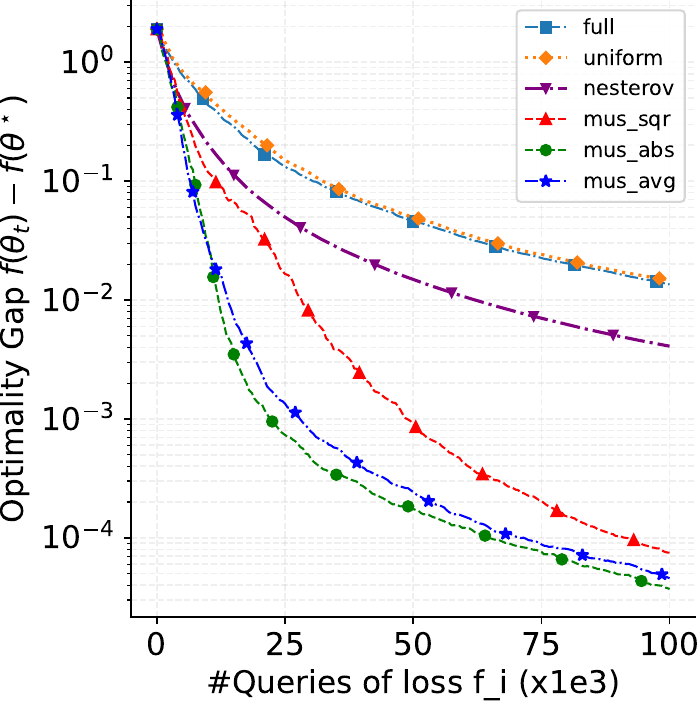}\label{fig:ridge_zo_10}}
  \subfigure[Logistic $\alpha=2$]{
  \includegraphics[scale=0.3]{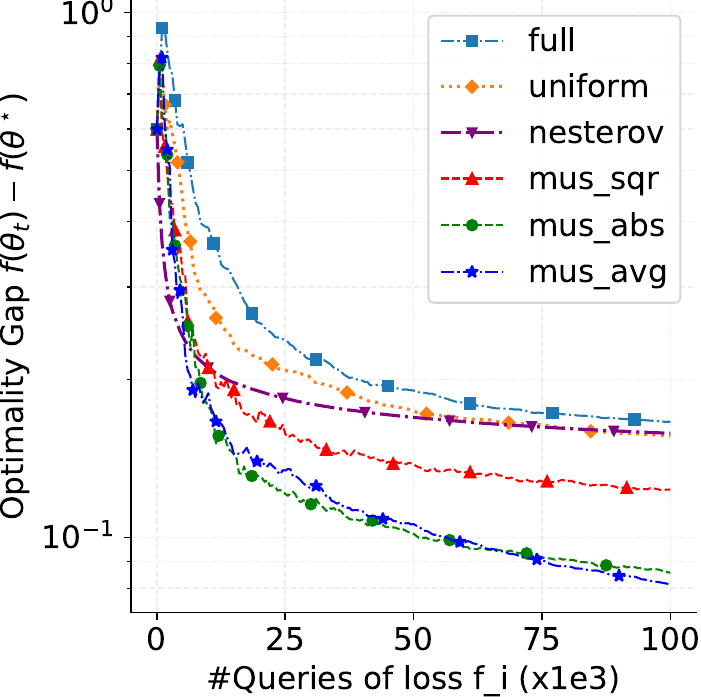}\label{fig:log_zo_5}}
  \subfigure[Logistic $\alpha=5$]{
  \includegraphics[scale=0.3]{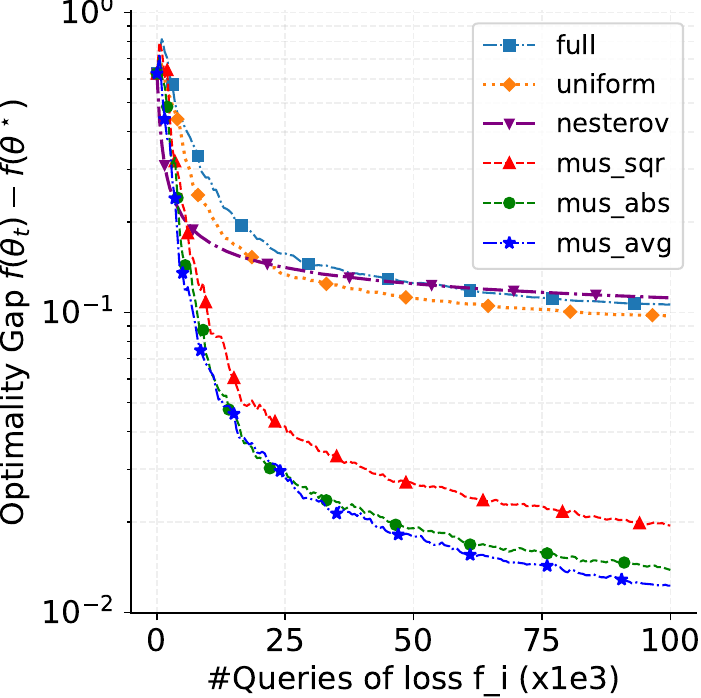}\label{fig:log_zo_10}}
  \caption{$[f(\theta_t)-f(\theta^\star)]$ for Ridge and Logistic on Synthetic data with different block structures.}
\label{fig:zo_losses_linear}
\end{figure}

\noindent
 \textbf{Neural Networks.} We focus on the training of neural networks within the framework of multi-label classification. The datasets in the experiments are popular publicly available deep learning datasets: MNIST \citep{deng2012mnist} and Fashion-MNIST \citep{xiao2017fashion}. Given an image, the goal is to predict its label among ten classes. The neural architecture is based on linear layers in dimension $p=55,050$ with $T  = 234$. Figure \ref{fig:deep_training_loss} shows the means and standard deviations of the training losses of the different ZO methods averaged over $10$ independent runs. Interestingly, the performance of MUSKETEER also benefit from the adaptive structure in terms on accuracy of the test set (see Figures \ref{fig:acc_mnist} and \ref{fig:acc_fash}). This allows to quantify the statistical gain brought by MUSKETEER over standard ZO methods. 

\begin{figure}[h]
\begin{minipage}{0.5\textwidth}
  \centering
  \subfigure[MNIST]{
  \includegraphics[scale=0.28]{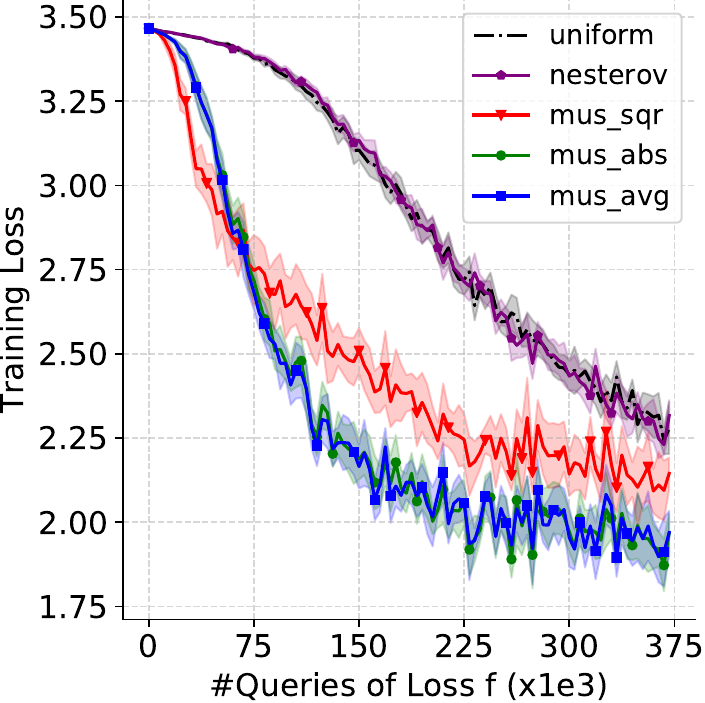}\label{fig:mnist}}
  \subfigure[Fashion-MNIST]{
  \includegraphics[scale=0.28]{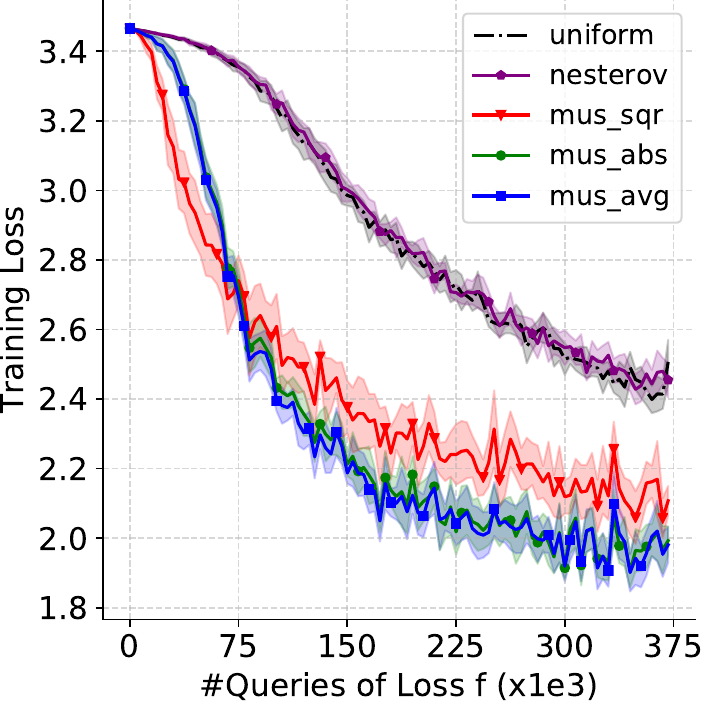}\label{fig:fash}}
    \caption{Evolution of training loss.}
    \label{fig:deep_training_loss}
\end{minipage}
\begin{minipage}{0.5\textwidth}
  \subfigure[MNIST]{
  \includegraphics[scale=0.28]{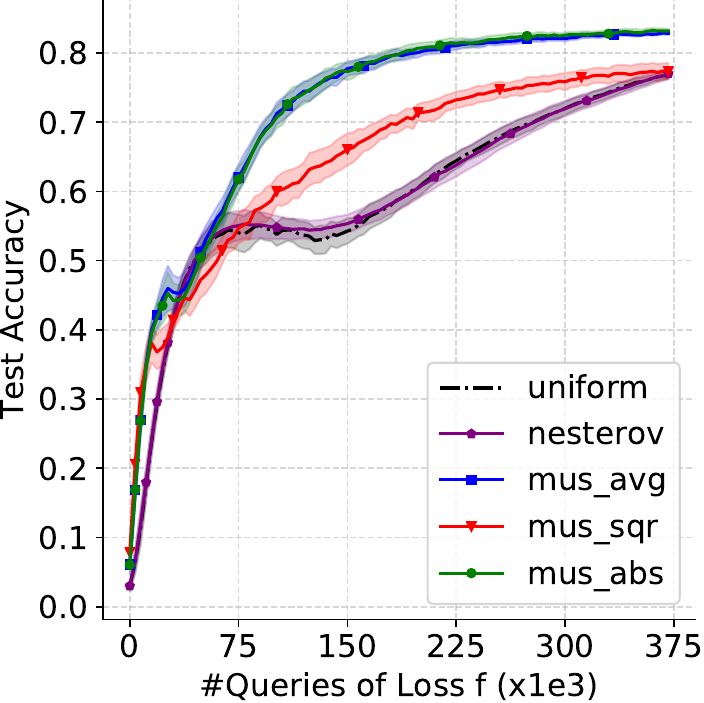}\label{fig:acc_mnist}}
  \subfigure[Fashion-MNIST]{
  \includegraphics[scale=0.28]{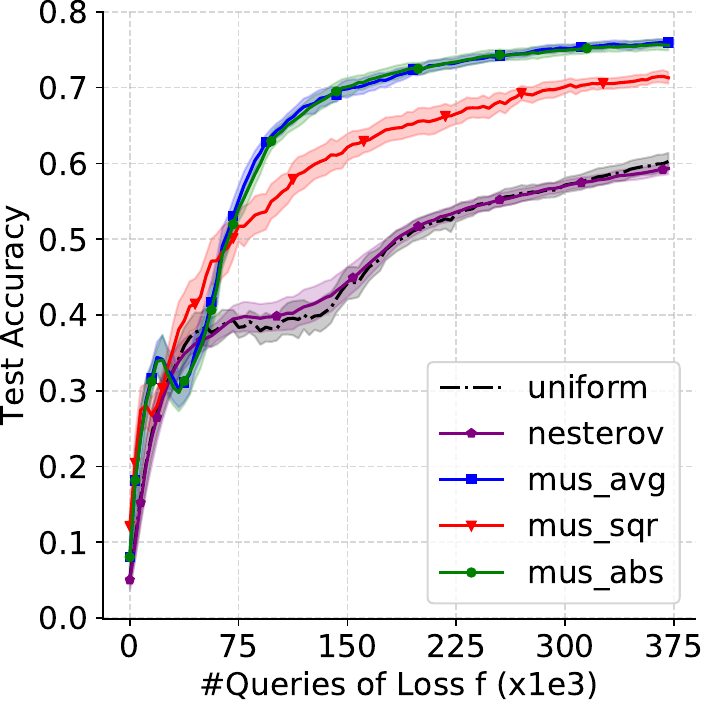}\label{fig:acc_fash}}
  \caption{Evolution of test accuracy.}
\label{fig:deep_acc_test}
\end{minipage}
\end{figure}
\vspace{-0.5cm}
\acks{The authors are grateful to the Associate
Editor and two anonymous Reviewers for their many valuable
comments and interesting suggestions.}

\newpage

\appendix
\addcontentsline{toc}{section}{Appendix} 
\part{Appendices}

\parttoc

\section{Technical Proofs} \label{app:proofs}

\subsection{Proof of Proposition \ref{useful_bound}}
Under Eq.\eqref{cond:grad_NEW2}, using Jensen inequality, we find
\begin{align*}
\| \expec_{\xi}[ g_\mu (\theta,\xi ) ]  - \nabla f (\theta)  \| _2^2 
& =  \left\| \int_{\rset^p}   x  \left(   \frac{f(\theta + \mu x ) - f(\theta )}{\mu}      -    x^\top \nabla f (\theta) \right) \nu (\mathrm{d}x)\right\|_2^2\\
&\leq \int_{\rset^p} \| x\|^2_2 \left(    \frac{f(\theta + \mu x ) - f(\theta )}{\mu}     -    x^\top \nabla f (\theta) \right)^2 \nu (\mathrm{d}x)\\
&= \mu^{-2} \int_{\rset^p} \| x\|^2_2 \left(    {f(\theta + \mu x ) - f(\theta )}    -   \mu x^\top \nabla f (\theta) \right)^2 \nu (\mathrm{d}x)
\end{align*}
Using the quadratic bound of $L$-smooth functions, we obtain
\begin{align*}
\| \expec_{\xi}[ g_\mu (\theta,\xi ) ]  - \nabla f (\theta)  \| _2^2 
& \leq  \mu^{-2} \frac{L^2}{4}  \int \|x\|_2^2 \|\mu x \|_2^4\nu (\mathrm{d}x)    
=  \mu^{2} \frac{L^2}{4}  \int \|x\|_2^6  \nu (\mathrm{d}x).
\end{align*}

\subsection{Deterministic results for convergence of gradients} \label{app:high_level_results}

In this section we provide results ensuring the convergence to $0$ of several gradient descent algorithms. They are meant to be \textit{high-level} as they may be applied in different situations  and \textit{deterministic} because no randomness is measured but only an inclusion of events is considered. The results are key in the proofs.

\begin{lemma}[Deterministic result 1] \label{lemma:key_ps}
Let $f:\rset^p \to \rset$ be a $L$-smooth function, $(\gamma_t)_{t \geq 1}$ a positive sequence of learning rates such that $\sum_t \gamma_t = \infty$. Let $(\theta_t)$ a random sequence obtained by the SGD update rule $\theta_{t+1} = \theta_t - \gamma_{t+1} g_t$. Let $\omega\in \Omega$ such that the following limits exist:
\begin{align*}
&(i) \ \sum_{t\geq 0} \gamma_{t+1} \|\nabla f(\theta_t(\omega)) \|_2^2 < \infty \quad (ii) \ \sum_{t\geq 1}   \gamma _ {t} (g_{t-1 }(\omega) - \nabla f(\theta_{t-1} (\omega)) ) < \infty 
\end{align*}
then  $\nabla f(\theta_t(\omega)) \to 0$ as $t \to \infty$.
\end{lemma}

The next Lemma  is the equivalent of Lemma \ref{lemma:key_ps} for a specific procedure which, at each iteration, moves only one well-chosen coordinate: the one with highest gradient value.

\begin{lemma}[Deterministic result 2]\label{lemma:key_ps2} 
Let $f:\rset^p \to \rset$ be a $L$-smooth function (with respect to $|\cdot | _\infty $), 
$(\gamma_t)_{t \geq 1}$ a positive sequence of learning rates such that $\sum_t \gamma_t = \infty$. Let $(\theta_t)$ a random sequence obtained by the SCGD update rule $\theta_{t+1} = \theta_t - \gamma_{t+1} D (\zeta_{t+1})  g_t$ with $\zeta_{t+1} = \argmax_{k=1,\ldots , p} | \partial_k f(\theta_t)|$. Let $\omega\in \Omega$ such that the following limits exist:
\begin{align*}
&(i) \ \sum_{t\geq 0} \gamma_{t+1}   |  \nabla  f(\theta_t(\omega) |_\infty ^2 < \infty \quad (ii) \ \sum_{t\geq 1}   \gamma _ {t} 
D(\zeta_{t}) ( g_{t-1}(\omega)  -    \nabla  f(\theta_{t-1} (\omega)  ) ) < \infty,
\end{align*}
then  $\nabla f(\theta_t(\omega)) \to 0$ as $t \to \infty$.
\end{lemma}

We conclude with one last result which is valid for procedure where only one coordinate (chosen randomly) is moved at each iteration.

\begin{lemma}[Deterministic result 3] \label{lemma:key_ps3}
Let $f:\rset^p \to \rset$ be a $L$-smooth function, $(\gamma_t)_{t \geq 1}$ a positive sequence of learning rates such that $\sum_t \gamma_t = \infty$. Let $(\theta_t)$ a random sequence obtained by the SCGD update rule $\theta_{t+1} = \theta_t - \gamma_{t+1}  D (\zeta_{t+1}) g_t$ where $\zeta_{t+1} \sim Q ( d_t)$.
Let $\omega\in \Omega$ such that the following limits exist:
\begin{align*}
&(i) \ \sum_{t\geq 0} \gamma_{t+1}   \|\nabla f(\theta_t(\omega)) \|_2^2 < \infty \quad (ii) \ \sum_{t\geq 1}   \gamma _ {t} (D (\zeta_{t}(\omega)) g_{t-1 }(\omega) -D_{t-1} \nabla f(\theta_{t-1} (\omega)) ) < \infty 
\end{align*}
then  $\nabla f(\theta_t(\omega)) \to 0$ as $t \to \infty$.
\end{lemma}

\noindent \textbf{Proof of Lemma \ref{lemma:key_ps}.}
The proof (and in particular the reasoning by contradiction) is inspired from the proof of Proposition 1 in \cite{bertsekas2000gradient}. For ease of notation we omit the $\omega$ in the proof. Note that condition (i) along with $\sum_t \gamma_t = \infty$ implie that $\liminf _{t} \| \nabla f(\theta_t) \| = 0$. Now, by contradiction, let $\varepsilon>0$ and assume that
\begin{align*}
\limsup _{t} \| \nabla f(\theta_t) \| >\varepsilon
\end{align*}
We have that  there is infinitely many $t$ such that $ \| \nabla f(\theta_t) \| < \varepsilon/2$ and also infinitely many $t$ such that $ \| \nabla f(\theta_t) \| >\varepsilon$. It follows that there is infinitely many \textit{crossings} between the sets $\{t \in \nset:  \| \nabla f(\theta_t) \| < \varepsilon/2\}$ and $\{t \in \nset:  \| \nabla f(\theta_t) \| > \varepsilon\}$. A \textit{crossing} is a collection of indexes $ I_k = \{L_k,L_k+1 ,\ldots, U_k - 1 \} $ with $L_k \leq U_k$ ($I_k = \emptyset $ when $L_k = U_k$) such that for all $t\in I_k $, 
$$\| \nabla f(\theta_{L_k-1} ) \| < \varepsilon / 2  \leq  \| \nabla f(\theta_t) \| \leq \varepsilon < \| \nabla f(\theta_{U_k} ) \| .$$
Define the following partial Cauchy sequence $R_k =  \sum_{t = L_k } ^{ U_k}   \gamma _ {t} (g_{t-1 } - \nabla f(\theta_{t-1} ) ) $ and note that condition (ii) implies that $R_k \to 0$ as $k \to \infty$. For all $k\geq 1$,
\begin{align*}
\varepsilon / 2 
&\leq \| \nabla f(\theta_{U_k } )    \|_2  -   \| \nabla f(\theta_{L_k-1} )    \|_2 \\
&\leq \| \nabla f(\theta_{U_k } )  -  \nabla f(\theta_{L_k-1} )    \|_2 \\
&\leq L \| \theta_{U_k}  - \theta_{L_k-1} \|_2,
\end{align*}
where we use that $\nabla f$ is $L$-Lipschitz. Then using the update rule $\theta_{t} - \theta_{t-1} = - \gamma_t g_{t-1}$, we have by sum
\begin{align*}
\varepsilon / 2  
&\leq L  \| \sum_{t = L_k } ^{ U_k}   \theta_{t}  - \theta_{t-1}  \|_2 =L \| \sum_{t = L_k } ^{ U_k}   \gamma _ {t} g_{t-1 }  \|_2\\
& \leq  L \| \sum_{t = L_k } ^{ U_k}   \gamma _ {t} \nabla f(\theta_{t-1} ) \|_2+  L\| \sum_{t = L_k } ^{ U_k}   \gamma _ {t} (g_{t-1 } - \nabla f(\theta_{t-1} ) ) \|_2 \\
&\leq  L  \sum_{t = L_k } ^{ U_k}   \gamma _ {t} \|  \nabla f(\theta_{t-1} ) \|_2+ L \|R_k\|_2 
\end{align*} 
Since in the previous equation $ \|  \nabla f(\theta_{t-1} ) \|_2 > \varepsilon / 2   $, we get
$$ (\varepsilon / 2   )^2 \leq  L  \sum_{t = L_k } ^{ U_k}   \gamma _ {t} \|  \nabla f(\theta_{t-1} ) \|_2^2+ ( \varepsilon /2 ) L \|R_k\|_2 $$
But since $\sum_{t \geq 0 }  \gamma_{t+1} \| \nabla f(\theta_{t}) \|^2 $ is finite and $\lim_k R_k =  0$, the previous upper bound goes to $0$ and implies a contradiction.
\qed

\medskip
\noindent \textbf{Proof of Lemma \ref{lemma:key_ps2}.}
For ease of readability, the variable $\omega$ is removed during the proof. By assumption, $ |\nabla_{\zeta_{t+1}} f(\theta_t)| = | \nabla  f(\theta_t ) |_\infty$. Hence, (i) yields that $\liminf_t   | \nabla f(\theta_{t}) |_\infty =0$. 
The proof is by contradiction. Suppose that  $\limsup_t   | \nabla f(\theta_{t}) |_\infty >\epsilon $. There exists a sequence of \textit{crossings} between the sets $\{t \in \nset:  | \nabla f(\theta_{t}) |_\infty < \varepsilon/2\}$ and $\{t \in \nset:  | \nabla f(\theta_{t}) |_\infty > \varepsilon\}$. Formally, there is a collection of indexes $ I_k = \{L_k,L_k+1 ,\ldots, U_k - 1 \} $ with $L_k \leq U_k$ ($I_k = \emptyset $ when $L_k = U_k$) such that for all $t\in I_k $, 
$$ | \nabla f(\theta_{L_k-1}) |_\infty < \varepsilon / 2  \leq   | \nabla f(\theta_{t}) |_\infty  \leq \varepsilon < | \nabla f(\theta_{U_k}) |_\infty .$$
Define
$$R_ k  = \sum_{t = L_k } ^{U_k}    \gamma _ {t}  D_{\zeta_t}  ( g_{t-1}   -   \nabla  f(\theta_{t-1} ) ) $$
and use that $\nabla f$ is $L$-smooth to get
\begin{align*}
(\varepsilon / 2) &\leq   | \nabla f(\theta_{U_k }) |_\infty -  | \nabla f(\theta_{L_k-1}) |_\infty \\
&\leq  L |\theta_{U_k } -\theta_{L_k-1}  | _ \infty \\
& \leq   L | \sum_{t= L_k } ^{U_k} \gamma_{t} D_{\zeta_t}  \nabla f ( \theta_{t-1})  | _ \infty + L \left| \sum_{t = L_k } ^{U_k}    \gamma _ {t}  
D_{\zeta_{t}} ( g_{t-1}  -    \nabla  f(\theta_{t-1} ) ) \right|_\infty \\
& =   L | \sum_{t= L_k } ^{U_k} \gamma_{t} D_{\zeta_t}  \nabla f ( \theta_{t-1})  | _ \infty +  L | R_ k|_\infty\\
& \leq  L \sum_{t= L_k } ^{U_k} \gamma_{t} |D_{\zeta_t}  \nabla f ( \theta_{t-1})  | _ \infty +  L  | R_ k|_\infty
\end{align*}
Noting that $|D_{\zeta_t}  \nabla f ( \theta_{t-1} )  | _ \infty = |\nabla f ( \theta_{t-1} )  | _ \infty  > \varepsilon /2 $, we get
\begin{align*}
(\varepsilon / 2)^2  & \leq  L  \sum_{t= L_k } ^{U_k} \gamma_{t} | \nabla f ( \theta_{t-1})  |^2 _ \infty +  (\varepsilon /2 ) L| R_ k|_\infty.
\end{align*}
As the previous upper bound converges to $0$ by assumption we reach a contradiction.
\qed

\medskip
\noindent \textbf{Proof of Lemma \ref{lemma:key_ps3}.} Following the proof of Lemma \ref{lemma:key_ps}, we assume that $\limsup_{t} \|\nabla f(\theta_{t} ) \|_2 >\varepsilon$ and consider the same collection of crossing indexes $(L_k,U_k)$ to obtain that
\begin{align*}
\varepsilon / 2  
&\leq   L  \sum_{t = L_k } ^{ U_k}   \gamma _ {t} \| D _{t-1} \nabla f(\theta_{t-1} ) \|_2+ L \|R_k\|_2 
\end{align*} 
where $ R_k =  \sum_{t = L_k } ^{ U_k}   \gamma _ {t} (  D  ( \zeta_t ) g_{t-1 } -  D _{t-1} \nabla f(\theta_{t-1} ) ) $ is a sequence that goes to $0$.
Since in the previous equation $ D _{t-1} \preceq I_d$  and $ \|  \nabla f(\theta_{t-1} ) \|_2 > \varepsilon / 2   $, we get
$$ (\varepsilon / 2   )^2 \leq  L  \sum_{t = L_k } ^{ U_k}   \gamma _ {t} \|  \nabla f(\theta_{t-1} ) \|_2^2+ ( \varepsilon /2 ) L \|R_k\|_2 $$
and a contradiction follows as the above term goes to $0$. 
\qed

\subsection{Proof of Theorem \ref{th:convergence_as_sgd}}

The proof follows from applying Lemma \ref{lemma:key_ps} in which two conditions are required:
\begin{align*}
&(i) \ \sum_{t\geq 0} \gamma_{t+1} \|\nabla f(\theta_t(\omega)) \|_2^2 < \infty \quad (ii) \ \sum_{t\geq 1}   \gamma _ {t} (g_{t-1 }(\omega) - \nabla f(\theta_{t-1} (\omega)) ) < \infty.
\end{align*}
\medskip
\noindent
\textbf{Proof of condition (i).} We classically rely on the Robbins-Siegmund Theorem (Theorem \ref{th:rs} in Section \ref{sec:rs_th}). Since $\theta \mapsto f(\theta)$ is $L$-smooth, we have the quadratic bound $f(\eta) \leq f(\theta) + \langle \nabla f(\theta),\eta-\theta \rangle + \frac{L}{2} \|\eta-\theta\|_2^2$. Using the update rule $\theta_{t+1} = \theta_{t} - \gamma_{t+1} g_t$, we get
\begin{align*}
f(\theta_{t+1}) 
&\leq f(\theta_{t}) + \langle \nabla f(\theta_{t}),\theta_{t+1}-\theta_{t} \rangle + \frac{L}{2}\|\theta_{t+1}-\theta_{t}\|_2^2 \\
&= f(\theta_{t}) - \gamma_{t+1} \langle \nabla f(\theta_{t}),  g_t \rangle + \frac{L}{2} \gamma_{t+1}^2\|  g_t\|_2^2.
\end{align*}
Using that 
\begin{align*}
2\langle a ,  b \rangle = \| a\|_2 ^2 + \| b\|_2 ^2 - \| a-b\|_2 ^2 &\geq \| a\|_2 ^2  - \| a-b\|_2 ^2
\end{align*}
and taking the conditional expectation, we get
\begin{align*}
\expec_t\left[f(\theta_{t+1})\right] 
&\leq f(\theta_{t})  - \gamma_{t+1}\langle \nabla f(\theta_{t}), \expec_t[  g_t] \rangle + \frac{L}{2} \gamma_{t+1}^2\expec_t[\|  g_t\|_2^2]
\\
& \leq f(\theta_{t})  - \frac{\gamma_{t+1}}{2} \|  \nabla f(\theta_{t})\| _2^2 + \frac{\gamma_{t+1}}{2} \|  \nabla f(\theta_{t}) - \expec_t[  g_t] \| _2^2  + \frac{L}{2} \gamma_{t+1}^2\expec_t[\|  g_t\|_2^2]
\end{align*}
On the one hand, using Assumption \ref{cond:biased_grad}, we obtain
\begin{align*}
\|  \nabla f(\theta_{t}) - \expec_t[  g_t] \| _2^2 &\leq h_{t+1}^{2} c^2
\end{align*}
On the other hand, using Assumption \ref{ass:exp_smooth}, there exist $0 \leq \mathcal{L},\sigma^2 <\infty$ such that almost surely
\begin{align*}
\forall t \in \nset,  \quad \expec_t\left[\|g_t\|_2^2\right]  =\expec_\xi \left[\| g(\theta_t, \xi) \|_2^2 \right] \leq 2 \mathcal{L} \left( f(\theta_{t}) - f^{\star}\right) + \sigma^2.
\end{align*}
Injecting $-f^\star $ on both sides, it follows that
\begin{align*}
\expec_t\left[f(\theta_{t+1})- f ^\star\right]  
\leq (1 + L \mathcal {L}  \gamma_{t+1}^2) ( f(\theta_{t}) -f^\star) - \frac{\gamma_{t+1}}{2}\|  \nabla f(\theta_{t})\| _2^2 + 
\gamma_{t+1}  h_{t+1} ^2 c^2 + \frac{L}{2} \gamma_{t+1}^2 \sigma^2
\end{align*}
Introduce $V_t = f(\theta_{t})-f^{\star}, W_t = \gamma_{t+1}  \|\nabla f(\theta_{t})\|_2^2 /2 $, $a_t = L \mathcal {L}  \gamma_{t+1}^2$  and $b_t =  c^2 h_{t+1} ^2  \gamma_{t+1} + (L/2) \gamma_{t+1} ^2 \sigma^2 $. These four random sequences are non-negative $\mcf_t$-measurable sequences with $\sum_t a_t <\infty$ and $\sum_t b_t <\infty$ almost surely. We have: $\forall t \in \nset,\expec\left[V_{t+1}|\mcf_t\right] \leq (1+a_t) V_{t} - W_t + b_t.$ We can apply Robbins-Siegmund Theorem  to have
\begin{align*}
(a) \  \sum_{t \geq 0} W_t <\infty  \ a.s. \qquad (b) \  V_{t} \stackrel{a.s.}{\longrightarrow} V_{\infty}, \expec\left[V_{\infty}\right]<\infty. \qquad (c) \ \sup _{t \geq 0} \expec\left[V_{t}\right]<\infty.
\end{align*}
Therefore we have a.s. that $(f(\theta_{t}))$ converges to a finite value $f_{\infty} \in L^1$ and $\sum_{t \geq 0} \gamma_{t+1} \|\nabla f(\theta_{t})\|_2^2 <+\infty$. There exists an event $\Omega_0 \subset \Omega$ such that, $\prob(\Omega_0)=1$ and for every $\omega \in \Omega_0$, $\lim_{t} f(\theta_t(\omega)) < \infty $ and
$ \sum_{t \geq 0} \gamma_{t+1} \|\nabla f(\theta_t(\omega)) \|_2^2 < \infty$.

\noindent
\textbf{Proof of condition (ii).}
We place ourselves on the event $\Omega_0$ and omit the $\omega$ in notation for ease of clarity. First, since $ \limsup_t f(\theta_{t}) < \infty$, we have that $(f(\theta_t))$ is bounded almost surely. It yields, in virtue of  Assumption \ref{ass:exp_smooth} that $ \expec_t\left[\|g_t\|_2^2\right]  \leq 2 \mathcal{L} \left( f(\theta_{t}) - f^{\star}\right) + \sigma^2 \leq C $ where $C$ is a some finite random variable and the latter holds almost surely. It then follows that, almost surely $\sum_{t\geq 1} \gamma_t^2 \mathbb E _t [ \|g_t\|^2] \leq  C  \sum_{t\geq 1} \gamma_t^2 < \infty$. Now, observe that condition (ii) is satisfied as soon as 
\begin{align*}
(a) \ \| \sum_{t\geq 0}  \gamma_{t+1} (g_t - \mathbb E _ t [g_t] ) \|_2  <\infty \quad \text{and} \quad (b)\ \| \sum_{t\geq 0}  \gamma_{t+1} ( \mathbb E _ t [g_t]  -\nabla f(\theta_t)  ) \|_2  <\infty.
\end{align*}
Equation (a) involves martingale increments whose quadratic variation satisfies
\begin{align*}
\sum_{t\geq 0} \gamma_{t+1}^2 \mathbb E _t [ \|g_t - \mathbb E_ t [ g_t]\|^2]  
&\leq  \sum_{t\geq 0} \gamma_{t+1}^2 \mathbb E _t [ \|g_t\|^2] < \infty,
\end{align*}
which ensures that $\sum_{t\geq 0} \gamma_{t+1} ( g_t - \mathbb E_ t [ g_t] ) <\infty $ a.s. in virtue of Theorem 2.17 in \cite{hall1980martingale}. The term in equation (b) is bounded using assumption \ref{cond:biased_grad} and we have
\begin{align*}
\sum_{t\geq 0}  \gamma_{t+1}^2 \| \mathbb E _ t [g_t]  -\nabla f(\theta_t)  \|_2^2 \leq c^2 \sum_{t\geq 0}  \gamma_{t+1}^2 h_t^2 <\infty,
\end{align*}
which finally proves
\begin{align*}
(ii) \ \sum_{t\geq 0}   \gamma _ {t+1} (g_{t}(\omega) - \nabla f(\theta_{t} (\omega)) ) < \infty 
\end{align*}
and gives, in virtue of Lemma \ref{lemma:key_ps} the conclusion $\nabla f(\theta_t) \to 0$ almost surely as $t \to +\infty$.

\subsection{Proof of Theorem \ref{th:convergence_as_csgd_grad}}

\noindent
\textit{Part (a) Maximum gradient.} The proof follows from applying Lemma \ref{lemma:key_ps2} in which two conditions are required:
\begin{align*}
&(i) \ \sum_{t\geq 0} \gamma_{t+1}   |  \nabla  f(\theta_t(\omega) |_\infty ^2 < \infty \quad (ii) \ \sum_{t\geq 0}   \gamma _ {t+1} 
D(\zeta_{t+1}) ( g_{t}(\omega)  -    \nabla  f(\theta_{t} (\omega)  ) ) < \infty.
\end{align*}
\noindent
\textbf{Proof of condition (i).} Again, we rely on the quadratic bound
\begin{align*}
f(\theta_{t+1}) 
&\leq  f(\theta_{t}) - \gamma_{t+1} \langle \nabla f(\theta_{t}),  D(\zeta_{t+1})  g_t \rangle + \frac{L}{2} \gamma_{t+1}^2\| D(\zeta_{t+1})  g_t\|_2^2\\
& =  f(\theta_{t}) - \gamma_{t+1}  \nabla_{\zeta_{t+1}} f(\theta_{t})   g_t^{(\zeta_{t+1})}  + \frac{L}{2} \gamma_{t+1}^2   g_t^{(\zeta_{t+1})2}
\end{align*}
Taking the expectation with respect to $\xi_{t+1}$  and using Assumption \ref{ass:indep_random}, we find
\begin{align*}
\mathbb E_ { \xi_{t+1}} [ f(\theta_{t+1}) -  f^\star ]  
& \leq  f(\theta_{t}) - f^\star - \gamma_{t+1}  \nabla_{\zeta_{t+1}} f(\theta_{t})   \tilde g_t   ^{(\zeta_{t+1})}  + \frac{L}{2} \gamma_{t+1}^2  \mathbb E_ { \xi_{t+1}} [ g_t^{(\zeta_{t+1})2}] 
\end{align*}
where $\tilde g_t =  E_ { \xi} [g _{h_{t+1}} ( \theta_t, \xi) ] $.
We use the inequality $2ab \geq a^2  - (a-b)^2 $ and Assumption \ref{cond:biased_grad} to get
\begin{align*}
2 \nabla_{\zeta_{t+1}} f(\theta_{t})  \tilde  g_t^{(\zeta_{t+1})} 
&\geq \nabla_{\zeta_{t+1}} f(\theta_{t}) ^2- ( \nabla_{\zeta_{t+1}} f(\theta_{t})  -    \tilde g_t^{(\zeta_{t+1})}  )^2\\
&\geq \nabla_{\zeta_{t+1}} f(\theta_{t}) ^2- \max_{k=1,\ldots, p}  ( \partial_{k}  f(\theta_{t})  -    \tilde g_t^{(k)}  )^2\\
& \geq \nabla_{\zeta_{t+1}} f(\theta_{t}) ^2-c^2 h_{t+1} ^2 
\end{align*}
We also have, invoking Assumption \ref{ass:exp_smooth}, that  
$$ \mathbb E_ { \xi_{t+1}} [ g_t^{(\zeta_{t+1})2}]  \leq \max_{k=1,\ldots, p} \mathbb E_ { \xi_{t+1}} [ g_t^{(k)2}] \leq 2 \mathcal L ( f(\theta_{t}) - f^\star ) +  \sigma^2. $$ 
We finally obtain that
\begin{align*}
&\mathbb E_ { \xi_{t+1}} [ f(\theta_{t+1}) -  f^\star ]  \\
& \leq (1 + L \mathcal L \gamma_{t+1}^2)  ( f(\theta_{t}) - f^\star) - \gamma_{t+1}  \nabla_{\zeta_{t+1}} f(\theta_{t}) ^2 /2   +   c^2 \gamma_{t+1}  h_{t+1} ^2 /2+ 
\frac{L}{2} \gamma_{t+1}^2   \sigma^2.
\end{align*}
Apply Robbins-Siegmund Theorem to obtain that almost surely
\begin{align*}
\sum_{t\geq 0} \gamma_{t+1}  \nabla_{\zeta_{t+1}} f(\theta_{t}) ^2 = \sum_{t\geq 0}  \gamma_{t+1}  \| \nabla f(\theta_{t}) \|_\infty ^2 <\infty.
\end{align*}

\noindent
\textbf{Proof of condition (ii).}  Note that from the proof of Theorem 1, we already have $\sum_{t\geq 0}   \gamma _ {t+1} (g_{t}(\omega) -\nabla f(\theta_{t} (\omega)) ) < \infty$, so using that $$\|D(\zeta_{t+1})(g_{t}(\omega) -\nabla f(\theta_{t} (\omega)) )\|_2 \leq \|(g_{t}(\omega) -\nabla f(\theta_{t} (\omega)) )\|_2,$$ we deduce the convergence $\sum_{t\geq 0}   \gamma _ {t+1} D(\zeta_{t+1})(g_{t}(\omega) -\nabla f(\theta_{t} (\omega)) ) < \infty$ which gives, in virtue of Lemma \ref{lemma:key_ps2} the result $\nabla f(\theta_t) \to 0$ almost surely as $t \to +\infty$.
\qed

\medskip
\noindent
\textit{Part (b) gradient weights.} Here we assume that the weights of the coordinate sampling policy are proportional to any norm of the current gradient: $D_t \propto (|\partial_k f(\theta_t)|^q)_{1 \leq k \leq p}$ with $q > 0$. As before, the proof follows from applying Lemma \ref{lemma:key_ps2}. The proof of condition (i) relies on the equivalence of the norms in finite dimension.

\medskip
\noindent
\textbf{Proof of condition (i).} From the proof of Theorem \ref{th:convergence_as_csgd_grad}, we get
\begin{align*}
&\mathbb E_ { \xi_{t+1}} [ f(\theta_{t+1}) -  f^\star ]  \\
& \leq (1 + L \mathcal L \gamma_{t+1}^2)  ( f(\theta_{t}) - f^\star) - \gamma_{t+1}  \nabla_{\zeta_{t+1}} f(\theta_{t}) ^2 /2   +   c^2 \gamma_{t+1}  h_{t+1} ^2 /2+ 
\frac{L}{2} \gamma_{t+1}^2   \sigma^2.
\end{align*}
Taking the expectation with respect to $\zeta_{t+1}$, we get
\begin{align*}
&\mathbb E_ { {t}} [ f(\theta_{t+1}) -  f^\star ]\\
 & \leq (1 + L \mathcal L \gamma_{t+1}^2) (  f(\theta_{t}) - f^\star) - \gamma_{t+1} \sum_{k=1} ^p  d_{t,k} \partial_{k} f(\theta_{t}) ^2 /2   +   c^2 \gamma_{t+1}  h_{t+1} ^2 /2+ 
\frac{L}{2} \gamma_{t+1}^2   \sigma^2.
\end{align*}
Apply Robbins-Siegmund Theorem to obtain $\sum_{t \geq 0} \gamma_{t+1} \nabla f(\theta_t)^\top D_t \nabla f(\theta_t) <  \infty$ almost surely. Now observe that
since $D_t \propto (|\partial_k f(\theta_t)|^q)_{1 \leq k \leq p}$, it means that for all $k=1,\ldots,p$ we have $d_{t,k} \propto |\partial_k f(\theta_t)|^q/ \|\nabla f(\theta_t)\|_q^q$ and
\begin{align*}
\nabla f(\theta_t)^\top D_t \nabla f(\theta_t) 
= \sum_{k=1}^p d_{t,k} \partial_k f(\theta_t)^2 
\propto \sum_{k=1}^p \frac{|\partial_k f(\theta_t)|^q}{\|\nabla f(\theta_t)\|_q^q} \partial_k f(\theta_t)^2 
\propto \frac{\|\nabla f(\theta_t)\|_{q+2}^{q+2}}{\|\nabla f(\theta_t)\|_q^q}.
\end{align*}
All norms are equivalent on $\rset^p$ and using Hölder's inequality we have for $0 < l < q$ that $\|\cdot \|_l \leq p^{1/l - 1/q} \|\cdot\|_q$ so the last term is lower bounded as 
\begin{align*}
\frac{\|\nabla f(\theta_t)\|_{q+2}^{q+2}}{\|\nabla f(\theta_t)\|_q^q}
\geq C \|\nabla f(\theta_t)\|_{q+2}^{2} \quad \text{with } C = p^{-2/(q+2)},
\end{align*}
and again using the equivalence of the norms we get the square of the infinity norm  $\nabla f(\theta_t)^\top D_t \nabla f(\theta_t)  \propto \|\nabla f(\theta_t)\|_{\infty}^2$ which finally proves
\begin{align*}
(i) \ \sum_{t\geq 0}  \gamma_{t+1}  \| \nabla f(\theta_{t}) \|_\infty ^2 <\infty.
\end{align*}

\medskip
\noindent
\textbf{Proof of condition (ii).} It is the same as for \textit{Part (a) maximum gradient}. We deduce the convergence $\sum_{t\geq 0}   \gamma _ {t+1} D(\zeta_{t+1})(g_{t}(\omega) -\nabla f(\theta_{t} (\omega)) ) < \infty$ which gives, in virtue of Lemma \ref{lemma:key_ps2} the result $\nabla f(\theta_t) \to 0$ almost surely as $t \to +\infty$. \qed

\subsection{Proof of Theorem \ref{th:convergence_as_csgd}}

Similarly to the proof of Theorem \ref{th:convergence_as_sgd}, we rely on Lemma \ref{lemma:key_ps3} where $g_{t-1}$ is replaced by $D(\zeta_{t}) g_{t-1}$. Therefore we need to check that, with probability $1$, it holds that 
\begin{align*}
&(i) \ \sum_{t\geq 0} \gamma_{t+1} \| \nabla f(\theta_t(\omega)) \|_2^2 < \infty \quad (ii) \ \sum_{t\geq 0}   \gamma _ {t+1} (D(\zeta_t) g_{t}(\omega) - D_t \nabla f(\theta_{t} (\omega)) ) < \infty.
\end{align*}

\noindent
\textbf{Proof of condition (i).} From the proof of Theorem \ref{th:convergence_as_csgd_grad}, we get
\begin{align*}
&\mathbb E_ { \xi_{t+1}} [ f(\theta_{t+1}) -  f^\star ]  \\
& \leq (1 + L \mathcal L \gamma_{t+1}^2)  ( f(\theta_{t}) - f^\star) - \gamma_{t+1}  \nabla_{\zeta_{t+1}} f(\theta_{t}) ^2 /2   +   c^2 \gamma_{t+1}  h_{t+1} ^2 /2+ 
\frac{L}{2} \gamma_{t+1}^2   \sigma^2.
\end{align*}
Taking the expectation with respect to $\zeta_{t+1}$ and using that $\min _{k=1,\ldots, d} d_{t,k} \geq \beta$ gives
\begin{align*}
&\mathbb E_ { {t}} [ f(\theta_{t+1}) -  f^\star ]\\
 & \leq (1 + L \mathcal L \gamma_{t+1}^2) (  f(\theta_{t}) - f^\star) - \gamma_{t+1} \sum_{k=1} ^p  d_{t,k} \partial_{k} f(\theta_{t}) ^2 /2   +   c^2 \gamma_{t+1}  h_{t+1} ^2 /2+ 
\frac{L}{2} \gamma_{t+1}^2   \sigma^2 \\
&\leq (1 + L \mathcal L \gamma_{t+1}^2) (  f(\theta_{t}) - f^\star) - \gamma_{t+1} \beta \| \nabla f(\theta_{t}) \|_2^2 /2   +   c^2 \gamma_{t+1}  h_{t+1} ^2 /2+ 
\frac{L}{2} \gamma_{t+1}^2   \sigma^2, 
\end{align*}
and Robbins-Siegmund Theorem allows to conclude $\sum_{t \geq 0} \gamma_{t+1} \|  \nabla f(\theta_{t})\|_2^2 <+\infty$.

\medskip
\noindent
\textbf{Proof of condition (ii).} Again, we place ourselves on the event $\Omega_0$ and omit the $\omega$ in notation for ease of clarity. First, note that $\|   g_t\|_2^2 \leq g_t^{(\zeta_{t+1})2} \leq \| g_t\|^2_2 $. As a consequence, $\sum_{t\geq 0} \gamma_{t+1}^2 \expec_t[\| D(\zeta_{t+1})  g_t\|_2^2] \leq  \sum_{t\geq 0} \gamma_{t+1}^2 \expec_t[\| g_t\|^2_2]$ and this last series converges as shown in the proof of Theorem \ref{th:convergence_as_sgd}. Now observe that condition (ii) is satisfied as soon as 
\begin{align*}
&(a) \ \| \sum_{t\geq 0}  \gamma_{t+1}  (  D(\zeta_{t+1})  g_t -  \mathbb E _ t [  D(\zeta_{t+1})  g_t] ) \|_2  <\infty \\
&(b)\ \| \sum_{t\geq 0}  \gamma_{t+1} ( \mathbb E _ t  [  D(\zeta_{t+1})  g_t]  -D_t\nabla f(\theta_t)  ) \|_2  <\infty
\end{align*}
Note that equation (a) involves martingale increments whose quadratic variation satisfies
\begin{align*}
\sum_{t\geq 0} \gamma_{t+1}^2 \mathbb E _t [ \| D(\zeta_{t+1})  g_t -  \mathbb E _ t [  D(\zeta_{t+1})  g_t] \|_2^2]  
&\leq  \sum_{t\geq 0} \gamma_{t+1}^2 \mathbb E _t [ \|D(\zeta_{t+1}) g_t\|^2] < \infty,
\end{align*}
which proves Equation (a). Finally the term in equation (b) is bounded using assumption \ref{cond:biased_grad} and $\|D_t\|_2 \leq 1$. We have $\sum_{t\geq 0}  \gamma_{t+1}^2 \| \expec_t[D(\zeta_{t+1})  g_t]  -D_t\nabla f(\theta_t) \|_2^2 \leq c^2 \sum_{t\geq 0}  \gamma_{t+1}^2 h_t^2 <\infty$ which finally proves condition (ii) and gives, in virtue of Lemma \ref{lemma:key_ps3} that $\nabla f(\theta_t) \to 0$ almost surely as $t \to +\infty$.

\subsection{Proof of Theorem \ref{th:non_as_bound}} \label{sec:non_as}

From the proof of Theorem \ref{th:convergence_as_csgd} and using $\beta$ as a uniform lower bound on $\beta_{t+1}$, we have
\begin{align*}
\expec_t\left[f(\theta_{t+1})-f^\star\right] \leq \left(1 +  L \mathcal{L} \gamma_{t+1}^2 \right)\left[f(\theta_{t})-f^\star\right] - \gamma_{t+1} \beta \| \nabla f(\theta_{t})\|_2^2 +  \frac{\sigma^2 L + c^2}{2} \gamma_{t+1}^2.
\end{align*}
Inject the PL inequality $\| \nabla f(\theta_{t})\|_2^2\geq 2\mu (f(\theta_t) - f(\theta^*) )$ from Assumption \ref{ass:pl_ineq} to have
\begin{align*}
\expec_t\left[f(\theta_{t+1})-f^\star\right] \leq \left(1  -  2\mu \beta \gamma_{t+1} +   L \mathcal{L} \gamma_{t+1}^2 \right)\left[f(\theta_{t})-f^\star\right]  +  \frac{\sigma^2 L + c^2}{2} \gamma_{t+1}^2.
\end{align*}
Define $\delta_t = \expec\left [ f(\theta_t) - f^\star \right]$ to finally obtain the recursion equation
\begin{align*}
\delta_t \leq \left( 1 - 2\mu \beta \gamma_t + L \mathcal{L} \gamma_t^2 \right) \delta_{t-1} + \frac{\sigma^2 L + c^2}{2} \gamma_{t}^2 
\end{align*}
Applying the same result from \citep{moulines2011non} with the family of functions $\varphi_\alpha$ defined by $\varphi_\alpha(t) = \alpha^{-1}(t^\alpha - 1)$ if $\alpha \neq 0$ and $\varphi_\alpha(t) =\log(t)$ if $\alpha=0$ along with the learning rates $\gamma_t = \gamma t^{-\alpha}$. 
\begin{align*}
\delta_t \leq
\left\{\begin{array}{ll}
2 \exp\left( 2 L \mathcal{L} \gamma^2 \varphi_{1-2\alpha}(t)\right) \exp\left(-\frac{\mu \beta \gamma}{4} t^{1-\alpha}\right) \left( \delta_0 + \frac{\sigma^2+2c^2}{2 \mathcal{L}} \right) + \frac{\gamma (\sigma^2 L + 2c^2) }{\mu \beta} t^{-\alpha}  &\text { if } \alpha < 1 \\
2 \exp\left( L \mathcal{L} \gamma^2\right)  \left( \delta_0 + \frac{\sigma^2 + 2c^2}{2 \mathcal{L}} \right) t^{-\mu \beta \gamma} + \left(\frac{\sigma^2L}{2}+c^2\right) \gamma^2 \varphi_{\mu \beta \gamma/2 - 1}(t) t^{-\mu \beta \gamma/2} &\text{ if } \alpha=1 
\end{array}\right.
\end{align*}

\subsection{Proof of Theorem \ref{th:convergence_law}}

Starting from $G_0=(0,\ldots,0)$, the total average gain $G_n$ is updated in a online manner during the exploitation phase and collects all the empirical sums of the gradient gradient estimates as
\begin{equation*} 
G_n =  \frac{1}{nT}\sum_{t=1}^{nT} D_t^{-1} D(\zeta_{t+1})g(\theta_t,\xi_{t+1}), \qquad \expec\left[G_n \right ] =  \frac{1}{nT}\sum_{t=1}^{nT} \nabla f(\theta_t).
\end{equation*} 
The goal is to show that $G_n \to 0$ using martingale properties. Thanks to Theorem \ref{th:convergence_as_csgd}, we have the almost sure convergence $\theta_t \to \theta^*$ which gives, since $\theta \mapsto \nabla f(\theta)$ is continuous, that $\nabla f(\theta_t) \to 0$ almost surely. Applying Cesaro's Lemma, it holds that $\expec\left[G_n \right ] \to 0$. It is enough to consider the difference $\left(G_n^{(k)}  - \expec\left[G_n^{(k)}  \right ] \right)$ for each $k\in \llbracket 1,p \rrbracket$. Introducing the martingale increments
\begin{align*}
\Delta_ {t+1}^{(k)}  =  \frac{   g (\theta_t , \xi_{t+1} )^{(k)} }{d_t^{(k)}} \mathds 1 _{\{\zeta_{t+1} = k\} }   -   \partial _ k  f(\theta_t) , \qquad \expec\left[\Delta_ {t+1}^{(k)}   | \mcf_t \right]=0.
\end{align*}
It remains to show that, with probability $1$,
\begin{align*}
G_n^{(k)}  - \expec\left[G_n ^{(k)}  \right ] =  \frac{1}{nT} \sum_{t=1}^{nT}  \Delta_ {t+1} ^{(k)}  \to 0 .
\end{align*}
Or equivalently, that, for each coordinate $k\in \llbracket 1,p \rrbracket$
\begin{align}\label{eq_final}
\sum_{t=1}^{nT}  \Delta_ {t+1}^{(k)}  = o (n)  .
\end{align}
The latter being a sum of martingale increments, we are in position to apply the strong law of large numbers for martingales which can be find as Assertion 2 of Theorem 1.18 in \citep{bercu2015concentration}. 
Using Assumption \ref{ass:exp_smooth}, there exist $0 \leq \mathcal{L},\sigma^2 <\infty$ such that almost surely
\begin{align*}
\forall t \in \nset,  \quad \expec\left[ (g(\theta_t,\xi_{t+1} )^{(k)} ) ^2 | \mcf_t \right] 
\leq 2 \mathcal{L} \left( f(\theta_{t}) - f^\star\right) + \sigma^2.
\end{align*}
Using the almost sure convergence $\theta_t \to \theta^\star$, we deduce that there is exist a compact set $K$ which contains the sequence of iterates $(\theta_t)_{t \in \nset}$ and using that $f$ is continuous gives the upper bound
\begin{align*}
\forall k\in \llbracket 1,p  \rrbracket \quad \expec\left[ (g(\theta_t,\xi_{t+1})^{(k)})^2 | \mcf_t \right] \leq  M = 2\mathcal L \sup_{\theta\in K} (f(\theta) - f(\theta^*) ) + \sigma^2.
\end{align*}
Hence, the quadratic variation is bounded as follows
\begin{align*}
\sum_{t=1}^{nT}  \expec \left[  (\Delta_ {t+1}^{(k)})^2 | \mcf_{t} \right]&\leq 
\sum_{t=1}^{nT}  \expec \left[ \left(  \frac{    g(\theta_t,\xi_{t+1})^{(k)}  }{ d_t^{(k)}  }\right) ^2  | \mcf_{t} \right] \\
&\leq  (p / \lambda )^2 \sum_{t=1}^{nT}   \expec [   (g(\theta_t,\xi_{t+1})^{(k)})^2 | \mcf_{t} ] \\
&\leq (p / \lambda )^2 {nT}  M .
\end{align*}
Equation (\ref{eq_final}) follows from applying the previously mentioned law of large number.

\section{Additional Results} \label{sec:add_results}

\subsection{Almost sure convergence under stronger assumptions} \label{subsec:cv_stronger_assumptions}

Similary to \citet{gadat2018stochastic}, we consider some stronger assumptions where the function $f$ is coercive and there exists a unique stationary point $\theta^\star$. In such framework, the sequences of iterates $(\theta_t)_{t \geq 0}$ obtained by both SGD and SCGD satisfy $\theta_t \to \theta^\star$ almost surely as $t \to +\infty$.

\medskip
\noindent
\textbf{$f$ is coercive and $\{\theta \in \rset^p: \nabla f(\theta) = 0\} = \{ \theta^\star\}$}.  Following the proofs of Theorems \ref{th:convergence_as_sgd} and \ref{th:convergence_as_csgd}, we may apply Robbins-Siegmund Theorem. There exists an event $\Omega_0 \subset \Omega$ such that, $\prob(\Omega_0)=1$ and for every $\omega \in \Omega_0$, $\limsup_{t} f(\theta_t(\omega)) < \infty$ and the series $\sum_t \eta_{t+1} \|\nabla f(\theta_t(\omega)) \|_2^2$ converges (where $\eta_t = \gamma_t$ for SGD and $\eta_t = \gamma_t \beta_t$ for SCGD). Since $\lim_{\|\theta\|\to \infty} f(\theta) = \infty$, we deduce that for every $\omega \in \Omega_0$, the sequence $(\theta_t(\omega))_{t \geq 0}$ is bounded in $\rset^p$. Therefore the limit set $\chi_{\infty}(\omega)$ (set of accumulation points) of the sequence $(\theta_t(\omega))$ is non-empty. The convergence of the series $\sum_t \eta_{t+1} \|\nabla f(\theta_t(\omega)) \|_2^2 < \infty$ along with the condition $\sum_t \eta_{t+1} = +\infty$ only implie that : $\liminf_{t \to \infty} \|\nabla f(\theta_{t}(\omega))\|_2^2 = 0, \quad \prob-a.s.$ \\
Hence, since $\theta \mapsto \nabla f(\theta)$ is continuous, there exits a limit point $\theta_{\infty}(\omega) \in \chi_{\infty}(\omega)$ such that $\|\nabla f(\theta_{\infty}(\omega))\|_2^2 = 0$, \textit{i.e.}, $\nabla f(\theta_{\infty}(\omega))=0$. Because the set of solutions $\{ \theta \in \rset^p, \nabla f(\theta) = 0 \}$ is reduced to the singleton $\{ \theta^{\star} \}$, we have $\theta_{\infty}(\omega) = \theta^\star$. Since $(f(\theta_t(\omega)))$ converges, it implies that $\lim_t f(\theta_t(\omega)) = f^\star$ and for every limit point $\theta \in \chi_{\infty}(\omega)$, we have $f(\theta) = f^\star$. Since the set $\{ \theta \in \rset^p, f(\theta) = f^\star \}$ is equal to $\{ \theta^\star \}$, the limit set $\chi_{\infty}(\omega)$ is also reduced to $\{ \theta^{\star} \}$.

\subsection{Almost sure convergence of MUSKETEER } \label{sec:corollary}

By definition, we have for all $k\in \llbracket 1,p \rrbracket$,
\begin{equation*}
d_{t+1}^{(k)} = (1-\lambda_t) \varphi(G_t)^{(k)} + \lambda_t \frac{1}{p}
\end{equation*}
implying that $ \beta_{t+1} = \min_{k\in \llbracket 1,p \rrbracket} d_{t}^{(k)}  \geq \lambda_t / p$.  As a consequence, as soon as $\sum_{t\geq 1} \lambda_t \gamma_t = +\infty$, the assumption $\sum_{t\geq 1} \beta_t \gamma_t = +\infty$ is satisfied. Applying Theorem \ref{th:convergence_as_csgd} we obtain the almost sure convergence of MUSKETEER. The condition $\sum_{t\geq 1} \lambda_t \gamma_t = +\infty$ is easily satisfied with a fixed value $\lambda_t \equiv \lambda$ in the mixture update and one can also use a slowly decreasing sequence, e.g. $\lambda_t = 1/\log(t)$.

\subsection{Regret analysis in the convex case} \label{sec:reg_analysis}

In order to better understand the benefits of the adaptive sampling strategies over standard uniform sampling, let us consider a particular setting where the objective function $f$ is convex. The following proposition available in \citet{namkoong2017adaptive} presents a regret analysis which is useful for interpretability.

\begin{proposition}[Regret analysis for convex $f$ and unbiased estimates] \label{prop:regret} Assume that $f$ is convex and consider the sequence of iterates obtained by $\theta_{t+1} = \theta_t - \gamma D^{-1}_t D(\zeta_{t+1}) g_t$ with constant step size $\gamma >0$. We have
\begin{align*}
\expec\left[ f\left(\frac{1}{T} \sum_{t=1}^T \theta_t\right) - f(\theta^\star) \right] \leq \frac{\| \theta^\star\|^2}{2 \gamma T} + \frac{\gamma}{2T} \sum_{t=1}^T \expec\left[ \sum_{k=1}^p \frac{|\partial_k f(\theta_t)|^2}{d_{t}^{(k)}}\right].
\end{align*}
\end{proposition}

\begin{proof}
Assume that the objective $f$ is convex and consider the average estimate $\bar \theta_T = \frac{1}{T} \sum_{t=1}^T \theta_t$. 
along with the following quantity: $S(f,\hat \theta) = \expec[f(\hat \theta)] - f^\star$. 
Using convexity we have on the one hand $f(\theta_t)-f^\star \leq \langle \theta_t - \theta^\star, \nabla f(\theta_t) \rangle$ and on the other hand
\begin{align*}
f(\bar \theta_T)-f^\star \leq \frac{1}{T} \sum_{t=1}^T \left(f(\theta_t)-f^\star \right)
\end{align*}
which give together the following upper bound
\begin{align*}
f(\bar \theta_T)-f^\star \leq  \frac{1}{T} \sum_{t=1}^T \langle \theta_t - \theta^\star, \nabla f(\theta_t) \rangle.
\end{align*}
Using an unbiased gradient estimate $v_t$, \textit{i.e.} $\expec_t[v_t]=\nabla f(\theta_t)$, we can write
\begin{align*}
\expec[f(\bar \theta_T)]-f^\star \leq  \expec\left[\frac{1}{T} \sum_{t=1}^T \langle \theta_t - \theta^\star, \expec_t[v_t]) \rangle \right].
\end{align*}
The term in the expectation is bounded using Lemma \ref{lemma:regret} with $v_t = D_t^{-1} D(\zeta_{t+1})g_t$ as
\begin{align*}
\frac{1}{T}\sum_{t=1}^T \langle \theta_t-\theta^\star,v_t \rangle \leq \frac{\| \theta^\star\|^2}{2 \gamma T} + \frac{\gamma}{2T} \sum_{t=1}^T \|D_t^{-1} D(\zeta_{t+1})g_t\|^2.
\end{align*}
Take the expectation on both side to control the regret as 
\begin{align*}
S(f,\bar \theta_T) \leq \frac{\| \theta^\star\|^2}{2 \gamma T} + \frac{\gamma}{2T} \sum_{t=1}^T \expec\left[ \sum_{k=1}^p \frac{|\partial_k f(\theta_t)|^2}{d_{t}^{(k)}}\right].
\end{align*}

\end{proof}
The term in expectation should be minimized with respect to the probability weights $d_{t}^{(k)}$. Intuitively, in order to maintain the overall sum as small as possible, the large gradient coordinates should be sampled more often, \textit{i.e.} we would like to have $d_{t}^{(k)}$ large whenever $|\partial_k f(\theta_t)|^2$ is large. This is in line with the framework of coordinate smoothness discussed in Remark \ref{rem:coord_smooth} and the work of \citet{allen2016even}.

\noindent (Uniform Coordinate Sampling) For all $k \in \llbracket 1,p \rrbracket$, we have $d_{t}^{(k)} = 1/p$ so that
\begin{align*}
\frac{1}{T} \sum_{t=1}^T \expec\left[ \sum_{k=1}^p \frac{|\partial_k f(\theta_t)|^2}{d_{t}^{(k)}}\right] = \frac{p}{T} \sum_{t=1}^T \expec\left[ \sum_{k=1}^p |\partial_k f(\theta_t)|^2\right] = \frac{p}{T} \sum_{t=1}^T \expec\left[ \|\nabla f(\theta_t)\|^2\right].
\end{align*}
(MUSKETEER) For all $k \in \llbracket 1,p \rrbracket$, we have $d_{t}^{(k)} = (1-\lambda_{t-1})\varphi(G_{t-1})^{(k)} + \lambda_{t-1}/p$ so that
\begin{align*}
\frac{1}{T} \sum_{t=1}^T \expec\left[ \sum_{k=1}^p \frac{|\partial_k f(\theta_t)|^2}{d_{t}^{(k)}}\right] = \frac{p}{T} \sum_{t=1}^T \expec\left[ \sum_{k=1}^p \frac{|\partial_k f(\theta_t)|^2}{(1-\lambda_{t-1})p\varphi(G_{t-1})^{(k)} + \lambda_{t-1}}\right],
\end{align*}
where the denominator is stricly larger than $1$ for all the coordinates associated to large gains. Indeed, let $k \in \llbracket 1,p \rrbracket$ the index of such coordinate. Since it is a rewarding coordinate, the normalizing step implies that $\varphi(G_{t-1})^{(k)} > 1/p$ and $(1-\lambda_{t-1})p\varphi(G_{t-1})^{(k)} + \lambda_{t-1} > 1$. This property translates the adaptive nature of the probability weights used in the MUSKETEER strategy.

\subsection{Auxiliary Results} \label{sec:rs_th}

\begin{theorem}\citep{robbins1971convergence} \label{th:rs} Consider a filtration $\left(\mathcal{F}_{n}\right)_{n \geq 0}$ and four sequences of random variables$\left(V_{n}\right)_{n \geq 0},\left(W_{n}\right)_{n \geq 0}, \left(a_{n}\right)_{n \geq 0}$ and $\left(b_{n}\right)_{n \geq 0}$ that are adapted and non-negative. Assume that almost surely $\sum_{k} a_k <\infty$ and $\sum_{k} b_k <\infty$. Assume moreover that $\expec\left[V_0\right] < \infty$ and $\forall n \in \nset: \expec[V_{n+1} | \mcf_n ] \leq (1 + a_n) V_{n} - W_{n} + b_{n}.$ Then it holds 
\begin{align*}
(a) \ \sum_{k} W_k <\infty  \ a.s.  \qquad (b) \ V_{n} \stackrel{a.s.}{\longrightarrow} V_{\infty}, \expec\left[V_{\infty}\right]<\infty. \qquad (c) \ \sup _{n \geq 0} \expec\left[V_{n}\right]<\infty.
\end{align*}
\end{theorem}

\begin{lemma}\label{lemma:regret} Let $\theta_1,\ldots,\theta_T$ be an arbitrary sequence of vectors. Any algorithm with initialization $\theta_1=0$ and update rule $\theta_{t+1} = \theta_t - \gamma v_t$ satisfies
\begin{align*}
\sum_{t=1}^T \langle \theta_t-\theta^\star,v_t \rangle \leq \frac{\| \theta^\star\|^2}{2 \gamma} + \frac{\gamma}{2} \sum_{t=1}^T \|v_t\|^2.
\end{align*}
In particular, for $B,\rho >0$, if we have $\|v_t\| \leq \rho$ and we set $\gamma=\sqrt{B^2/(\rho^2 T)}$ then for every $\theta^\star$ with $\|\theta^\star\| \leq B$, we have $T^{-1} \sum_{t=1}^T \langle \theta_t-\theta^\star,v_t \rangle \leq B \rho/\sqrt{T}.$
\end{lemma}

\begin{proof}
Using algebraic manipulations (completing the square), we obtain:

\begin{align*}
\left\langle\theta_{t}-\theta^{\star}, v_{t}\right\rangle &=\frac{1}{\gamma}\left\langle\theta_{t}-\theta^{\star}, \gamma v_{t}\right\rangle \\
&=\frac{1}{2 \gamma}\left(-\left\|\theta_{t}-\theta^{\star}-\gamma v_{t}\right\|^{2}+\left\|\theta_{t}-\theta^{\star}\right\|^{2}+\gamma^{2}\left\|v_{t}\right\|^{2}\right) \\
&=\frac{1}{2 \gamma}\left(-\left\|\theta_{t+1}-\theta^{\star}\right\|^{2}+\left\|\theta_{t}-\theta^{\star}\right\|^{2}\right)+\frac{\gamma}{2}\left\|v_{t}\right\|^{2}
\end{align*}
where the last equality follows from the definition of the update rule. Summing the equality over $t,$ we have
$$
\sum_{t=1}^{T}\left\langle\theta_{t}-\theta^{\star}, v_{t}\right\rangle=\frac{1}{2 \gamma} \sum_{t=1}^{T}\left(-\left\|\theta_{t+1}-\theta^{\star}\right\|^{2}+\left\|\theta_{t}-\theta^{\star}\right\|^{2}\right)+\frac{\gamma}{2} \sum_{t=1}^{T}\left\|v_{t}\right\|^{2}
$$
The first sum on the right-hand side is a telescopic sum that collapses to
$\left\|\theta_{1}-\theta^{\star}\right\|^{2}-\left\|\theta_{T+1}-\theta^{\star}\right\|^{2}$. Then we have
\begin{align*}
\sum_{t=1}^{T}\left\langle\theta_{t}-\theta^{\star}, v_{t}\right\rangle &=\frac{1}{2 \gamma}\left(\left\|\theta_{1}-\theta^{\star}\right\|^{2}-\left\|\theta_{T+1}-\theta^{\star}\right\|^{2}\right)+\frac{\gamma}{2} \sum_{t=1}^{T}\left\|v_{t}\right\|^{2} \\
& \leq \frac{1}{2 \gamma}\left\|\theta_{1}-\theta^{\star}\right\|^{2}+\frac{\gamma}{2} \sum_{t=1}^{T}\left\|v_{t}\right\|^{2} \\
&=\frac{1}{2 \gamma}\left\|\theta^{\star}\right\|^{2}+\frac{\gamma}{2} \sum_{t=1}^{T}\left\|v_{t}\right\|^{2}
\end{align*}
where the last equality is due to the definition $\theta_{1}=0 .$ This proves the first part of the lemma. The second part follows by upper bounding $\left\|\theta^{\star}\right\|$ by $B,\left\|v_{t}\right\|$ by $\rho,$ dividing by $T,$ and plugging in the value of $\gamma$
\end{proof}

\newpage
\section{Illustrative Example (stochastic first order)} \label{app:illustrative}

We perform a comparison on a simple example in dimension $p=2$ with the functions $f(x,y) = (x^2 + y^2)/2$ and $h(x,y) = x^2/2$. Note that the function $h$ only depends on the first coordinate and an adaptive coordinate descent method should favor this direction. Figure \ref{fig:compare_2d} presents the optimization paths of the different methods: SGD, Uniform and MUKSTEER. With the function $f$ which does not present any particular design or favorable descent direction, the Uniform and Musketeer policies perform as good as classical SGD. More interestingly, when dealing with the function $h$, our method MUSKETEER (red) finds that the horizontal direction associated to axis $(Ox)$ is the relevant one for optimization. After collecting some information during the exploration phase, the probability weights got updated to favor the horizontal direction, leading to a faster convergence. For a visual demonstration of these optimization paths, please refer to the mp$4$-files \textit{optimize\_f.mp4} and \textit{optimize\_h.mp4} available in the supplementary material\footnote{https://github.com/RemiLELUC/SCGD-Musketeer}.

\begin{figure}[h]
\centering
\includegraphics[scale=0.35]{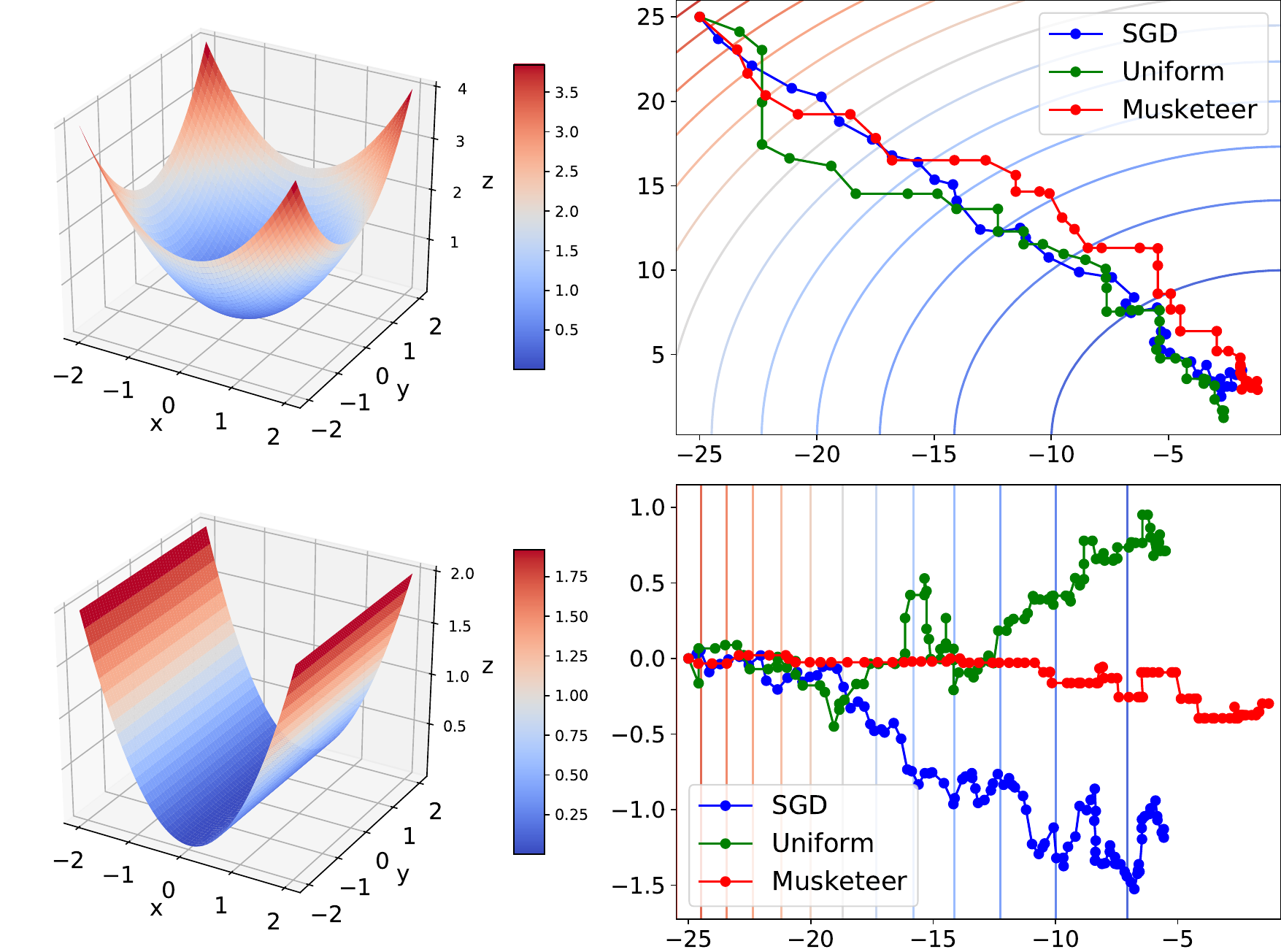}
\caption{Comparison of SGD/Uniform/Musketeer on simple 2D-examples}
\label{fig:compare_2d}
\end{figure}

\section{Numerical Experiments Details} \label{app:details}

\subsection{Regularized linear models}
We consider the ERM paradigm with linear models, namely regularized regression problems with objectives of the form $f(\theta) = (1/n) \sum_{i=1}^n f_i(\theta) + \mu \|\theta\|^2$. Similarly to \citep{namkoong2017adaptive}, we endow the data matrix $X$ with a block structure. The columns are drawn as $X[:,k] \sim \mathcal{N}(0,\sigma_k^2 I_n)$ with $\sigma_k^2 = k^{-\alpha}$ for all $k\in \llbracket 1,p \rrbracket$. The parameters are set to $n=10,000$ samples in dimension $p=250$ with an exploration size equal to $T = \lfloor \sqrt{p} \rfloor = 15$. The regularization parameter is set to the classical value $\mu=1/n$. We update the parameter vector with the optimal learning rate $\gamma_k = \gamma/(k+k_0)$ in the experiments.  Other learning rates in the framework of stochastic first order methods are considered in Appendix \ref{sec:more_num}. \\
\textbullet \ (zeroth-order) For the Ridge regression, we set $\gamma=3,k_0=10$ and for the logistic regession $\gamma=10,k_0=5$. The gradient estimate $g$ is computed using queries of a function $f_i$ where $i \sim \mathcal{U}(\llbracket 1,n \rrbracket)$. We use the $\ell_1$-reweighting with $\lambda_t = 1/\log(t)$ or softmax  with $\lambda_n \equiv 0.5$, which both satisfy Assumption \ref{ass:lr}.\\
\textbullet \ (first order) The learning rate is equal to $\gamma_k = 1/k$ ($\gamma=1,k_0=0$). The gradient estimate $g$ is computed using mini-batches of size $8$. The weighting parameter $\eta>0$ in the softmax part of the probability weights is set to $\eta=1$ and the parameter $\lambda$ in  Equation \eqref{eq:probas_update} is chosen as $\lambda_t = 1/\log(t)$ which satisfies the extended Robbins-Monro condition \ref{ass:lr}.

\subsection{Neural Networks}

\textbf{Dataset description and parameter configuration.} The three datasets in the experiments are popular publicly available deep learning datasets. The underlying machine learning task is the one of multi-label classification. 

\textbullet \  \textbf{MNIST} \citep{deng2012mnist}: a database of handwritten digits with a training set of 60,000 examples and a test set of 10,000 examples. The digits have been size-normalized and centered in a fixed-size image. The original black and white (bilevel) images from NIST were size normalized to fit in a 20x20 pixel box while preserving their aspect ratio. The resulting images contain grey levels as a result of the anti-aliasing technique used by the normalization algorithm. The images were centered in a 28x28 image by computing the center of mass of the pixels, and translating the image so as to position this point at the center of the 28x28 field. Each training and test example is assigned to the corresponding handwritten digit between $0$ and $9$.

\textbullet \ \textbf{Fashion-MNIST} \citep{xiao2017fashion}: a dataset of Zalando's article images, composed of a training set of 60,000 examples and a test set of 10,000 examples. Each example is a 28x28 grayscale image, associated with a label from 10 classes. It shares the same image size and structure of training and testing splits as the MNIST database. Each training and test example is assigned to one of the following labels: T-shirt/top (0); Trouser (1); Pullover (2); Dress (3); Coat (4); Sandal (5); Shirt (6); Sneaker (7); Bag (8); Ankle boot (9). 

\textbullet \ \textbf{Kuzushiji-MNIST}: This dataset is a drop-in replacement for the MNIST dataset (28x28 grayscale, 70,000 images), provided in the original MNIST format as well as a NumPy format. Since MNIST is restricted to 10 classes, one character here represents each of the 10 rows of Hiragana when creating Kuzushiji-MNIST.

\textbullet \ \textbf{CIFAR10} \citep{krizhevsky2009learning}: The CIFAR-10 dataset consists of $60,000$ $32 \times 32$ colour images in $10$ classes, with $6,000$ images per class. There are $50,000$ training images and $10,000$ test images. The dataset is divided into five training batches and one test batch, each with $10,000$ images. The test batch contains exactly $1,000$ randomly-selected images from each class. The training batches contain the remaining images in random order, but some training batches may contain more images from one class than another. Between them, the training batches contain exactly $5,000$ images from each class. Each training and test example is assigned to one of the following labels: airplane (0); automobile (1); bird (2); cat (3); deer (4); dog (5); frog (6); horse (7); ship (8); truck (9).

\begin{figure}[h]
  \centering
  \subfigure[MNIST]{
  \includegraphics[scale=0.35]{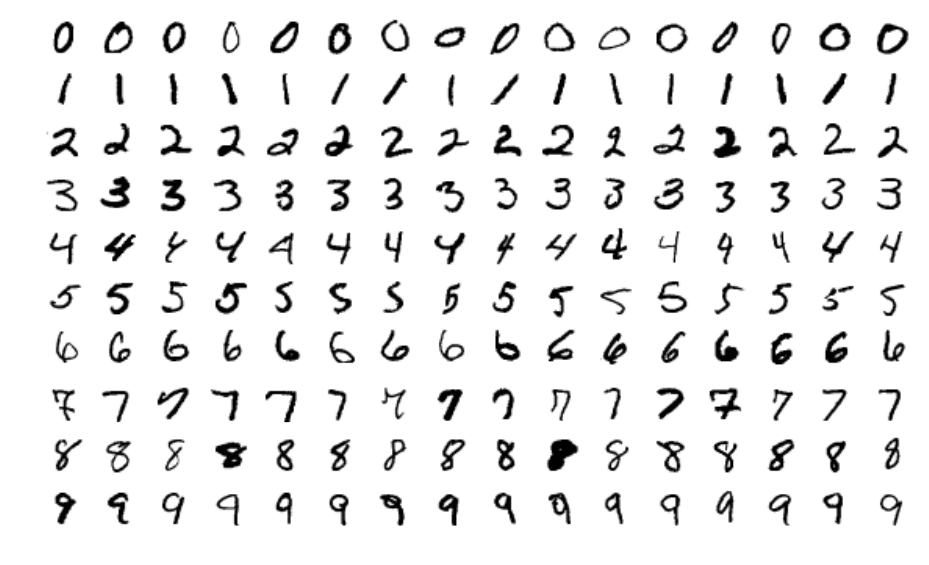}\label{fig:ex_mnist}}
  \subfigure[Fashion-MNIST]{
  \includegraphics[scale=0.3]{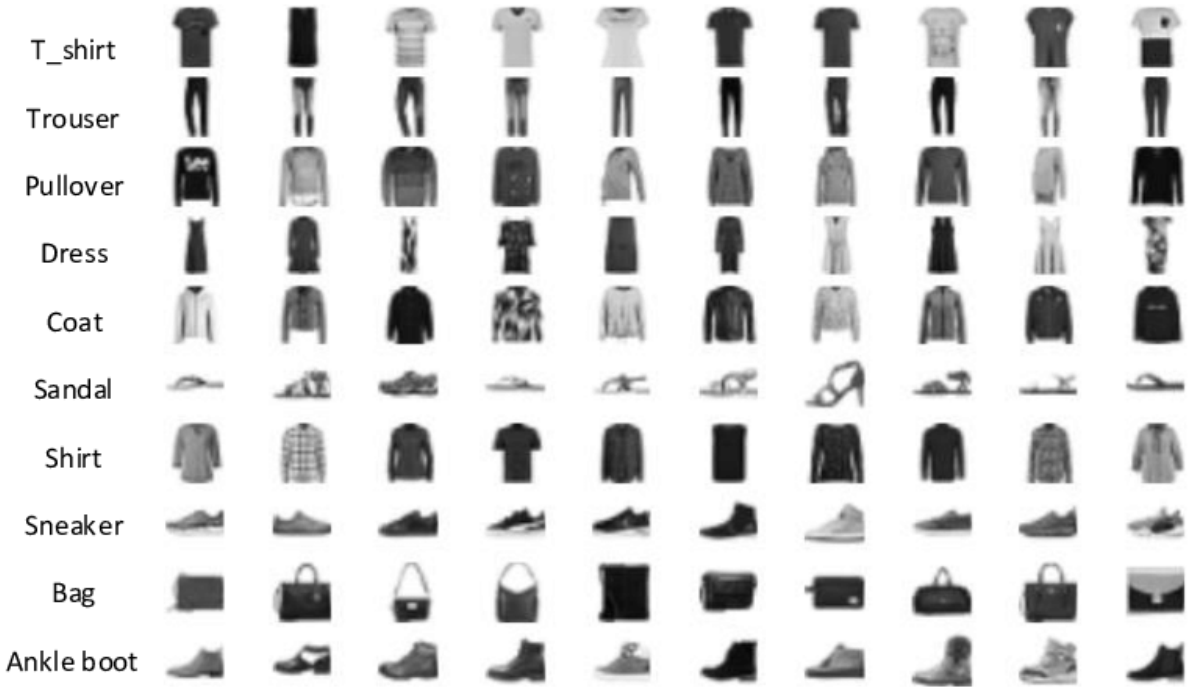}\label{fig:ex_fash}}
   \subfigure[CIFAR10]{
  \includegraphics[scale=0.23]{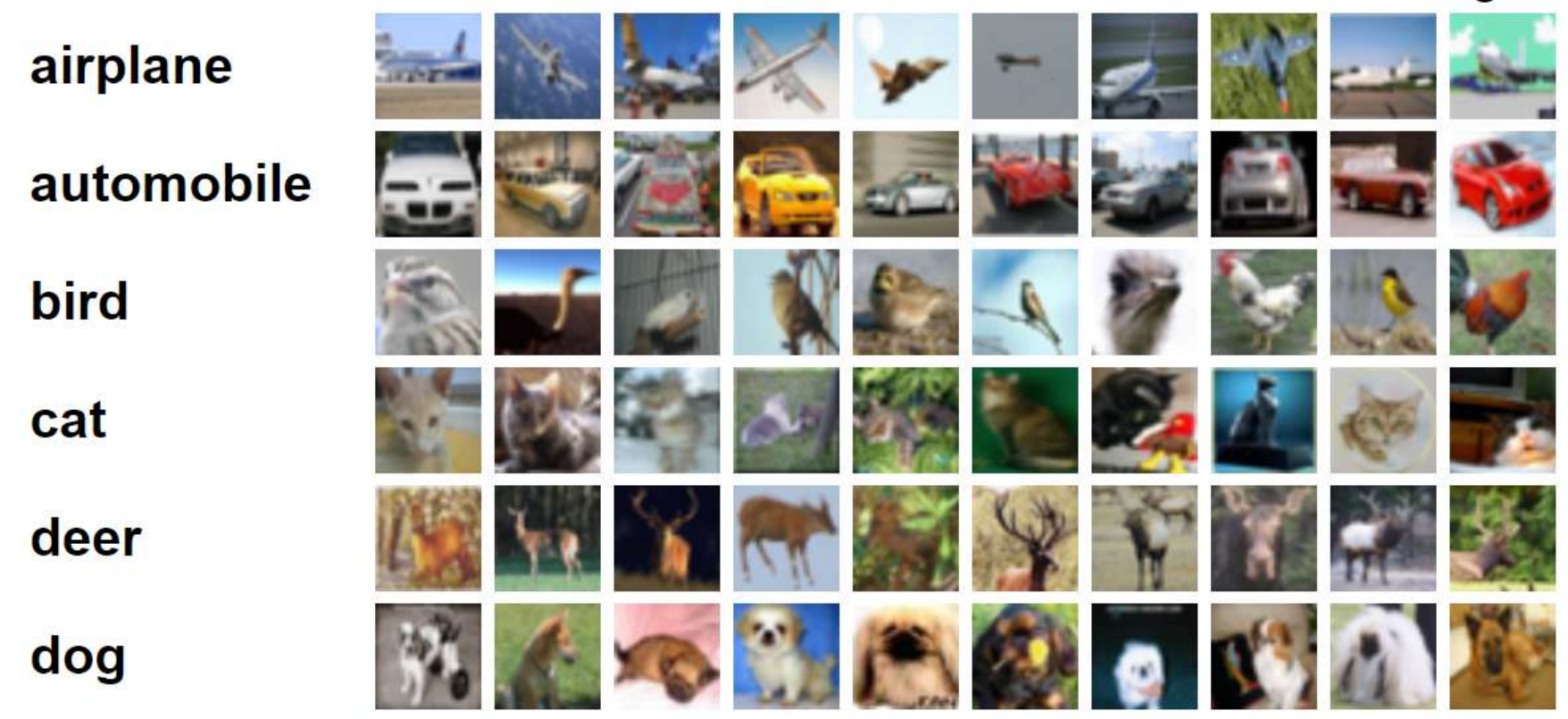}\label{fig:ex_cifar}}
  \subfigure[K-MNIST]{
  \includegraphics[scale=0.35]{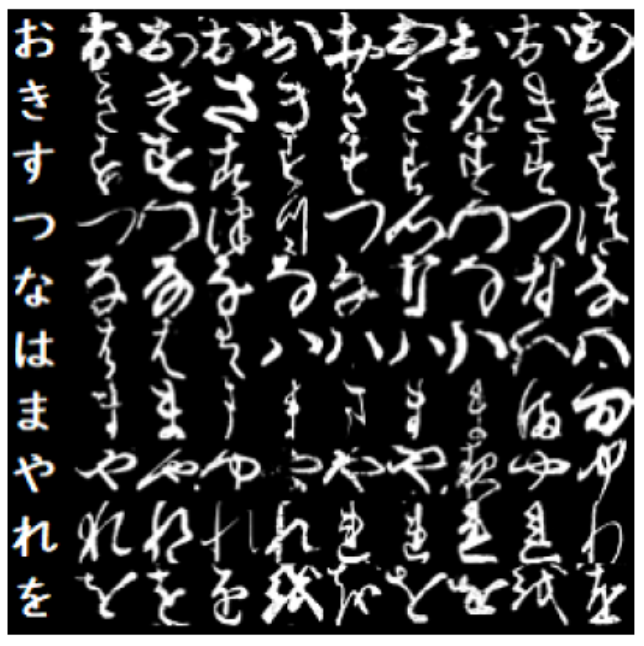}\label{fig:ex_mnist}}

  \caption{Samples for Mnist, Fashion-Mnist, K-Mnist and CIFAR-10.}
\label{fig:deep_examples}
\end{figure}

Two different neural networks are used in the experiments: one with linear layers for MNIST, Fashion-MNIST, K-MNIST another one with convolutional layers for CIFAR10. For the first network, the total number of parameters is $p=55,050$. For the second network,the dimension is $p=64,862$. In both cases, the exploration size is $T = \lfloor \sqrt{p} \rfloor$. 
In the experiments with stochastic first order methods, we use batches of coordinates with $m=p/10$.

\subsection{Hyperparameters and Hardware.}
\textbf{Hyperparameters.} When training neural networks with linear layers, we use: batch\_size = 32; input\_size = 28*28; hidden\_size = 32; output\_size = 64, along with the parameters \\
\textbullet \ (zeroth-order) $\gamma$ = 10 (Mnist and Fashion-Mnist) $\gamma$=15 (Kmnist); h = 0.01; $\ell_1$ normalization with $\lambda_n = 1/\log(n)$; softmax normalization with $\lambda_n \equiv 0.2$ and $\eta=5$ . \\
\textbullet \ (first order) $\gamma$ = 0.01 (Mnist,Fashion-Mnist,Cifar10); normalization = softmax with $\eta \in \{1,2,10\}$; $\lambda_t = 0$ (only exponential weights). 

\medskip
\noindent \textbf{Hardware.} The experiments of linear models are run using a processor Intel Core i7-10510U CPU 1.80GHz $\times$ 8; the neural networks are trained using GPU from Google Colab (GPU: Nvidia K80 / T4; GPU Memory: 12GB/16GB; GPU Memory Clock:	0.82GHz/1.59GHz; Performance:	4.1 TFLOPS / 8.1 TFLOPS)

\newpage
\textbf{ZO Neural Networks with $\ell_1$ normalization.}

\begin{figure}[h]
  \centering
  \subfigure[MNIST]{
  \includegraphics[scale=0.33]{graph/loss_mnist.pdf}} \
  \subfigure[Fashion-MNIST]{
  \includegraphics[scale=0.33]{graph/loss_fash.pdf}} \
  \subfigure[KMNIST]{
  \includegraphics[scale=0.33]{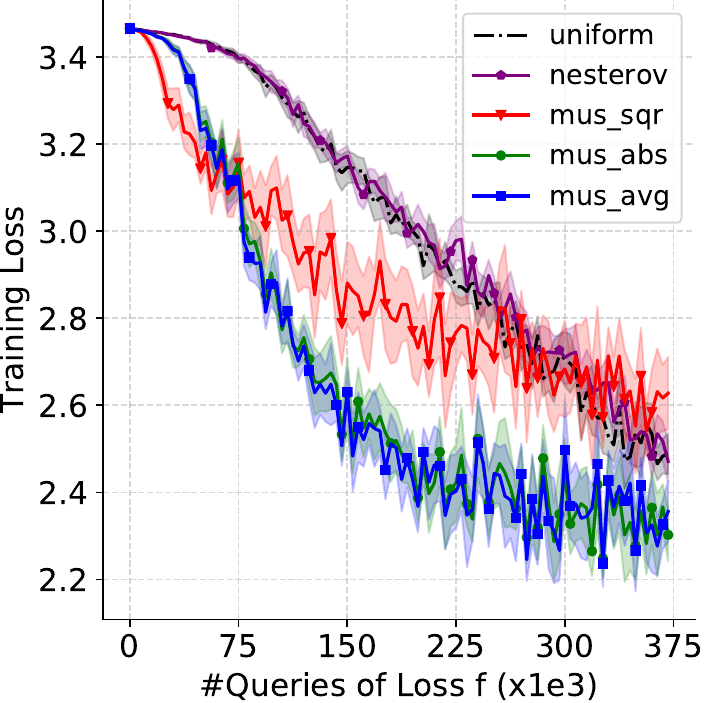}}
  \caption{Training Loss ZO Neural Networks with $\ell_1$ normalization.}
\end{figure}

\textbf{ZO Neural Networks, Comparison of $\ell_1$ and Softmax normalizations.}
\begin{figure}[h]
  \centering
  \subfigure[MNIST-$\ell_1$]{
  \includegraphics[scale=0.35]{graph/loss_mnist.pdf}} \
  \subfigure[MNIST-Exp]{
  \includegraphics[scale=0.35]{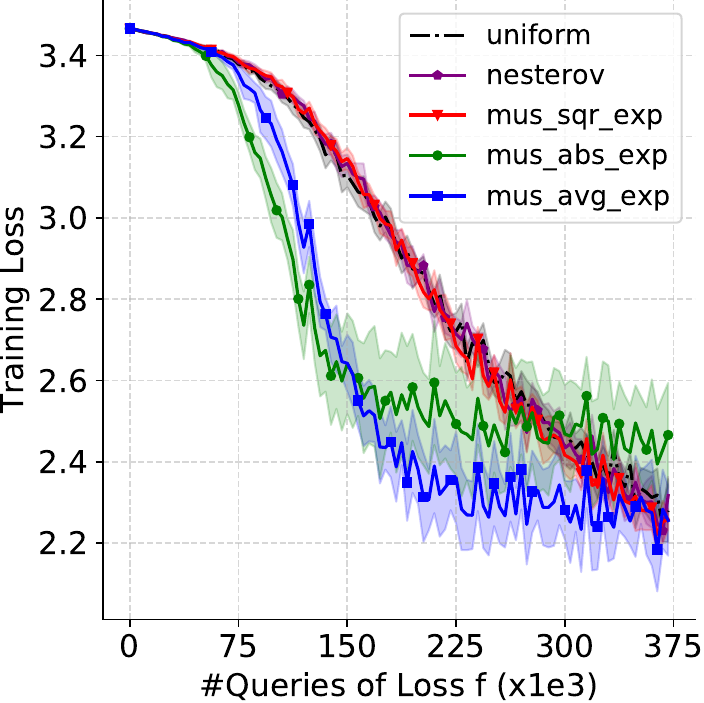}} \\
  \subfigure[Fashion-$\ell_1$]{
  \includegraphics[scale=0.35]{graph/loss_fash.pdf}} \
  \subfigure[Fashion-Exp]{
  \includegraphics[scale=0.35]{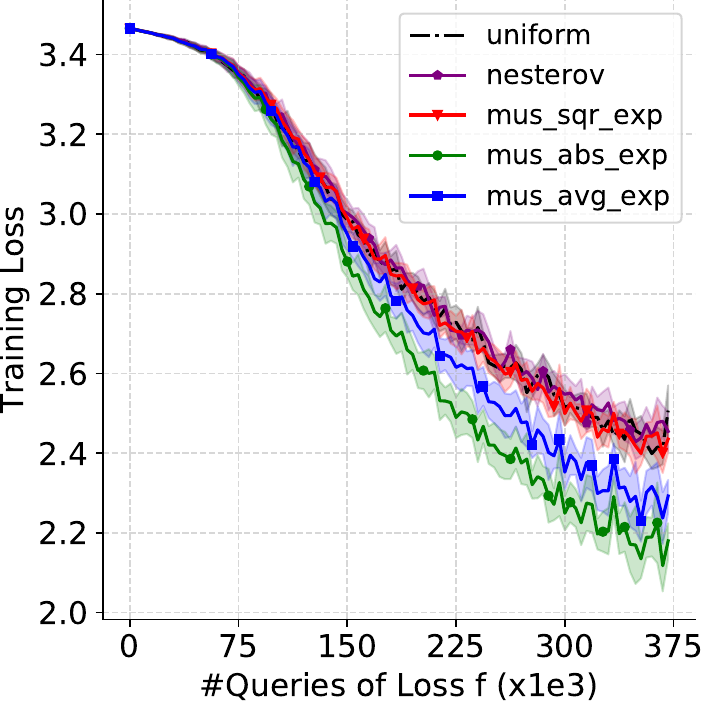}}
  \caption{Training Loss ZO Neural Networks with $\ell_1$ and Softmax normalizations.}
\end{figure}
\newpage
\section{Numerical Experiments with stochastic first order methods} \label{sec:simus_first}

In this section, we empirically validate the SCGD framework by running MUSKETEER and competitors on synthetic and real datasets problems with stochastic first order methods. First, we focus on ridge regression and regularized logistic regression problems adopting the data generation process of \citep{namkoong2017adaptive} in which the covariates exhibit a certain block structure. Second, MUSKETEER is employed to train different neural networks models on real datasets for multi-label classification task. From a practical point of view, the optimization procedure is implemented through a PyTorch optimizer which allows an easy deployment and integration.

\medskip
\noindent \textbf{Methods in competition.} 
The set of methods in competition is restricted to stochastic coordinate-based methods along with standard SGD playing the role of the baseline. This choice allows an honest comparison as the parameter tuning can be the same for all methods. MUSKETEER is implemented according to Section \ref{sec:main_algo} with an exploration size $T = \lfloor \sqrt{p} \rfloor $ and different values of $\eta$ are used to feed the discussion on the adaptiveness. The method UNIFORM stands for the uniform coordinate sampling policy in SCGD. The method ADAPTIVE is the importance sampling based method described in Remark \ref{rk:IS}. This method is no longer part of the SCGD framework and corresponds to the one developed in \citep{wangni2018gradient}. Among the different methods, MUSKETEER is the only one exhibiting a bias when generating gradients. In all cases, $\theta_0 = (0,\ldots, 0)^\top \in \rset^p$ and the optimal SGD learning rate $\gamma_k = 1/k$ is used. For a fair comparison of SGD against SCGD, we normalize the number of passes over the coordinates: one SGD step updates the $p$ coordinates of $\theta$ so we allow to take $p$ steps for the coordinate-based methods in the mean time. 

\begin{figure}[h]
  \centering
  \subfigure[Ridge $\alpha=5$]{
  \includegraphics[scale=0.28]{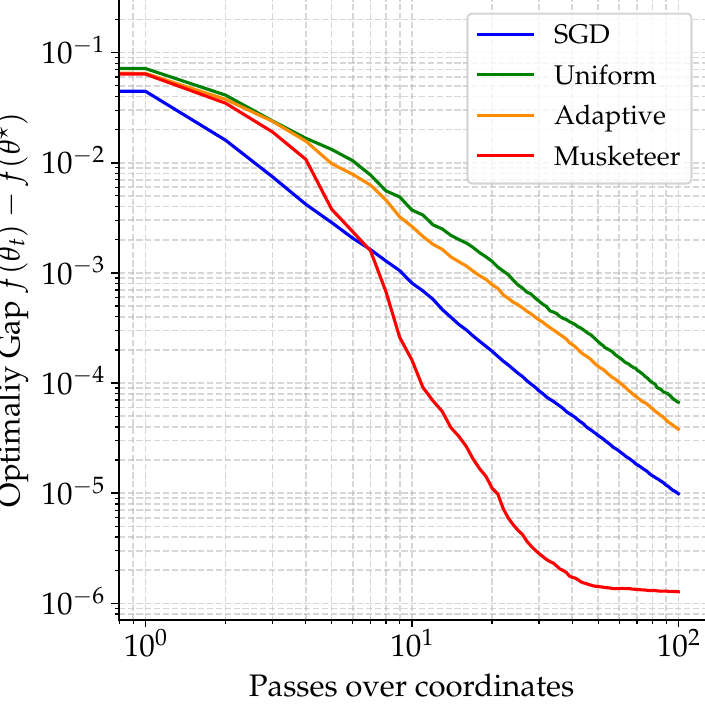}\label{fig:ridge5}}
  \subfigure[Ridge $\alpha=10$]{
  \includegraphics[scale=0.28]{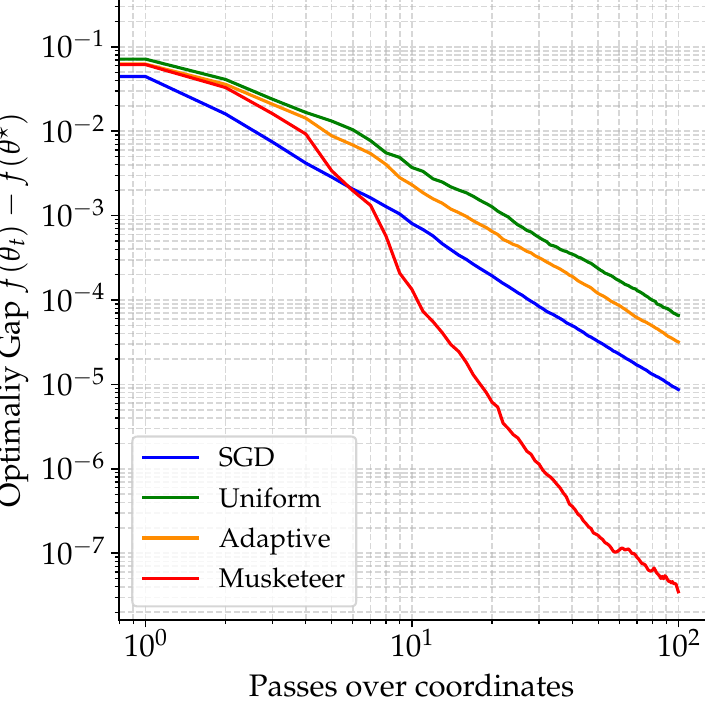}\label{fig:ridge10}}
  \subfigure[Logistic $\alpha=2$]{
  \includegraphics[scale=0.28]{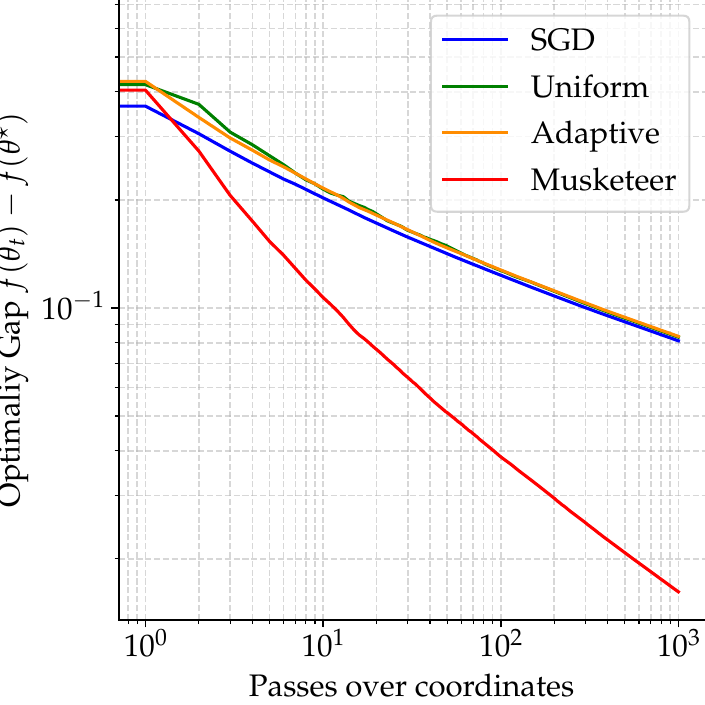}\label{fig:logistic2}}
  \subfigure[Logistic $\alpha=5$]{
  \includegraphics[scale=0.28]{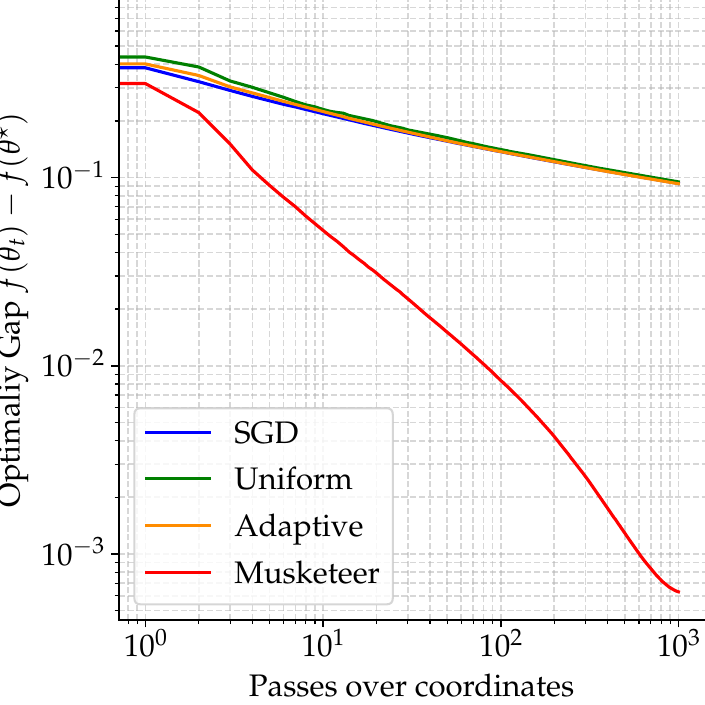}\label{fig:logistic5}}
  \caption{$[f(\theta_t)-f(\theta^\star)]$ for Linear Models on Synthetic data with different block structures.}
\label{fig:losses_linear}
\end{figure}

\noindent \textbf{Linear models.} We apply ERM to regularized regression and classification problems. Similarly to \citep{namkoong2017adaptive}, we endow the data matrix $X$ with a block structure. 
The columns are drawn as $X[:,k] \sim \mathcal{N}(0,\sigma_k^2 I_n)$ with $\sigma_k^2 = k^{-\alpha}$ for $k \in \llbracket 1,p \rrbracket$. The parameters are set to $n=10,000$ samples in dimension $p=250$  and  $T = 15$. Figure \ref{fig:losses_linear} provides the graphs of the optimaliy gap $t\mapsto f(\theta_t)-f^\star$ averaged over $20$ independent simulations for different values of $\alpha \in \{2;5;10\}$. First, note that the uniform sampling strategy shows similar performance to the classical SGD and that the (unbiased) importance sampling version ADAPTIVE is also of the same order. 
Besides, the clear winner is MUSKETEER as it offers the best performance in all configurations. Greater values of $\alpha$ (stronger block structure) improve our relative performance with respect to the other methods as shown by Figures \ref{fig:ridge10} and \ref{fig:logistic5}.

\medskip
\noindent \textbf{Neural Networks.} To asses the practical performance of MUSKETEER, we focus on the training of neural networks within the framework of multi-label classification. The datasets in the experiments are popular publicly available deep learning datasets: MNIST \citep{deng2012mnist}, Fashion-MNIST \citep{xiao2017fashion} and CIFAR10 \citep{krizhevsky2009learning}. Given an image, the goal is to predict its label among ten classes. Two different neural networks are used in the experiments: one with linear layers for MNIST and Fashion-MNIST ($p=55,050$ and $T  = 234$) , another one with convolutional layers for CIFAR10 ($p=64,862$ and $T = 254$). 
\begin{figure}[h]
  \centering
  \subfigure[MNIST]{
  \includegraphics[scale=0.38]{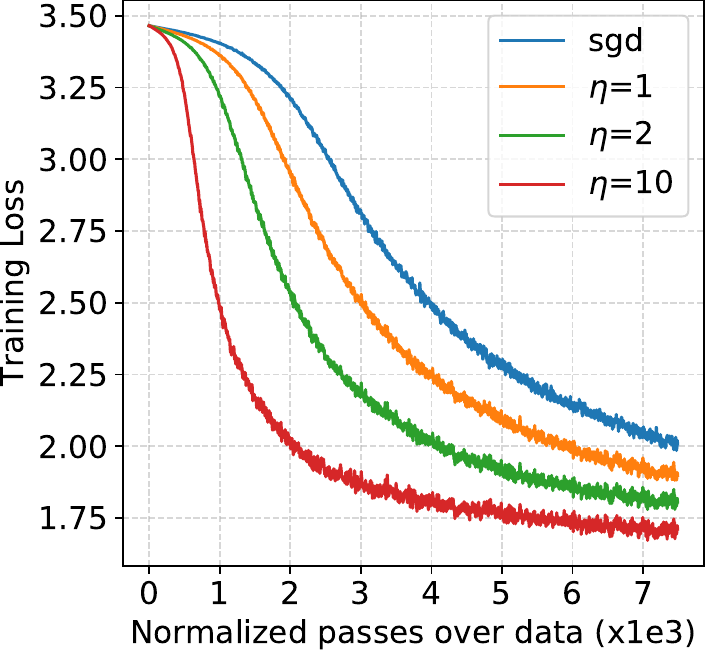}\label{fig:mnist}} \
  \subfigure[Fashion-MNIST]{
  \includegraphics[scale=0.38]{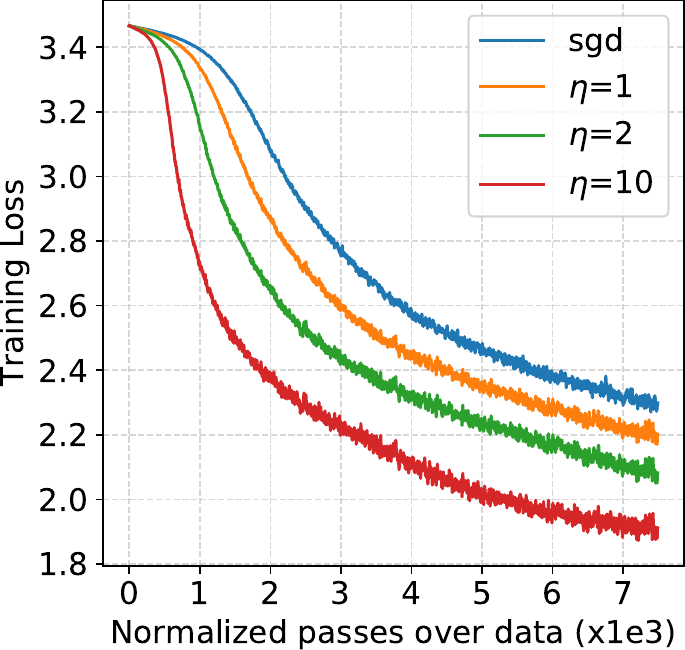}\label{fig:fashion}} \
  \subfigure[CIFAR-10]{
  \includegraphics[scale=0.38]{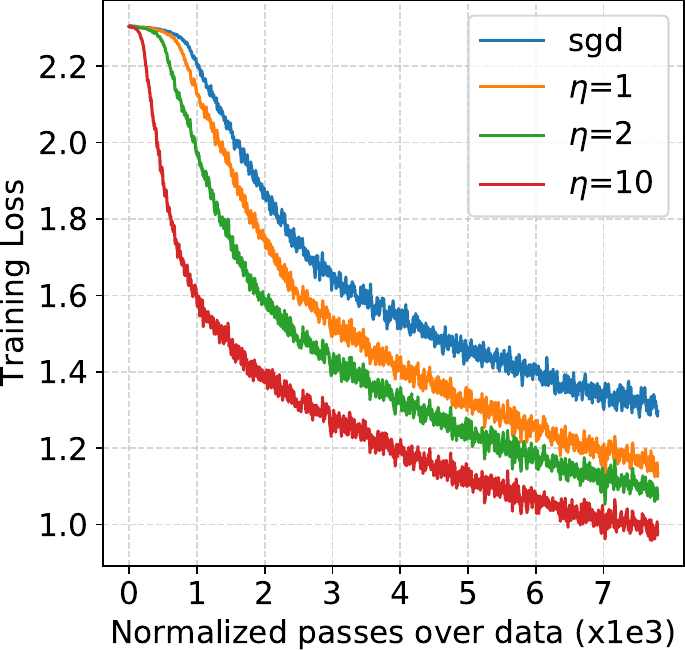}\label{fig:cifar10}}
  \caption{Training Loss of SGD vs. MUSKETEER on real-world datasets.}
\label{fig:losses_nn}
\end{figure}

Figure \ref{fig:losses_nn} compares the evolution of the training loss of SGD against MUSKETEER averaged over $10$ independent simulations with different values of $\eta$. A great value of this parameter strengthens the adaptive scheme as it gives more importance to the weights in Equation \eqref{eq:probas_update}, leading to stronger decrease of the objective function. Interestingly, the performance of MUSKETEER also benefit from such adaptive structure in terms on accuracy of the test set (see Table \ref{tab:test_acc}). This allows to quantify the statistical gain brought by MUSKETEER over SGD. 

\begin{table}[h!]
\centering
\begin{tabular}[b]{|c|c|c|c|c|}\hline
       & SGD & $\eta=1$ & $\eta=2$ & $\eta=10$\\ \hline
      MNIST & 84.7$\pm$1.0 & 86.7$\pm$0.5 & 88.9$\pm$0.4 & \textbf{91.3$\pm$0.2} \\ \hline
	  FASHION & 64.7$\pm$1.2 & 68.5$\pm$1.0 & 71.2$\pm$0.7 & \textbf{77.1$\pm$0.8} \\ \hline
	  CIFAR10 & 51.4$\pm$1.4 & 57.7$\pm$0.8 & 59.7$\pm$1.0 & \textbf{62.7$\pm$0.8} \\ \hline      
    \end{tabular}
    \caption{Test Accuracy (in \%).}
    \label{tab:test_acc}
\end{table}

\newpage
\section{Further Numerical Experiments with zeroth-order methods}
\label{sec:more_zo_simus}

\subsection{Ridge Regression ($\ell_1$-reweighting) with different settings of $(n,p)$} \label{subsec:zo_ridge}
We consider the Ridge regression problem with the classical regularization parameter value $\mu=1/n$ and run several experiments in various settings of $(n,p)$. We endow the data matrix $X$ with a block structure. The columns are drawn as $X[:,kB+1:kB+B] \sim \mathcal{N}(0,\sigma_k^2 I_n)$ with $\sigma_k^2 = k^{-\alpha}$ for all $k\in \llbracket 0,(p/B)-1 \rrbracket$. The parameter $B$ is the block-size and is set to $B=10$ for the Ridge regression. The parameter $\alpha$ represents the block structure and is set to $\alpha=5$. The different Figures below present the evolution of the optimality gap $t \mapsto[f(\theta_t)-f^\star]$ averaged over $20$ independent runs. The learning rates is the same for all methods, fixed to $\gamma_k = 1/(k+10)$. The different settings are: number of samples $n \in \{1,000;2,000;5,000\}$ and dimension $p \in \{20;50;100;200\}$.  We use the $\ell_1$ normalization in Equation \eqref{eq:probas_update} with $\lambda_n = 1/\log(n)$.

\begin{figure}[h]
  \centering
  \subfigure[$n=1000, p=20$]{
  \includegraphics[scale=0.26]{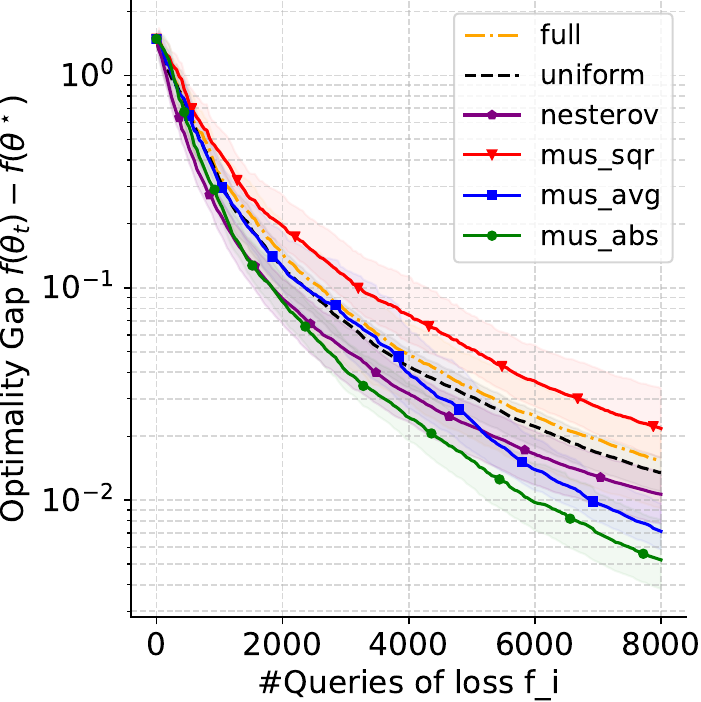}
  \label{fig:ridge_n1000_p20}}
  \subfigure[$n=1000, p=50$]{
  \includegraphics[scale=0.26]{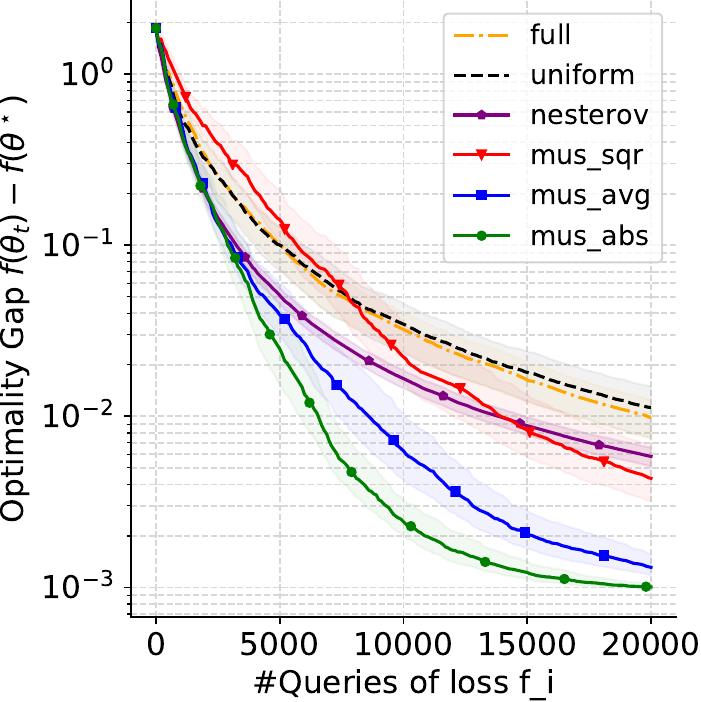}\label{fig:ridge_n1000_p50}}
  \subfigure[$n=1000, p=100$]{
  \includegraphics[scale=0.26]{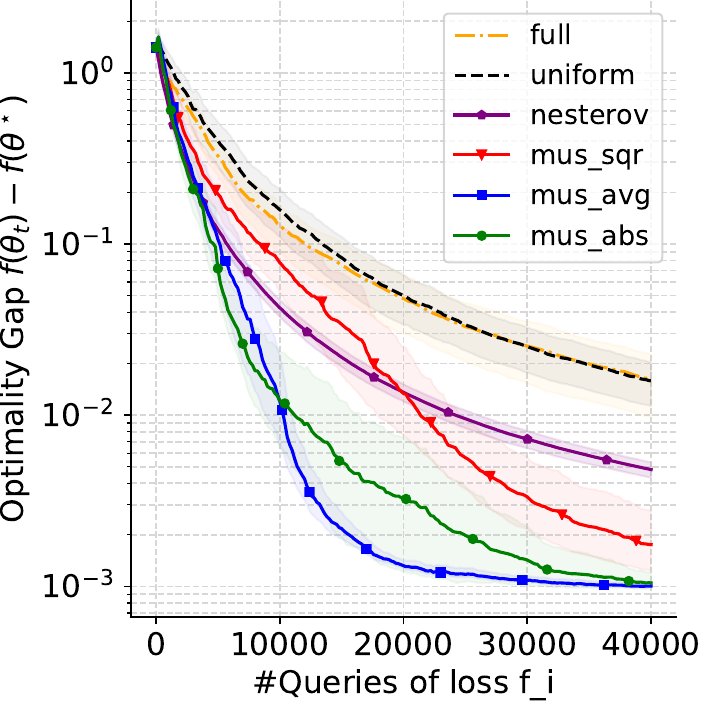}\label{ridge_n1000_p100}}
  \subfigure[$n=1000, p=200$]{
  \includegraphics[scale=0.26]{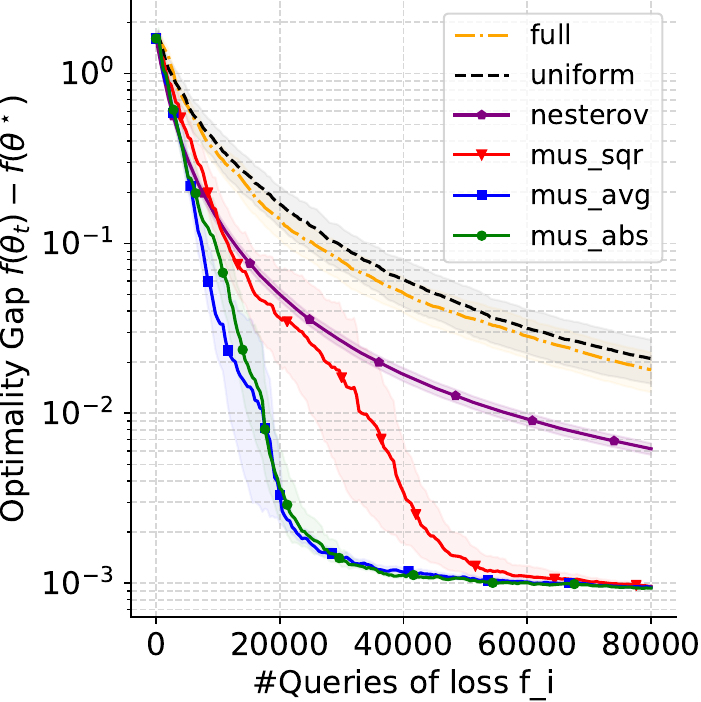}\label{ridge_n1000_p200}}
  \caption{$[f(\theta_t)-f^\star]$ for Ridge Regression with $n=1000$ and $p=20,50,100,200$}
\label{fig:ridge_n1000}

  \centering
  \subfigure[$n=2000, p=20$]{
  \includegraphics[scale=0.26]{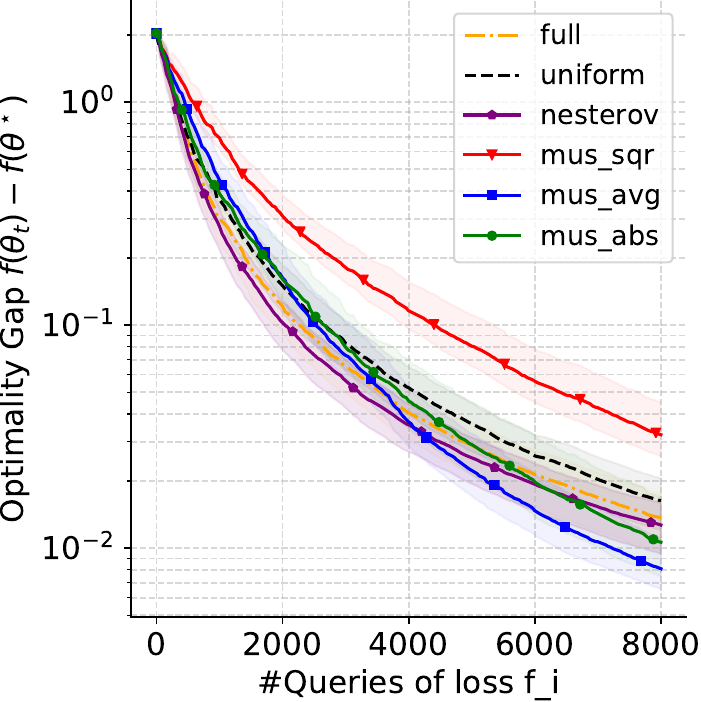}
  \label{fig:ridge_n2000_p20}}
  \subfigure[$n=2000, p=50$]{
  \includegraphics[scale=0.26]{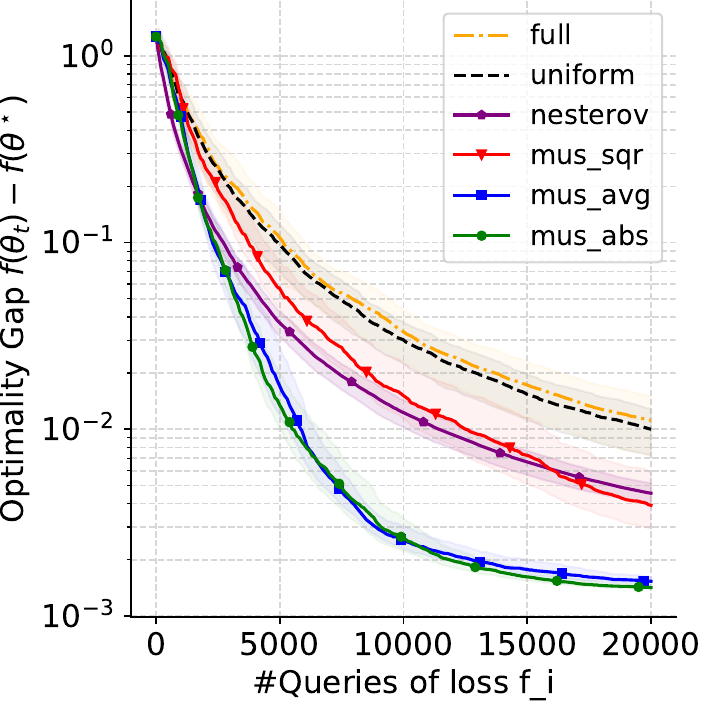}\label{fig:ridge_n2000_p50}}
  \subfigure[$n=2000, p=100$]{
  \includegraphics[scale=0.26]{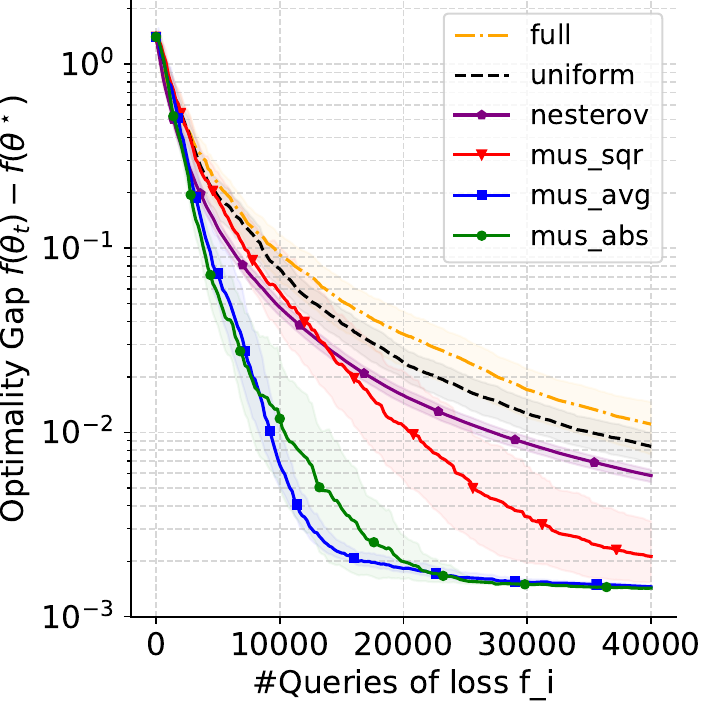}\label{ridge_n2000_p100}}
  \subfigure[$n=2000, p=200$]{
  \includegraphics[scale=0.26]{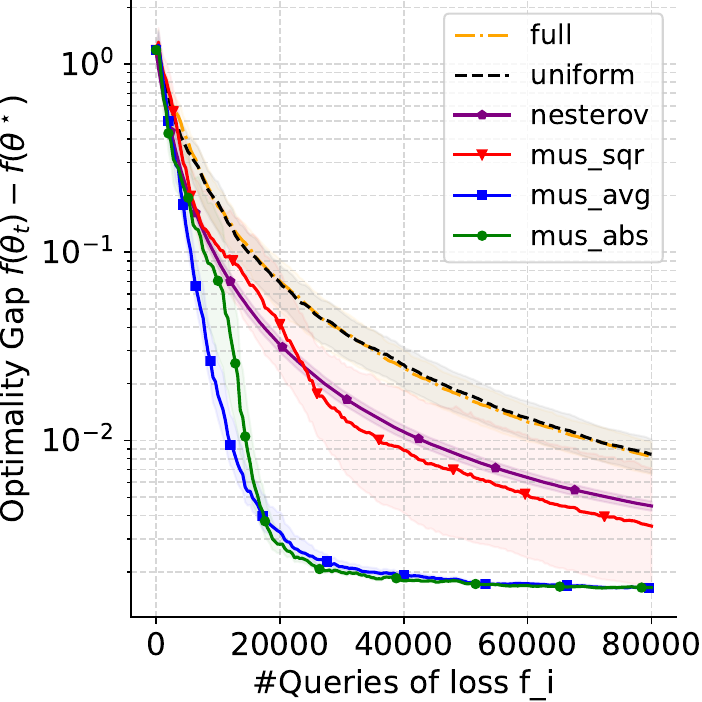}\label{ridge_n2000_p200}}
  \caption{$[f(\theta_t)-f^\star]$ for Ridge Regression with $n=2000$ and $p=20,50,100,200$}
\label{fig:ridge_n2000}

  \centering
  \subfigure[$n=5000, p=20$]{
  \includegraphics[scale=0.26]{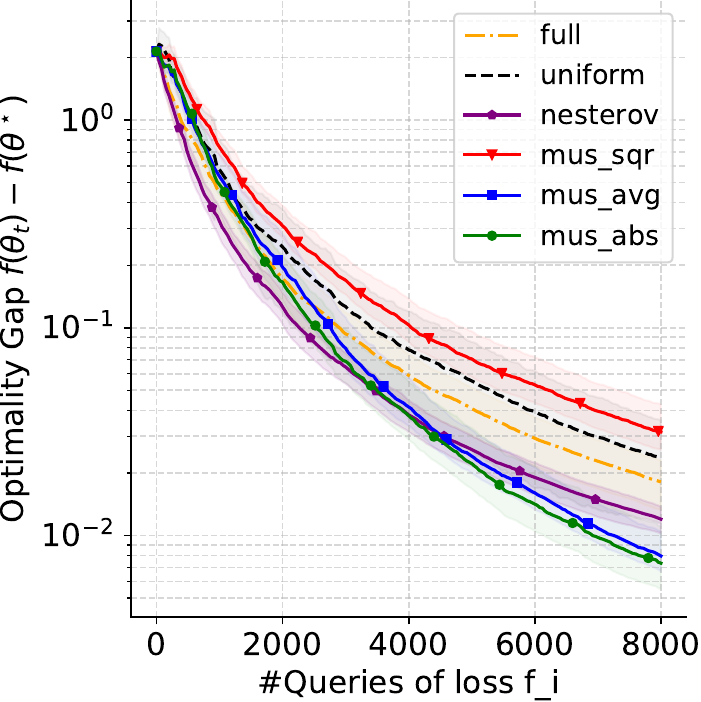}
  \label{fig:ridge_n5000_p20}}
  \subfigure[$n=5000, p=50$]{
  \includegraphics[scale=0.26]{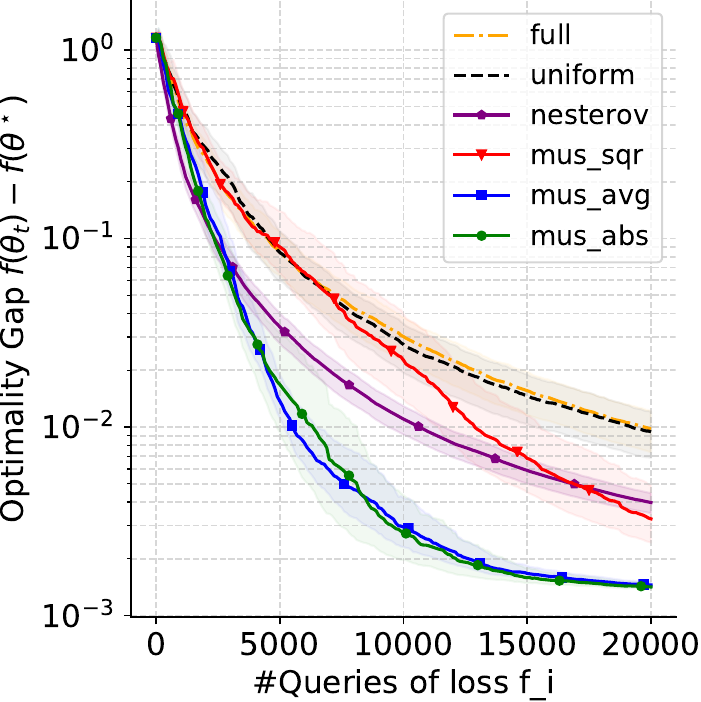}\label{fig:ridge_n5000_p50}}
  \subfigure[$n=5000, p=100$]{
  \includegraphics[scale=0.26]{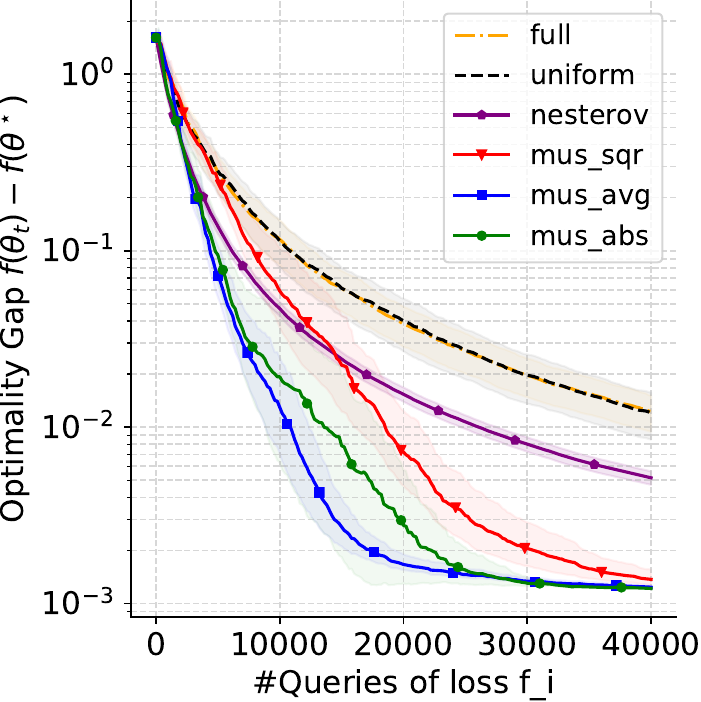}\label{ridge_n5000_p100}}
  \subfigure[$n=5000, p=200$]{
  \includegraphics[scale=0.26]{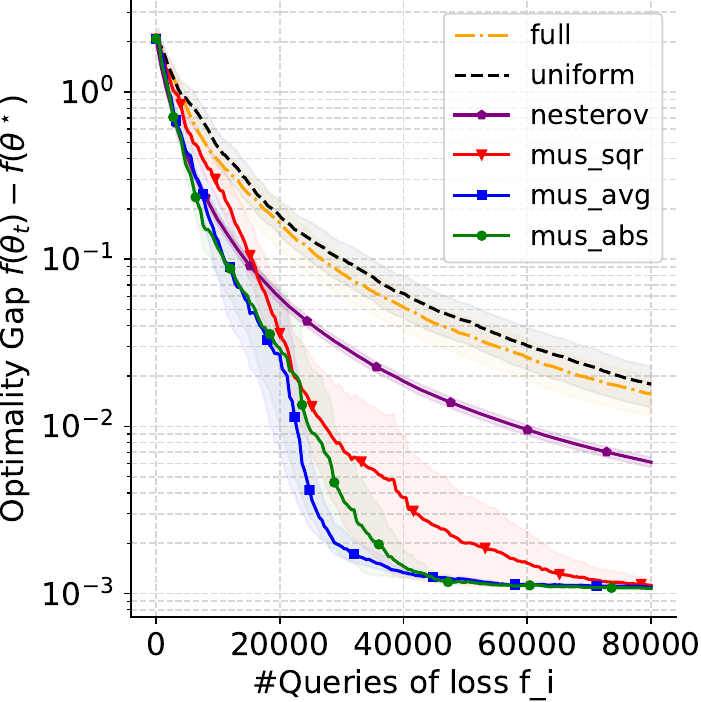}\label{ridge_n5000_p200}}
  \caption{$[f(\theta_t)-f^\star]$ for Ridge Regression with $n=5000$ and $p=20,50,100,200$}
\label{fig:ridge_n5000}
\end{figure}

\newpage

\subsection{Ridge Regression (softmax reweighting) with different settings of $(n,p)$} \label{subsec:zo_ridge_softmax}
We consider the Ridge regression problem with the classical regularization parameter value $\mu=1/n$ and run several experiments in various settings of $(n,p)$. We endow the data matrix $X$ with a block structure. The columns are drawn as $X[:,kB+1:kB+B] \sim \mathcal{N}(0,\sigma_k^2 I_n)$ with $\sigma_k^2 = k^{-\alpha}$ for all $k\in \llbracket 0,(p/B)-1 \rrbracket$. The parameter $B$ is the block-size and is set to $B=10$ for the Ridge regression. The parameter $\alpha$ represents the block structure and is set to $\alpha=5$. The different Figures below present the evolution of the optimality gap $t \mapsto[f(\theta_t)-f^\star]$ averaged over $20$ independent runs. The learning rates is the same for all methods, fixed to $\gamma_k = 1/(k+10)$. The different settings are: number of samples $n \in \{1,000;2,000;5,000\}$ and dimension $p \in \{20;50;100;200\}$. We use the softmax normalization in Equation \eqref{eq:probas_update} with $\lambda_n \equiv 0.5$ and $\eta=1$.

\begin{figure}[h]
  \centering
  \subfigure[$n=1000, p=20$]{
  \includegraphics[scale=0.26]{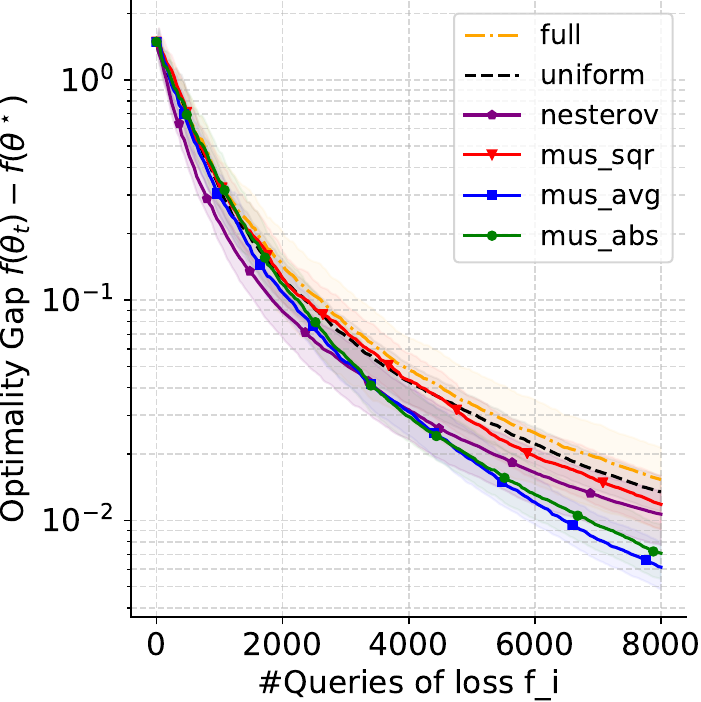}
  \label{fig:ridge_exp_n1000_p20}}
  \subfigure[$n=1000, p=50$]{
  \includegraphics[scale=0.26]{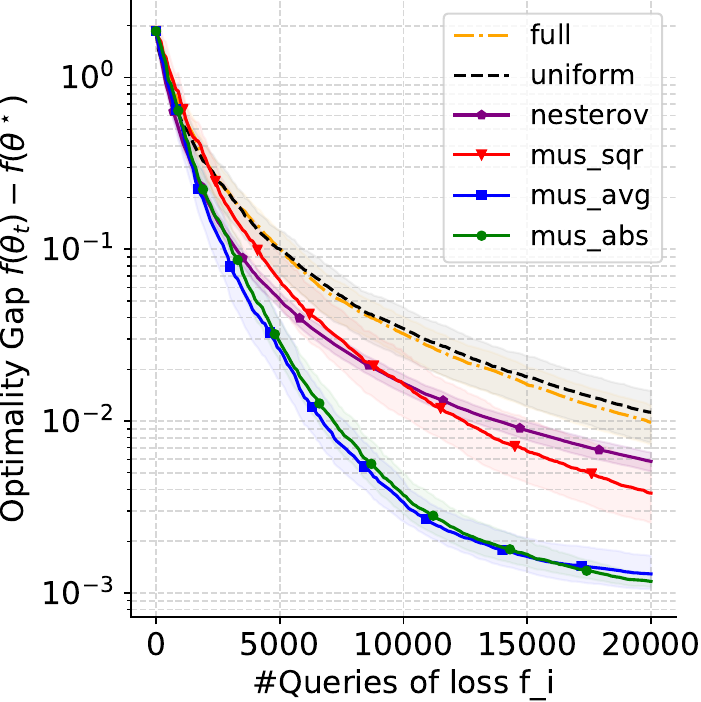}\label{fig:ridge_exp_n1000_p50}}
  \subfigure[$n=1000, p=100$]{
  \includegraphics[scale=0.26]{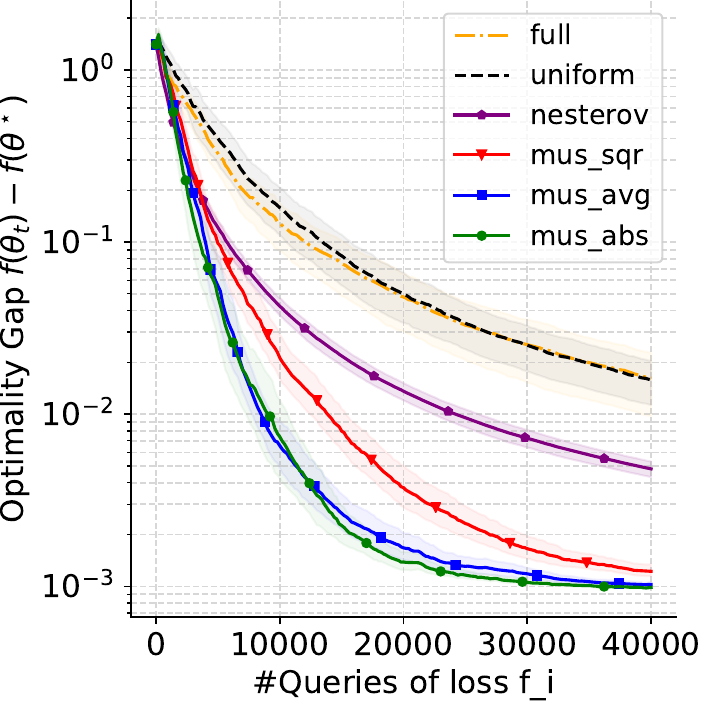}\label{ridge_exp_n1000_p100}}
  \subfigure[$n=1000, p=200$]{
  \includegraphics[scale=0.26]{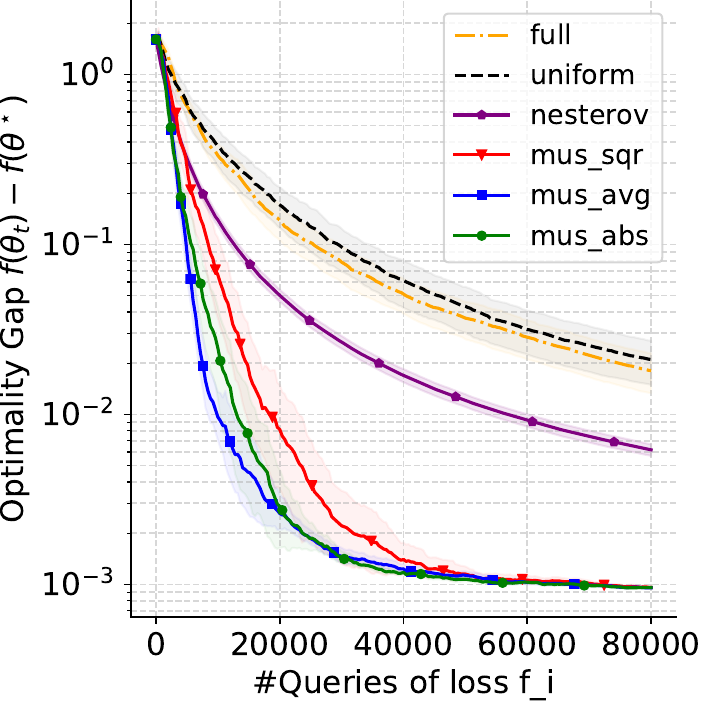}\label{ridge_exp_n1000_p200}}
  \caption{$[f(\theta_t)-f^\star]$ for Ridge Regression with $n=1000$ and $p=20,50,100,200$}
\label{fig:ridge_exp_n1000}

  \centering
  \subfigure[$n=2000, p=20$]{
  \includegraphics[scale=0.26]{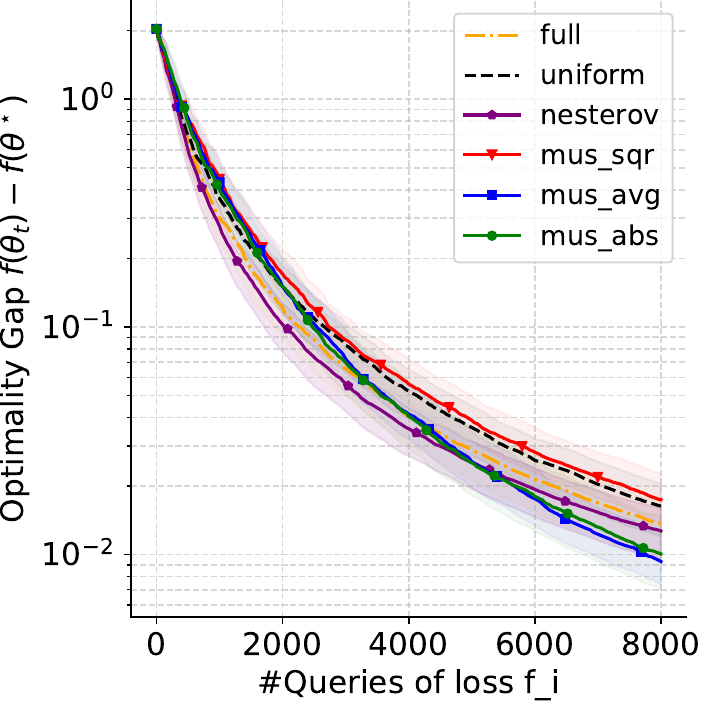}
  \label{fig:ridge_exp_n2000_p20}}
  \subfigure[$n=2000, p=50$]{
  \includegraphics[scale=0.26]{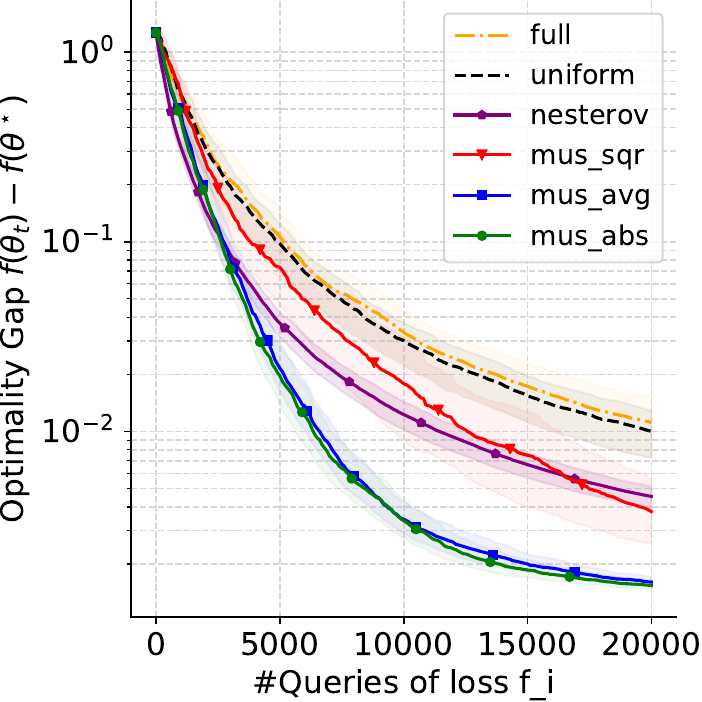}\label{fig:ridge_exp_n2000_p50}}
  \subfigure[$n=2000, p=100$]{
  \includegraphics[scale=0.26]{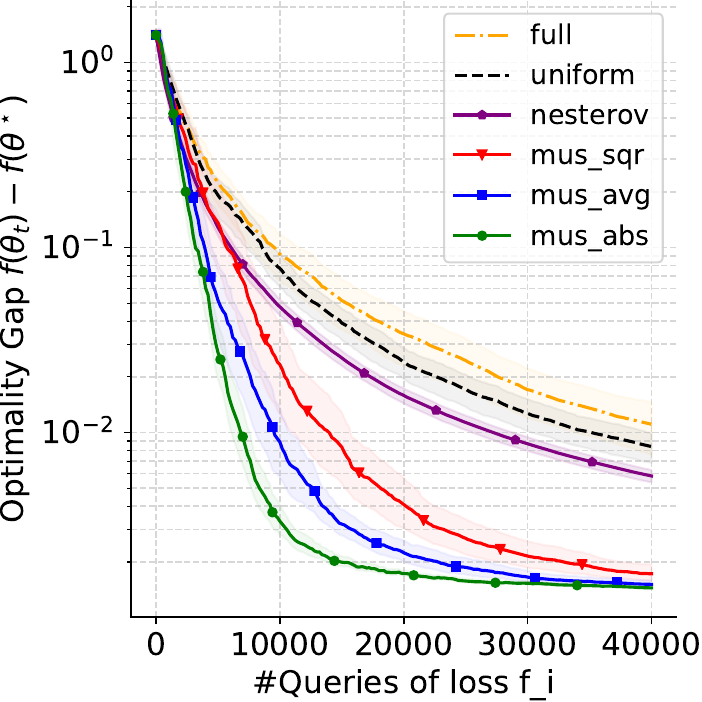}\label{ridge_exp_n2000_p100}}
  \subfigure[$n=2000, p=200$]{
  \includegraphics[scale=0.26]{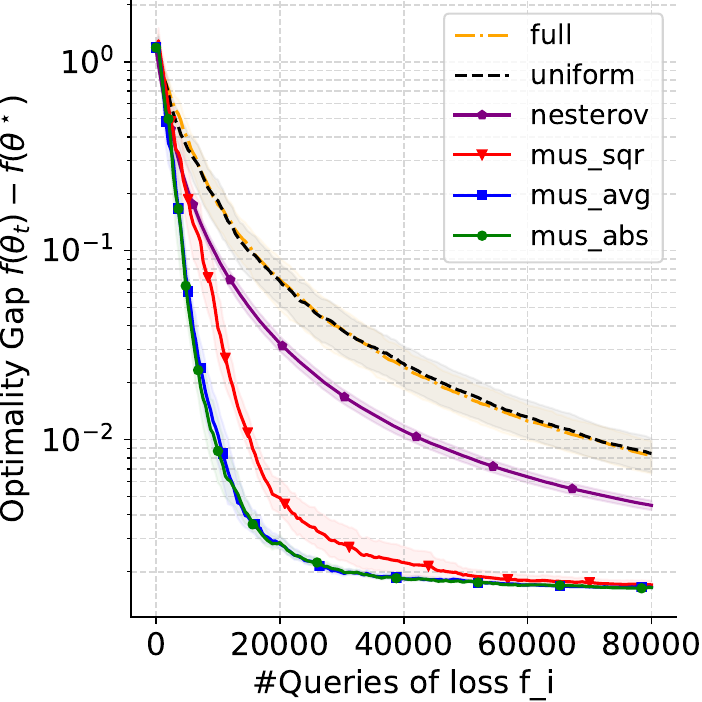}\label{ridge_exp_n2000_p200}}
  \caption{$[f(\theta_t)-f^\star]$ for Ridge Regression with $n=2000$ and $p=20,50,100,200$}
\label{fig:ridge_exp_n2000}

  \centering
  \subfigure[$n=5000, p=20$]{
  \includegraphics[scale=0.26]{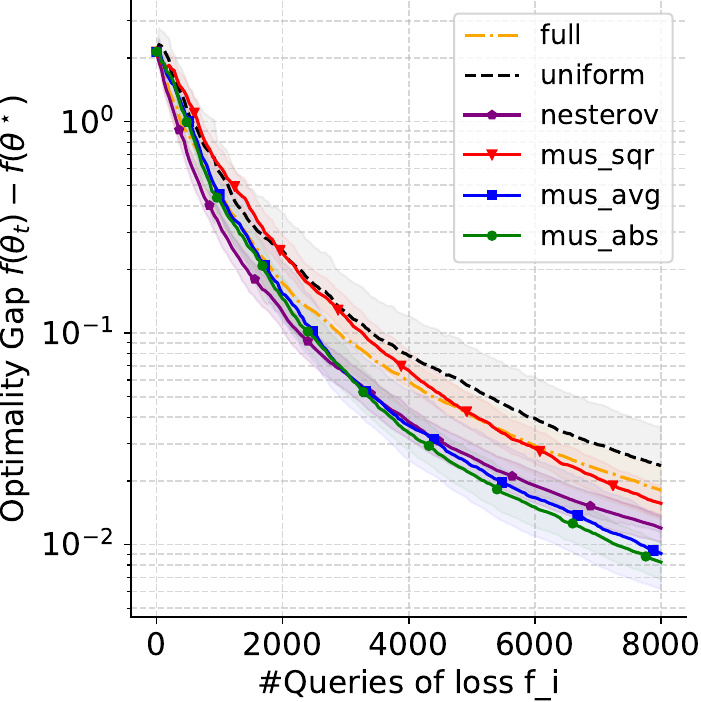}
  \label{fig:ridge_exp_n5000_p20}}
  \subfigure[$n=5000, p=50$]{
  \includegraphics[scale=0.26]{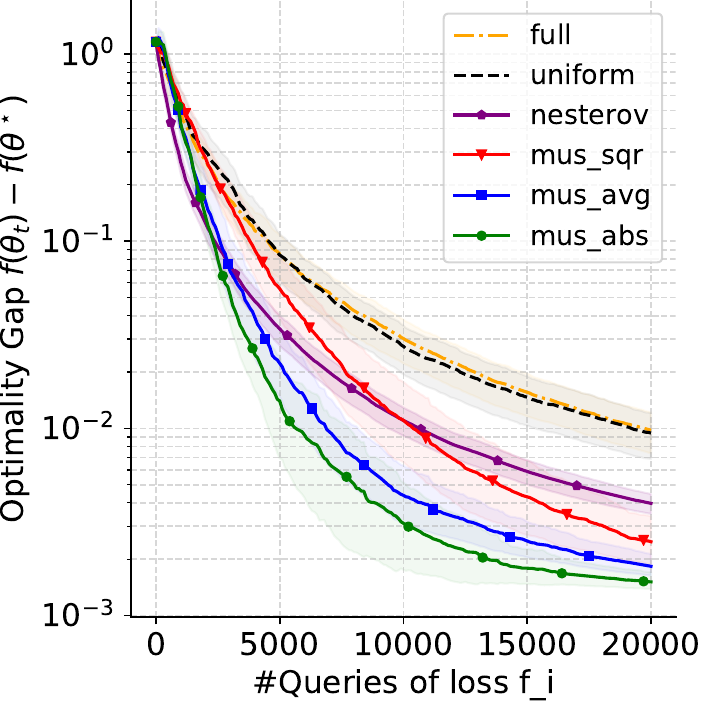}\label{fig:ridge_exp_n5000_p50}}
  \subfigure[$n=5000, p=100$]{
  \includegraphics[scale=0.26]{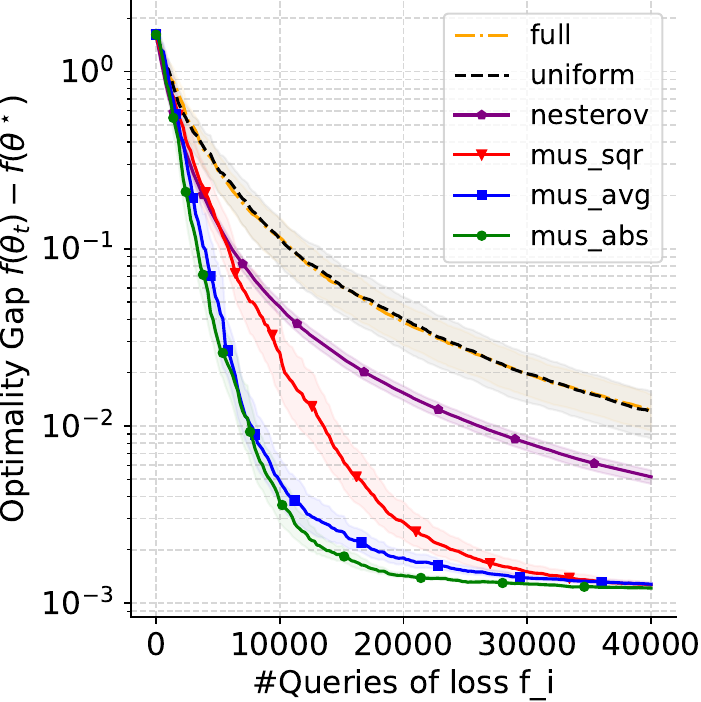}\label{ridge_exp_n5000_p100}}
  \subfigure[$n=5000, p=200$]{
  \includegraphics[scale=0.26]{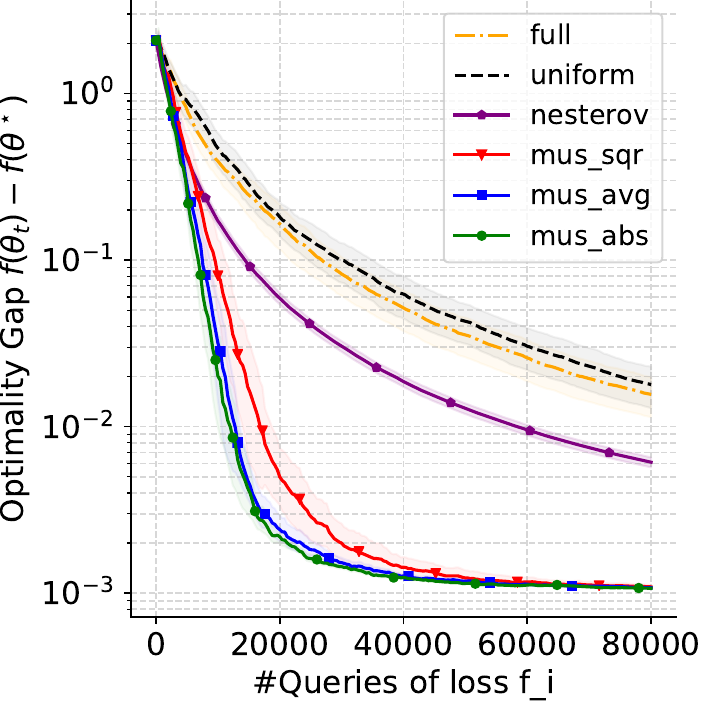}\label{ridge_exp_n5000_p200}}
  \caption{$[f(\theta_t)-f^\star]$ for Ridge Regression with $n=5000$ and $p=20,50,100,200$}
\label{fig:ridge_exp_n5000}
\end{figure}

\newpage
\subsection{Logistic Regression ($\ell_1$-reweighting) with different settings of $(n,p)$} \label{subsec:zo_log}
We consider the $\ell_2$-Logistic regression problem with the classical regularization parameter value $\mu=1/n$ and run several experiments in various settings of $(n,p)$. We endow the data matrix $X$ with a block structure. The columns are drawn as $X[:,kB+1:kB+B] \sim \mathcal{N}(0,\sigma_k^2 I_n)$ with $\sigma_k^2 = k^{-\alpha}$ for all $k\in \llbracket 1,(p/B)-1 \rrbracket$. The parameter $B$ is the block-size and is set to $B=5$ for the Logistic regression. The parameter $\alpha$ represents the block structure and is set to $\alpha=5$. The different Figures below present the evolution of the optimality gap $t \mapsto[f(\theta_t)-f^\star]$ averaged over $20$ independent runs. The learning rates is the same for all methods, fixed to $\gamma_k = 10/(k+5)$. The different settings are: number of samples $n \in \{1,000;2,000;5,000\}$ and dimension $p \in \{20;50;100;200\}$. We use the $\ell_1$ normalization in Equation \eqref{eq:probas_update} with $\lambda_n = 1/\log(n)$.
\begin{figure}[h]
  \centering
  \subfigure[$n=1000, p=20$]{
  \includegraphics[scale=0.26]{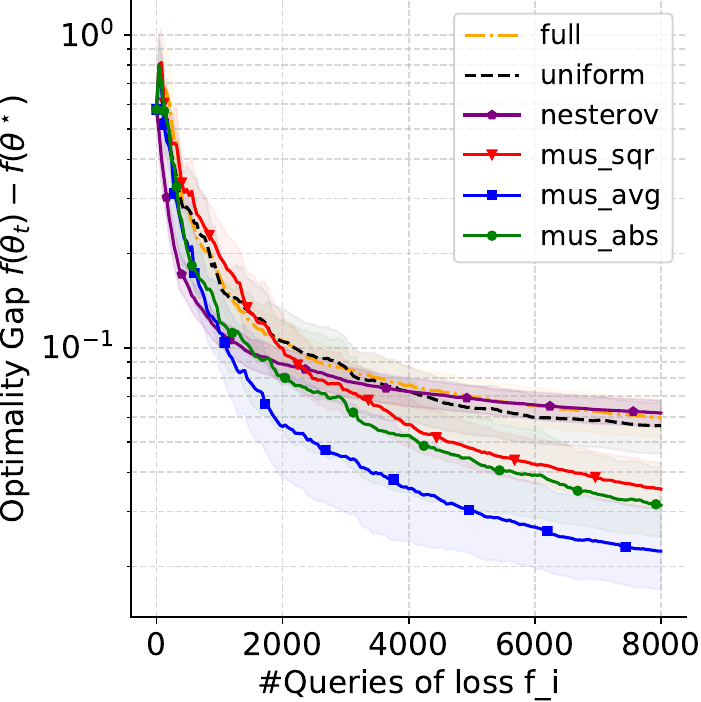}
  \label{fig:logistic_n1000_p20}}
  \subfigure[$n=1000, p=50$]{
  \includegraphics[scale=0.26]{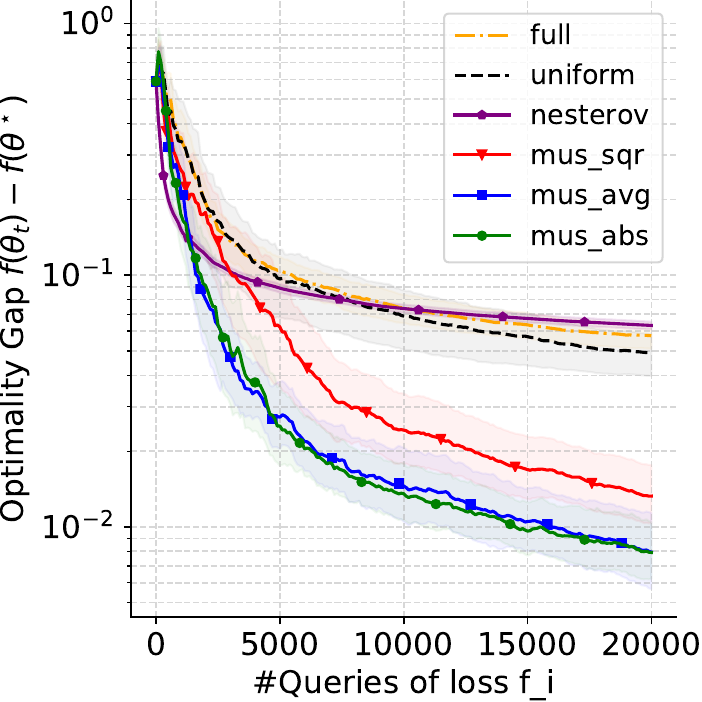}\label{fig:logistic_n1000_p50}}
  \subfigure[$n=1000, p=100$]{
  \includegraphics[scale=0.26]{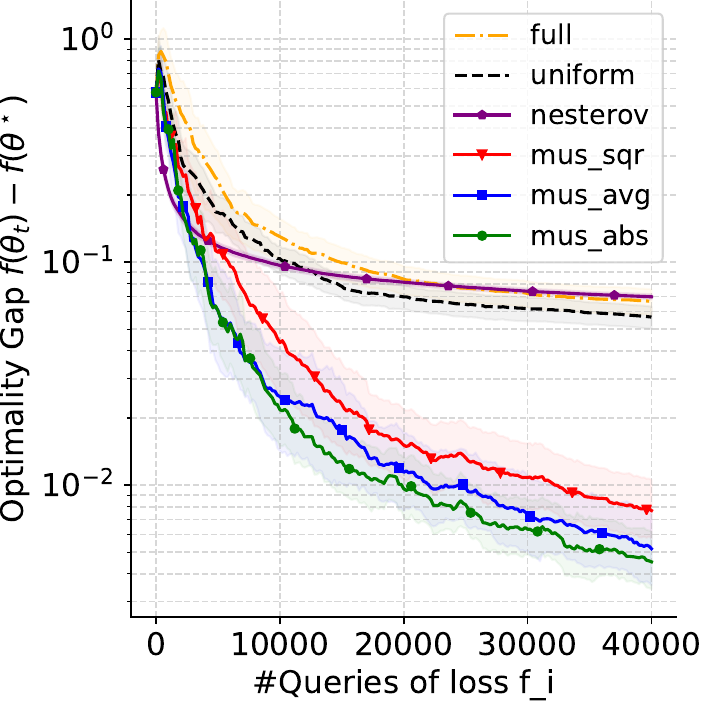}\label{logistic_n1000_p100}}
  \subfigure[$n=1000, p=200$]{
  \includegraphics[scale=0.26]{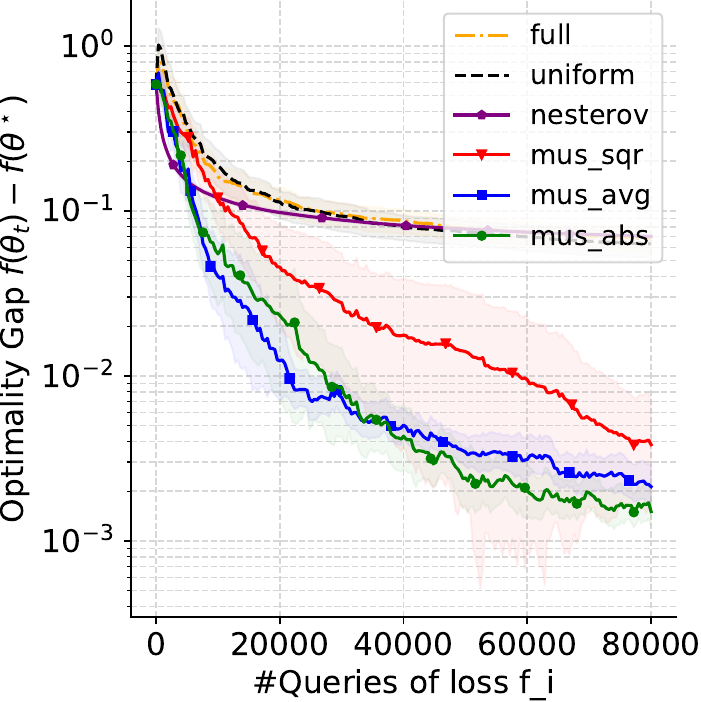}\label{logistic_n1000_p200}}
  \caption{$[f(\theta_t)-f^\star]$ for logistic Regression with $n=1000$ and $p=20,50,100,200$}
\label{fig:logistic_n1000}

  \centering
  \subfigure[$n=2000, p=20$]{
  \includegraphics[scale=0.26]{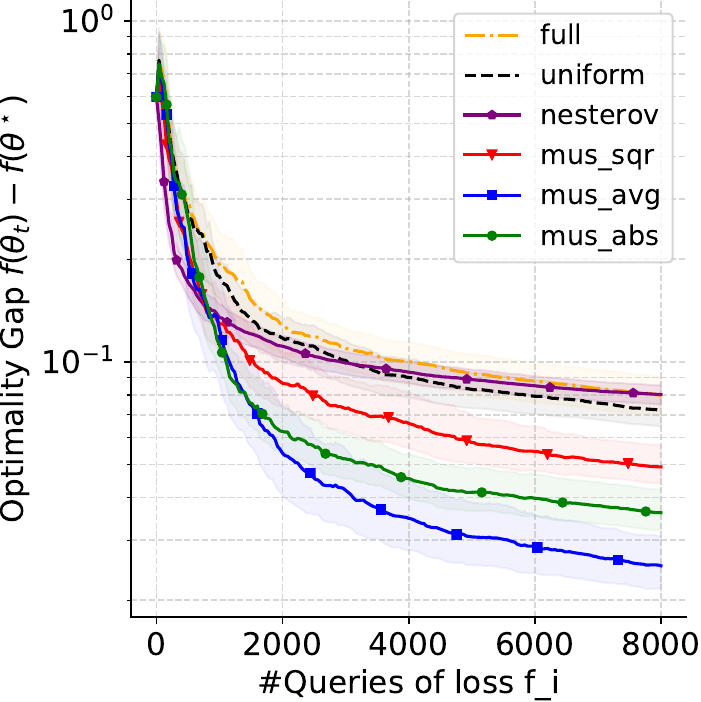}
  \label{fig:logistic_n2000_p20}}
  \subfigure[$n=2000, p=50$]{
  \includegraphics[scale=0.26]{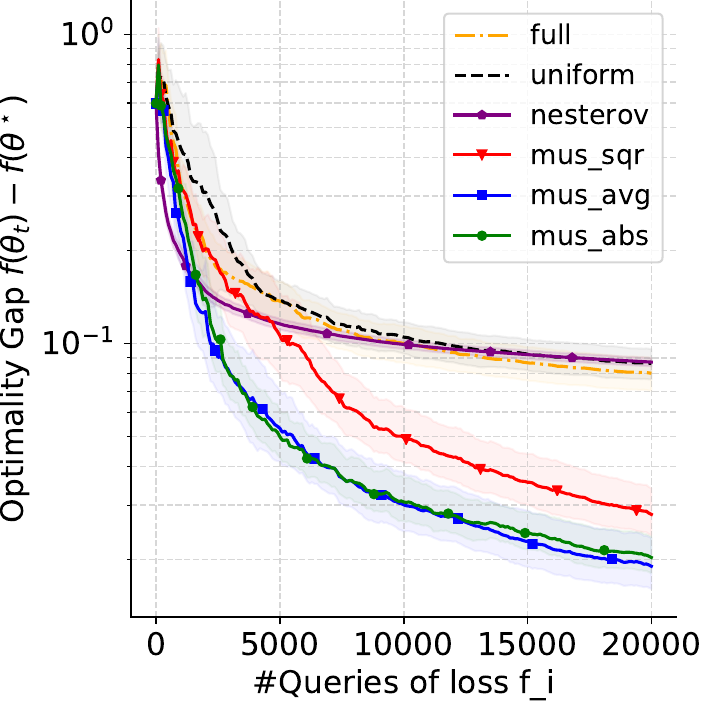}\label{fig:logistic_n2000_p50}}
  \subfigure[$n=2000, p=100$]{
  \includegraphics[scale=0.26]{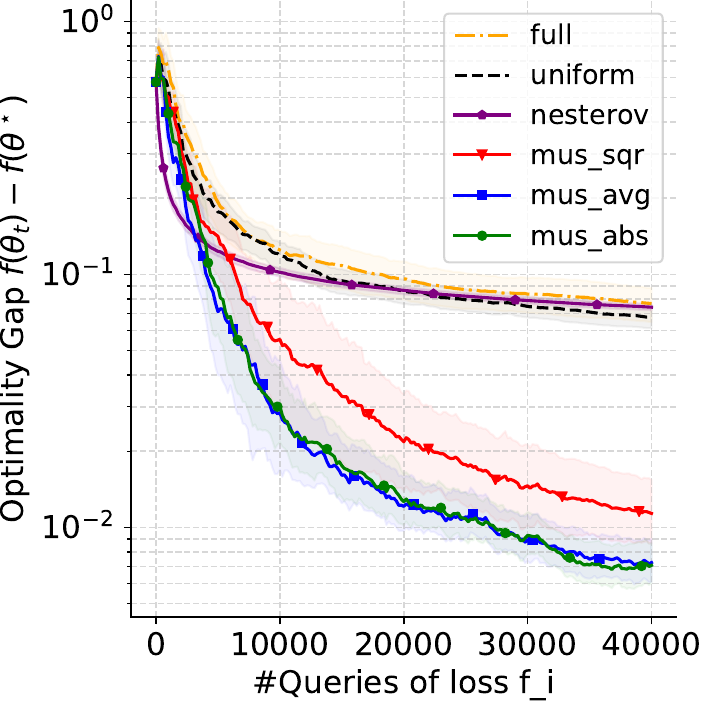}\label{logistic_n2000_p100}}
  \subfigure[$n=2000, p=200$]{
  \includegraphics[scale=0.26]{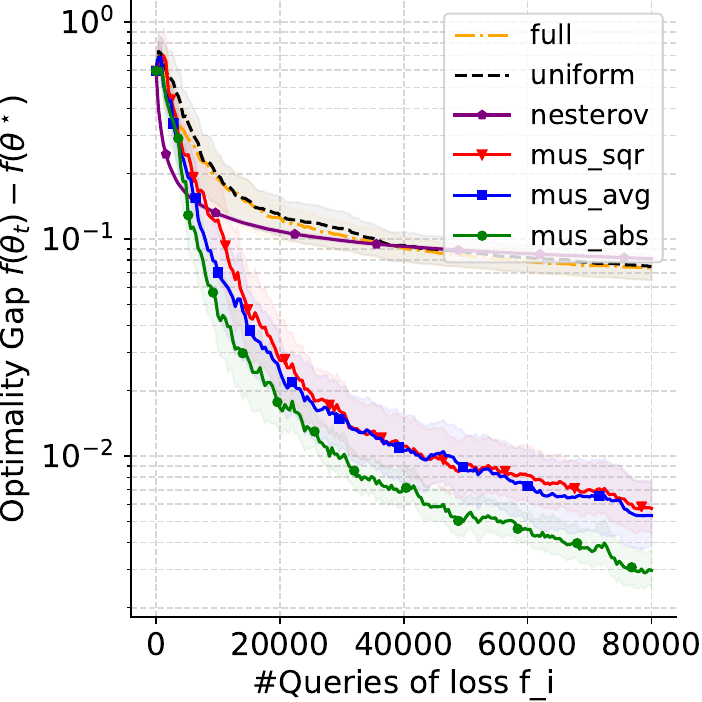}\label{logistic_n2000_p200}}
  \caption{ $[f(\theta_t)-f^\star]$ for logistic Regression with $n=2000$ and $p=20,50,100,200$}
\label{fig:logistic_n2000}

  \centering
  \subfigure[$n=5000, p=20$]{
  \includegraphics[scale=0.26]{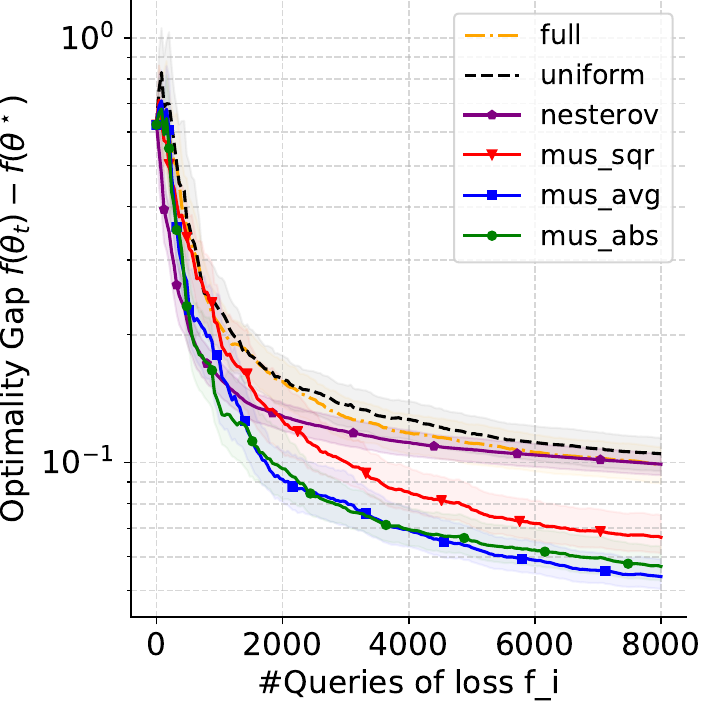}
  \label{fig:logistic_n5000_p20}}
  \subfigure[$n=5000, p=50$]{
  \includegraphics[scale=0.26]{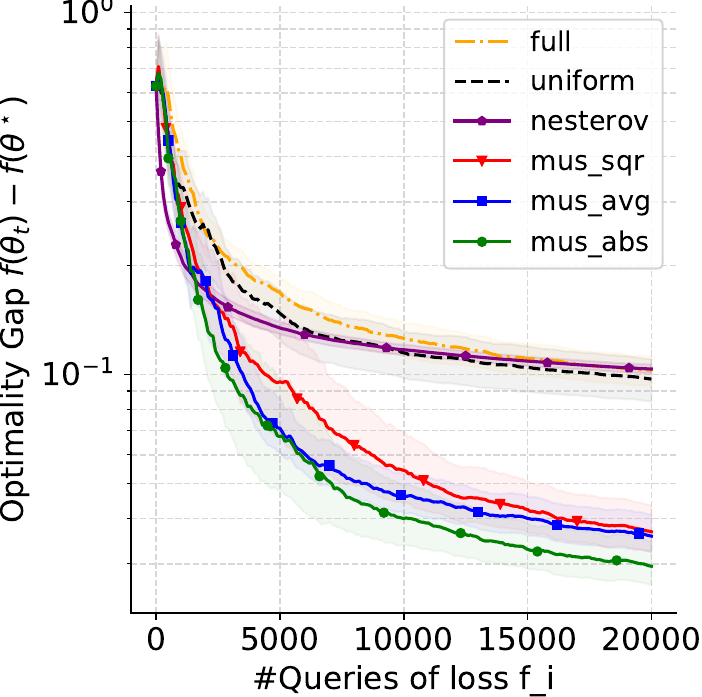}\label{fig:logistic_n5000_p50}}
  \subfigure[$n=5000, p=100$]{
  \includegraphics[scale=0.26]{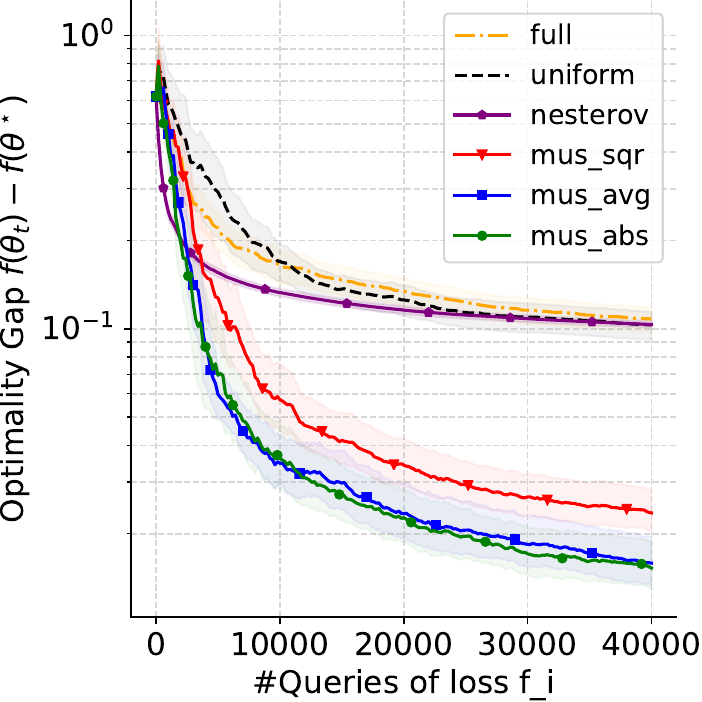}\label{logistic_n5000_p100}}
  \subfigure[$n=5000, p=200$]{
  \includegraphics[scale=0.26]{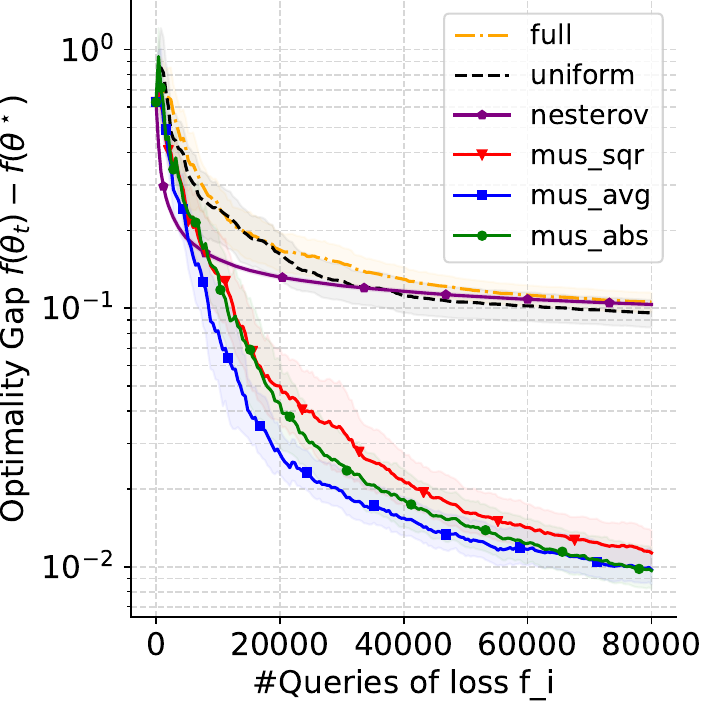}\label{logistic_n5000_p200}}
  \caption{$[f(\theta_t)-f^\star]$ for logistic Regression with $n=5000$ and $p=20,50,100,200$}
\label{fig:logistic_n5000}
\end{figure}
\newpage
\subsection{Logistic Regression (softmax reweighting) with different settings of $(n,p)$} \label{subsec:zo_log_soft}
We consider the $\ell_2$-Logistic regression problem with the classical regularization parameter value $\mu=1/n$ and run several experiments in various settings of $(n,p)$. We endow the data matrix $X$ with a block structure. The columns are drawn as $X[:,kB+1:kB+B] \sim \mathcal{N}(0,\sigma_k^2 I_n)$ with $\sigma_k^2 = k^{-\alpha}$ for all $k\in \llbracket 1,(p/B)-1 \rrbracket$. The parameter $B$ is the block-size and is set to $B=5$ for the Logistic regression. The parameter $\alpha$ represents the block structure and is set to $\alpha=5$. The different Figures below present the evolution of the optimality gap $t \mapsto[f(\theta_t)-f^\star]$ averaged over $20$ independent runs. The learning rates is the same for all methods, fixed to $\gamma_k = 10/(k+5)$. The different settings are: number of samples $n \in \{1,000;2,000;5,000\}$ and dimension $p \in \{20;50;100;200\}$. We use the softmax 
normalization in Equation \eqref{eq:probas_update} with $\lambda_n \equiv 0.5$ and $\eta=1$.
\begin{figure}[h]
  \centering
  \subfigure[$n=1000, p=20$]{
  \includegraphics[scale=0.26]{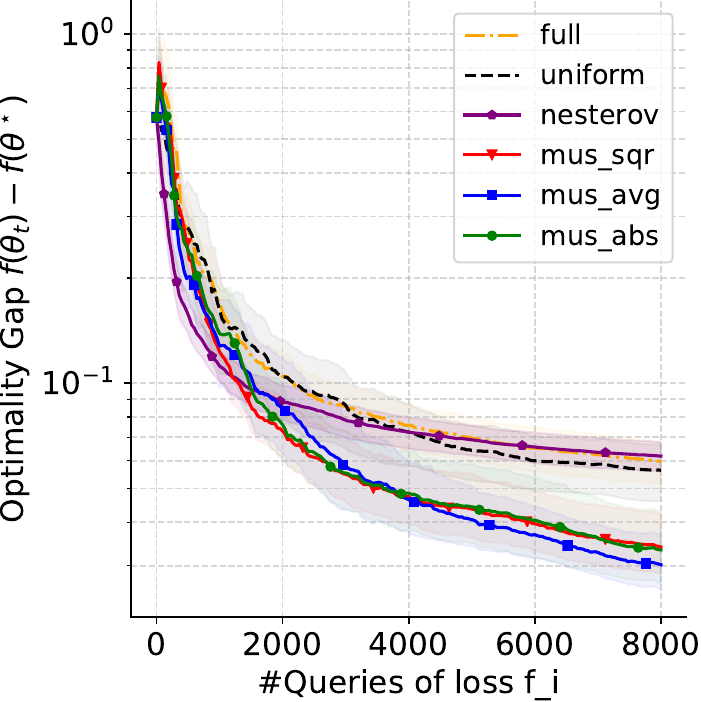}
  \label{fig:logistic_exp_n1000_p20}}
  \subfigure[$n=1000, p=50$]{
  \includegraphics[scale=0.26]{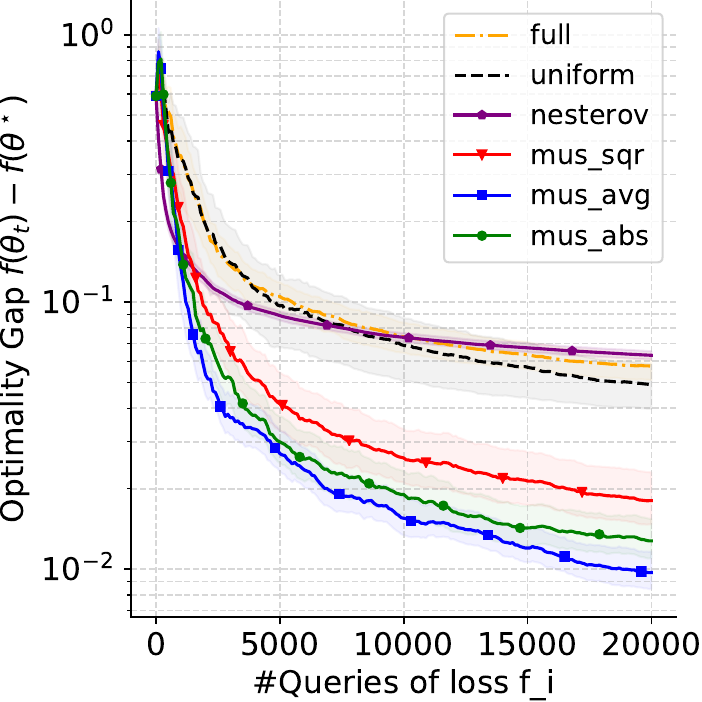}\label{fig:logistic_exp_n1000_p50}}
  \subfigure[$n=1000, p=100$]{
  \includegraphics[scale=0.26]{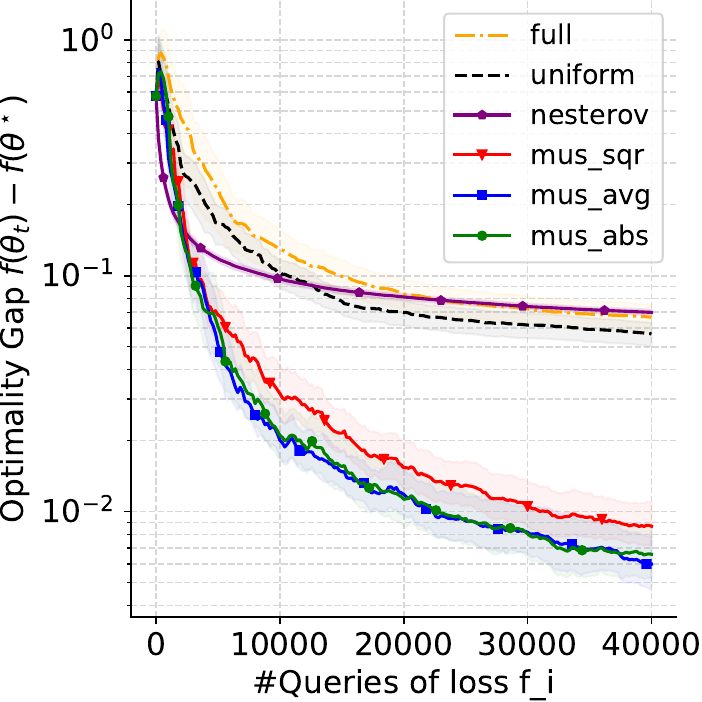}\label{logistic_exp_n1000_p100}}
  \subfigure[$n=1000, p=200$]{
  \includegraphics[scale=0.26]{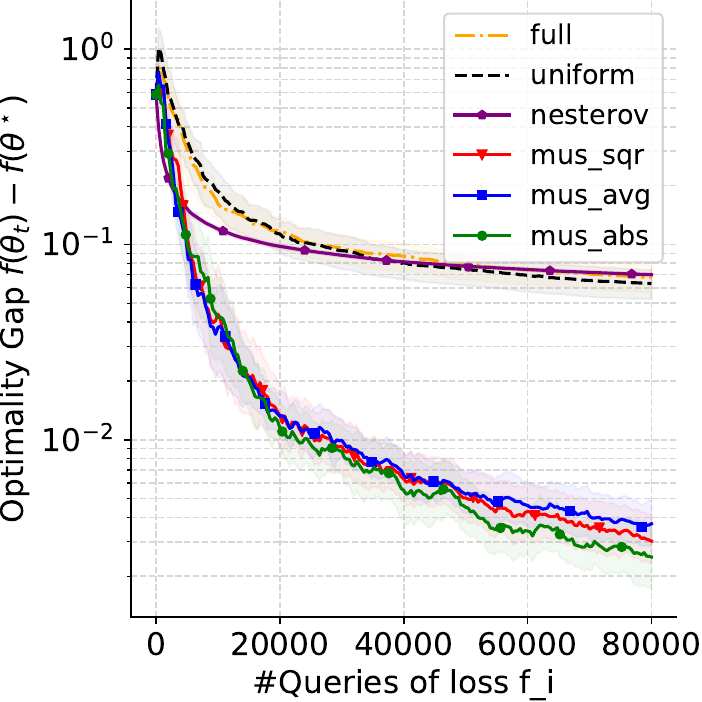}\label{logistic_exp_n1000_p200}}
  \caption{$[f(\theta_t)-f^\star]$ for logistic Regression with $n=1000$ and $p=20,50,100,200$}
\label{fig:logistic_exp_n1000}

  \centering
  \subfigure[$n=2000, p=20$]{
  \includegraphics[scale=0.26]{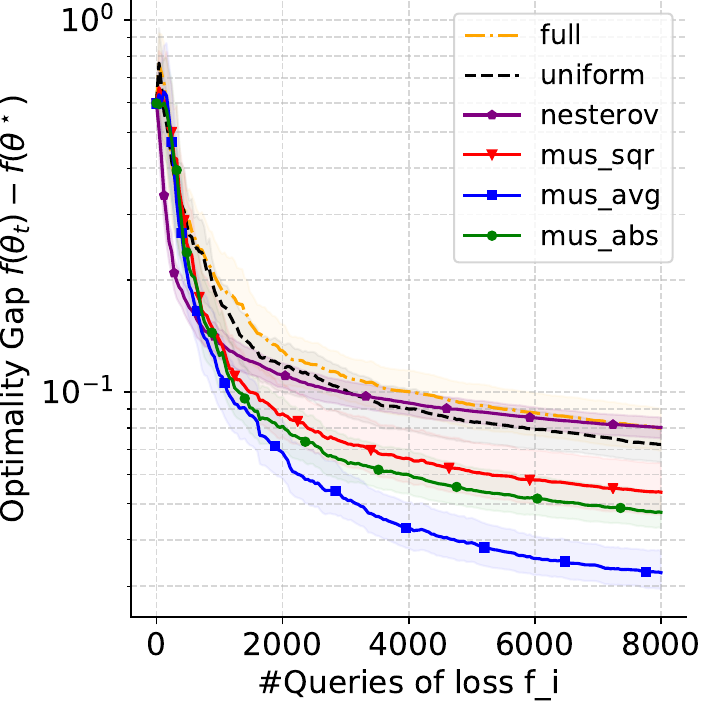}
  \label{fig:logistic_exp_n2000_p20}}
  \subfigure[$n=2000, p=50$]{
  \includegraphics[scale=0.26]{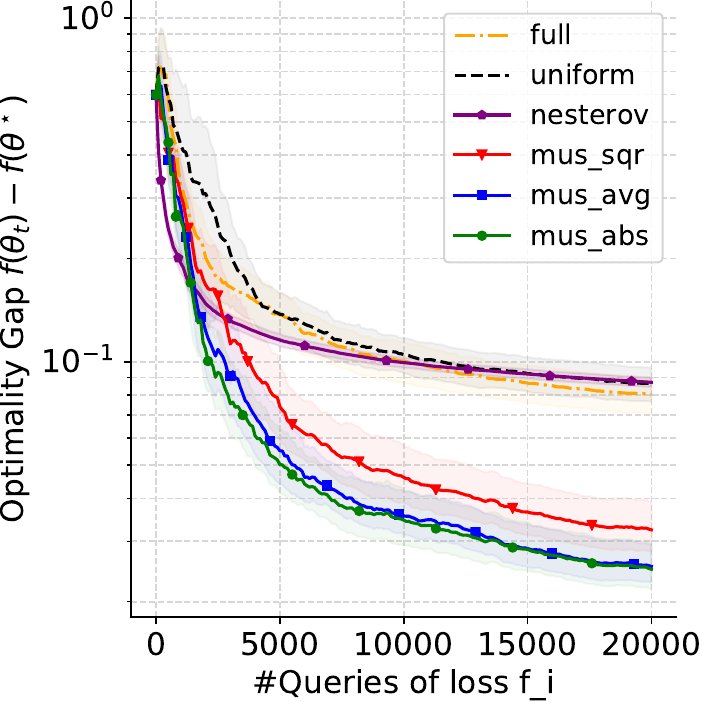}\label{fig:logistic_exp_n2000_p50}}
  \subfigure[$n=2000, p=100$]{
  \includegraphics[scale=0.26]{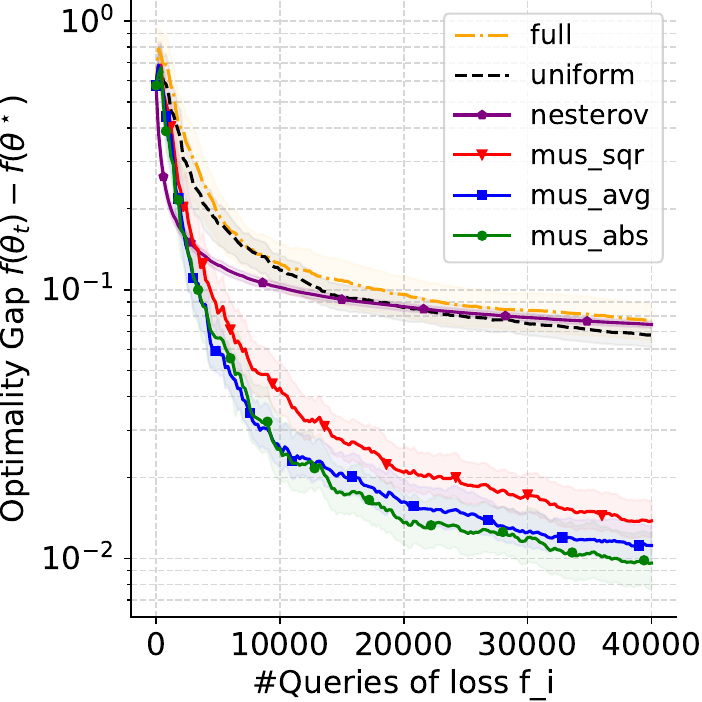}\label{logistic_exp_n2000_p100}}
  \subfigure[$n=2000, p=200$]{
  \includegraphics[scale=0.26]{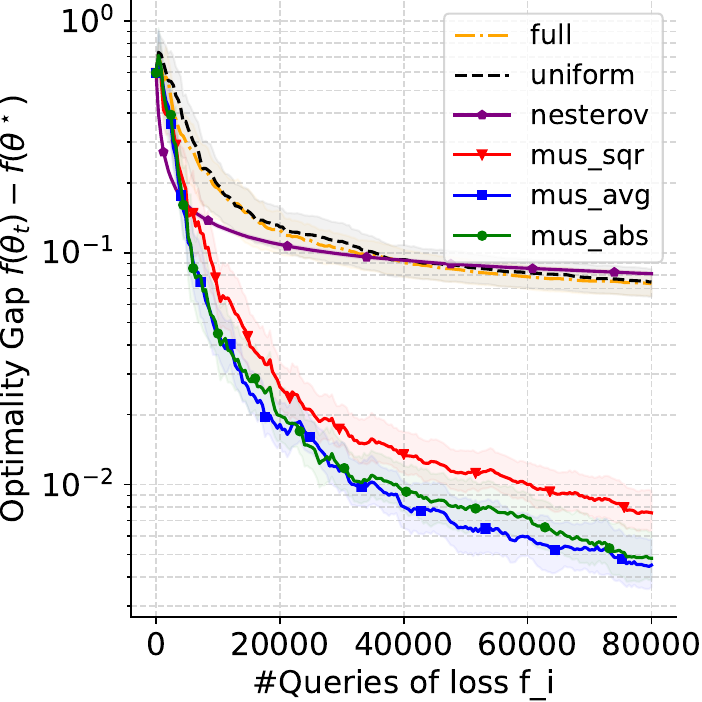}\label{logistic_exp_n2000_p200}}
  \caption{ $[f(\theta_t)-f^\star]$ for logistic Regression with $n=2000$ and $p=20,50,100,200$}
\label{fig:logistic_exp_n2000}

  \centering
  \subfigure[$n=5000, p=20$]{
  \includegraphics[scale=0.26]{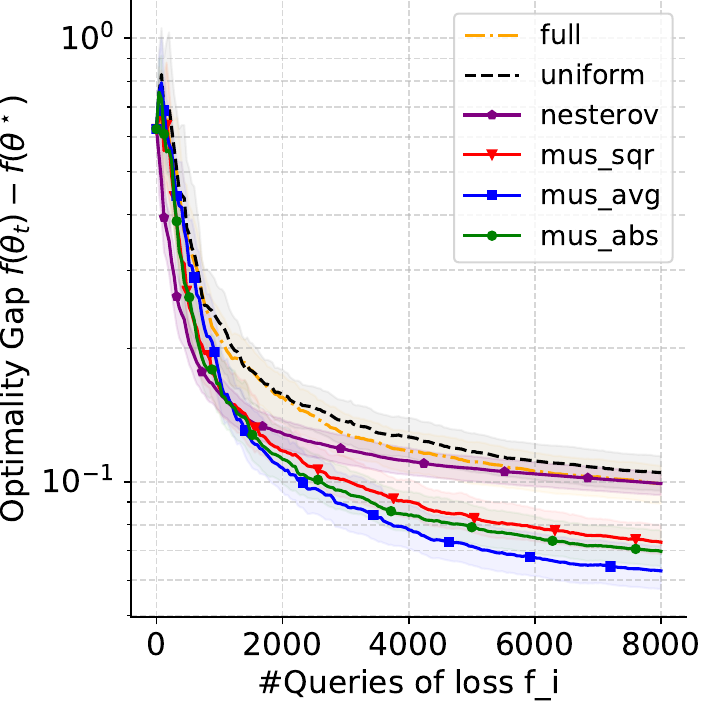}
  \label{fig:logistic_exp_n5000_p20}}
  \subfigure[$n=5000, p=50$]{
  \includegraphics[scale=0.26]{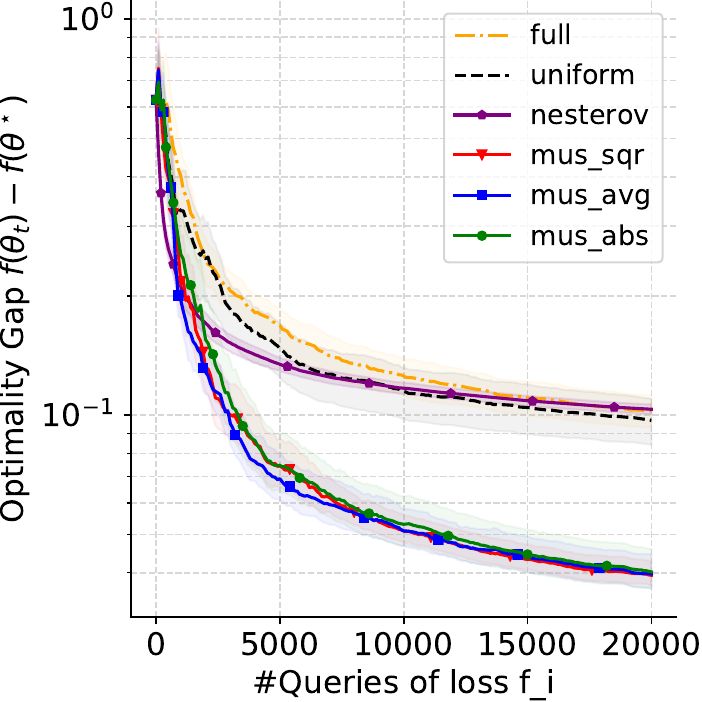}\label{fig:logistic_exp_n5000_p50}}
  \subfigure[$n=5000, p=100$]{
  \includegraphics[scale=0.26]{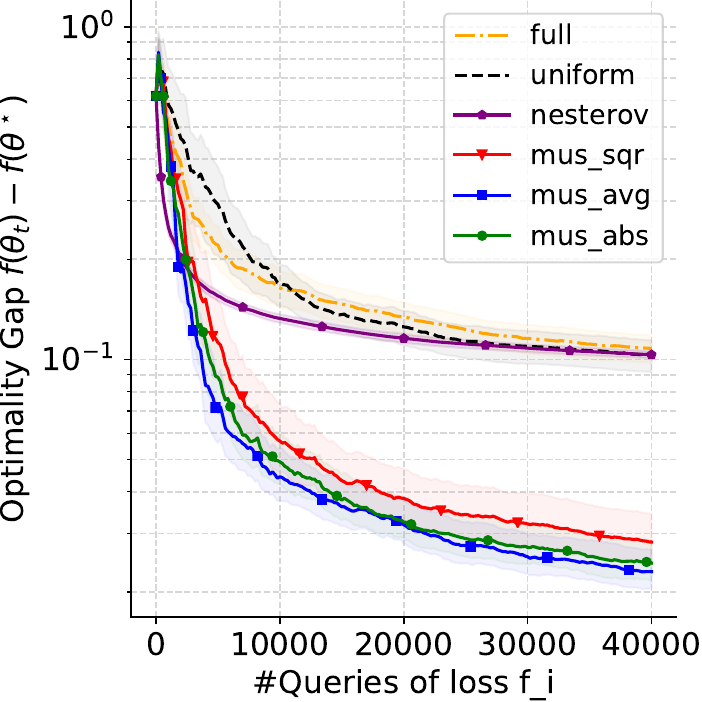}\label{logistic_exp_n5000_p100}}
  \subfigure[$n=5000, p=200$]{
  \includegraphics[scale=0.26]{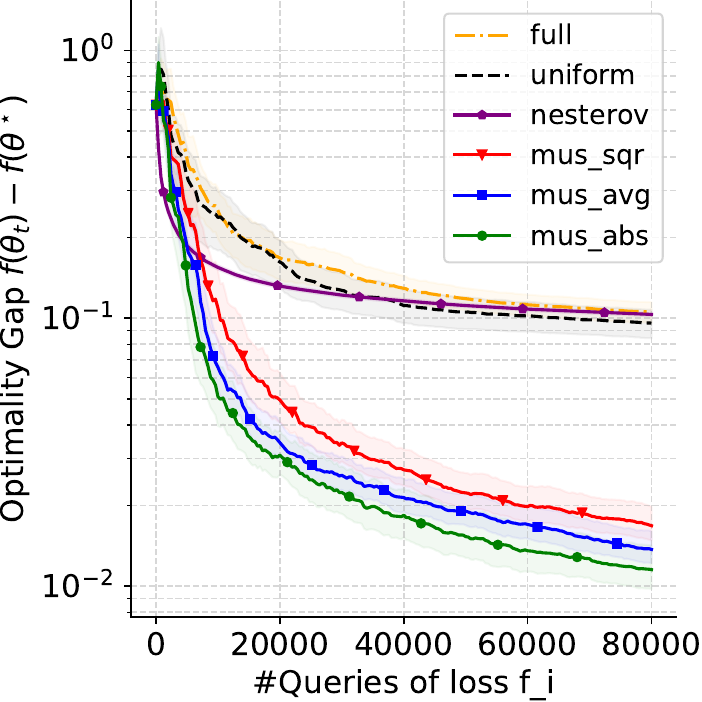}\label{logistic_exp_n5000_p200}}
  \caption{$[f(\theta_t)-f^\star]$ for logistic Regression with $n=5000$ and $p=20,50,100,200$}
\label{fig:logistic_exp_n5000}
\end{figure}
\newpage
\subsection{Effect of Importance Sampling (IS) on Ridge Regression}
We consider the same setting as in Subsection \ref{subsec:zo_ridge} and study the effect of using importance sampling weights in the update rule of MUSKETEER. Indeed, MUSKETEER update rule is defined with the following biased gradient estimate $\theta_{t+1} = \theta_{t} - \gamma_{t+1}D(\zeta_{t+1})g_t$ and the importance sampling (IS) strategy consists in adding $D_t^{-1}$ to reach an unbiased estimate
\begin{align*}
\theta_{t+1} = \theta_{t} - \gamma_{t+1}D_{t}^{-1}D(\zeta_{t+1})g_t.
\end{align*}
For the different configurations, we compare the MUSKETEER methods with their importance sampling counterparts. The Figures below show that the importance sampling methods perform similarly to the uniform coordinate sampling strategy and are therefore sub-optimal.
\begin{figure}[h]
  \centering
  \subfigure[$n=1000,p=50$]{
  \includegraphics[scale=0.26]{graph/graphs_zo_appendix/zo_ridge_n1000_p50.pdf}}
   \includegraphics[scale=0.26]{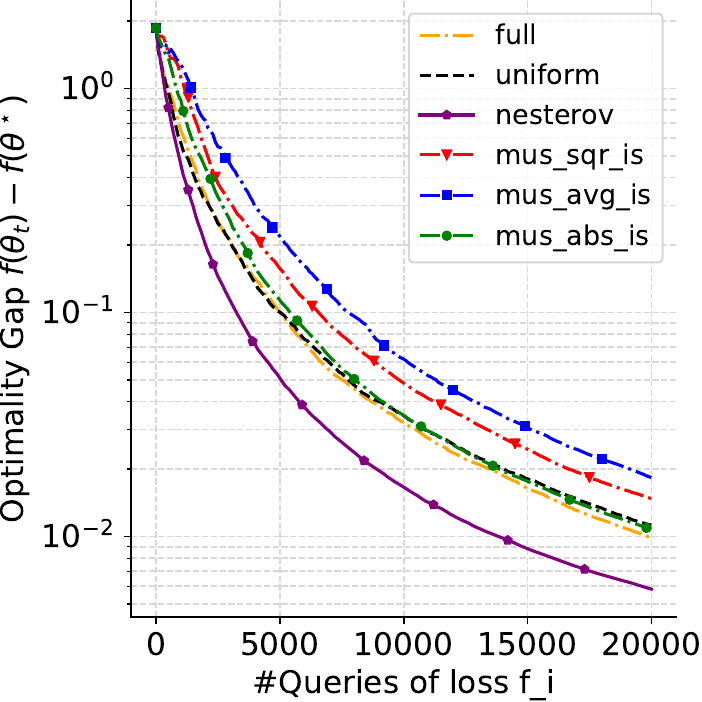}
  \subfigure[$n=1000, p=200$]{
  \includegraphics[scale=0.26]{graph/graphs_zo_appendix/zo_ridge_n1000_p200.pdf}}
  \includegraphics[scale=0.26]{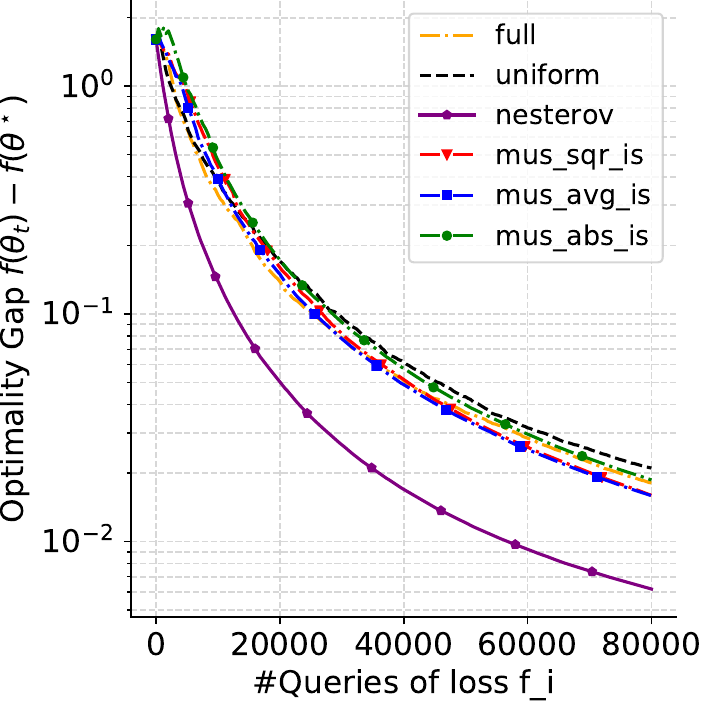}
  \caption{$[f(\theta_t)-f^\star]$ for Ridge Regression with $n=1000$ and $p=50,200$}
  
  \centering
  \subfigure[$n=2000, p=50$]{
  \includegraphics[scale=0.26]{graph/graphs_zo_appendix/zo_ridge_n2000_p50.pdf}
\includegraphics[scale=0.26]{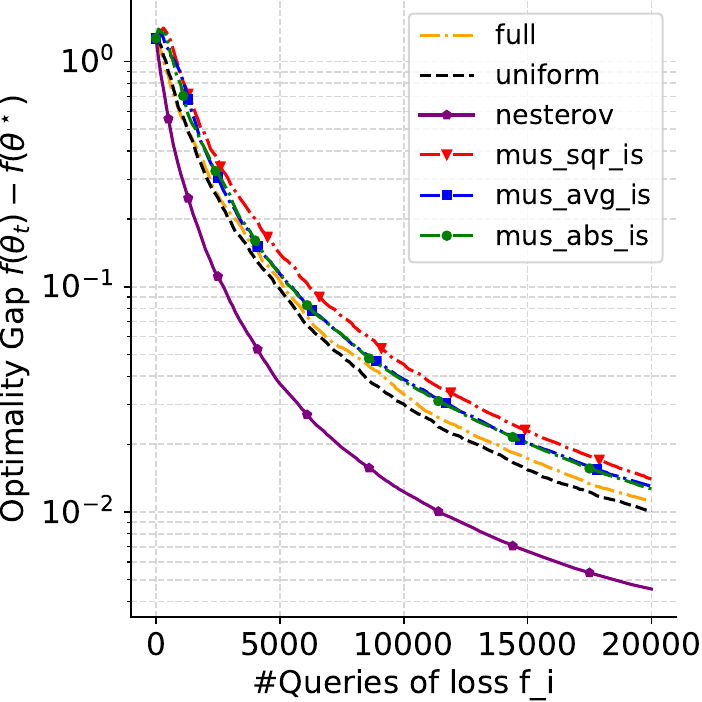}}
  \subfigure[$n=2000, p=200$]{
  \includegraphics[scale=0.26]{graph/graphs_zo_appendix/zo_ridge_n2000_p200.pdf}
  \includegraphics[scale=0.26]{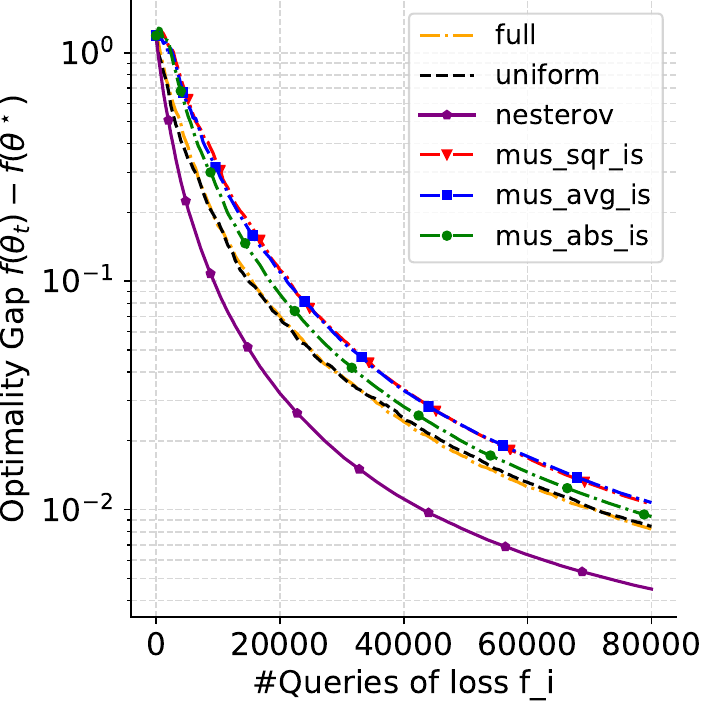}}
  \caption{$[f(\theta_t)-f^\star]$ for Ridge Regression with $n=2000$ and $p=50,200$}

  \centering

  \subfigure[$n=5000, p=50$]{
  \includegraphics[scale=0.26]{graph/graphs_zo_appendix/zo_ridge_n5000_p50.pdf}
  \includegraphics[scale=0.26]{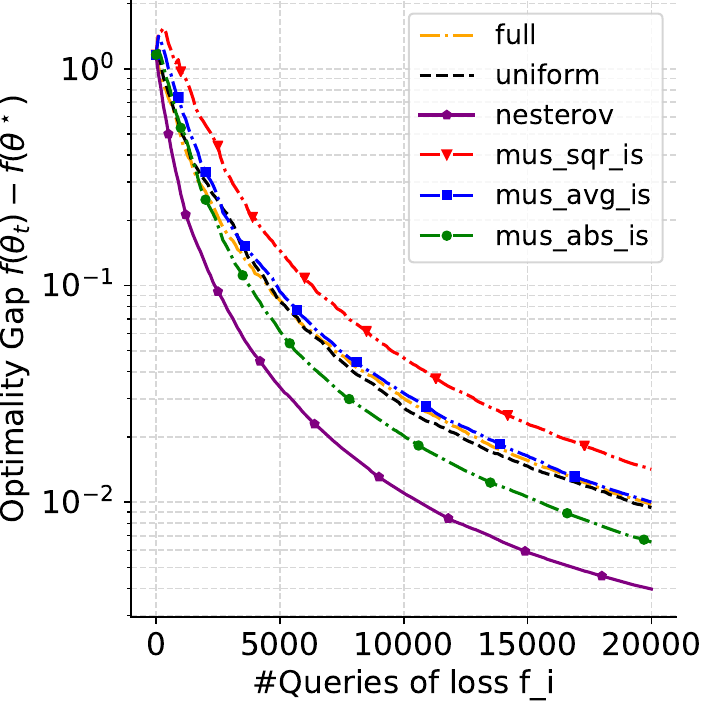}}
  \subfigure[$n=5000, p=200$]{
  \includegraphics[scale=0.26]{graph/graphs_zo_appendix/zo_ridge_n5000_p200.pdf}
  \includegraphics[scale=0.26]{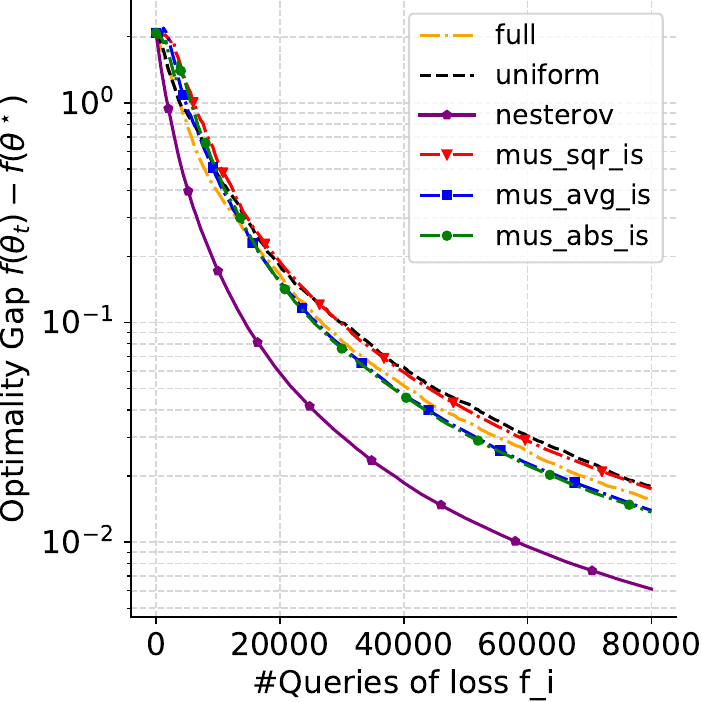}}
  \caption{$[f(\theta_t)-f^\star]$ for Ridge Regression with $n=5000$ and $p=50,200$}

\end{figure}

\newpage
\subsection{Effect of Importance Sampling (IS) on Logistic Regression}
We consider the same setting as in Subsection \ref{subsec:zo_log} and study the effect of using importance sampling weights in the update rule of MUSKETEER. Indeed, MUSKETEER update rule is defined with the following biased gradient estimate $\theta_{t+1} = \theta_{t} - \gamma_{t+1}D(\zeta_{t+1})g_t$ and the importance sampling (IS) strategy consists in adding $D_t^{-1}$ to reach an unbiased estimate
\begin{align*}
\theta_{t+1} = \theta_{t} - \gamma_{t+1}D_{t}^{-1}D(\zeta_{t+1})g_t.
\end{align*}
For the different configurations, we compare the MUSKETEER methods with their importance sampling counterparts. The Figures below show that the importance sampling methods perform similarly to the uniform coordinate sampling strategy and are therefore sub-optimal.
\begin{figure}[h]
  \centering
  \subfigure[$n=1000,p=50$]{
  \includegraphics[scale=0.26]{graph/graphs_zo_appendix/zo_log_n1000_p50.pdf}}
   \includegraphics[scale=0.26]{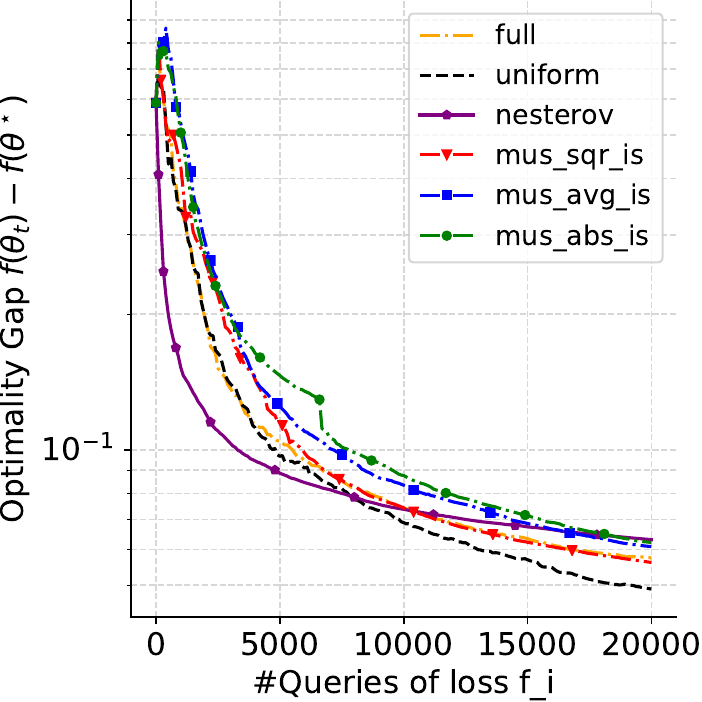}
  \subfigure[$n=1000, p=200$]{
  \includegraphics[scale=0.26]{graph/graphs_zo_appendix/zo_log_n1000_p200.pdf}}
  \includegraphics[scale=0.26]{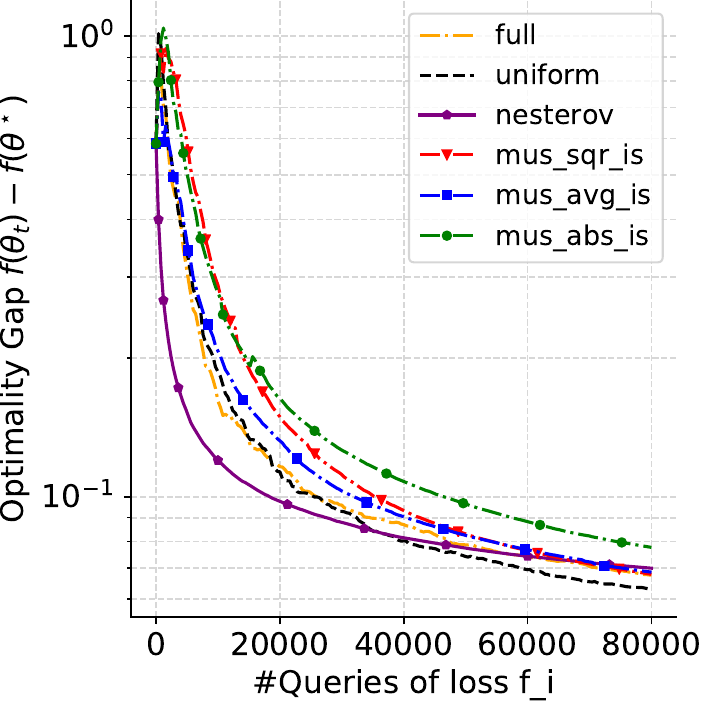}
  \caption{$[f(\theta_t)-f^\star]$ for Logistic Regression with $n=1000$ and $p=50,200$}
  
  \centering
  \subfigure[$n=2000, p=50$]{
  \includegraphics[scale=0.26]{graph/graphs_zo_appendix/zo_log_n2000_p50.pdf}
\includegraphics[scale=0.26]{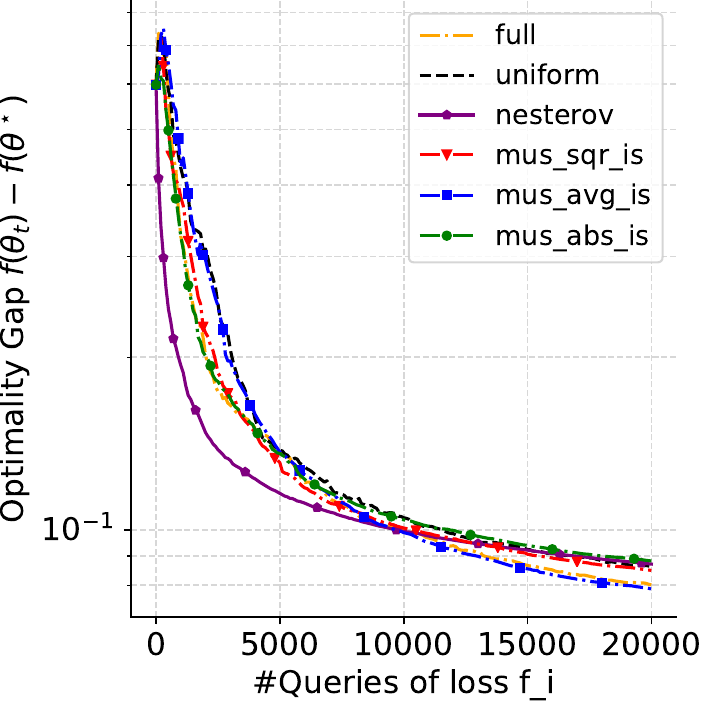}}
  \subfigure[$n=2000, p=200$]{
  \includegraphics[scale=0.26]{graph/graphs_zo_appendix/zo_log_n2000_p200.pdf}
  \includegraphics[scale=0.26]{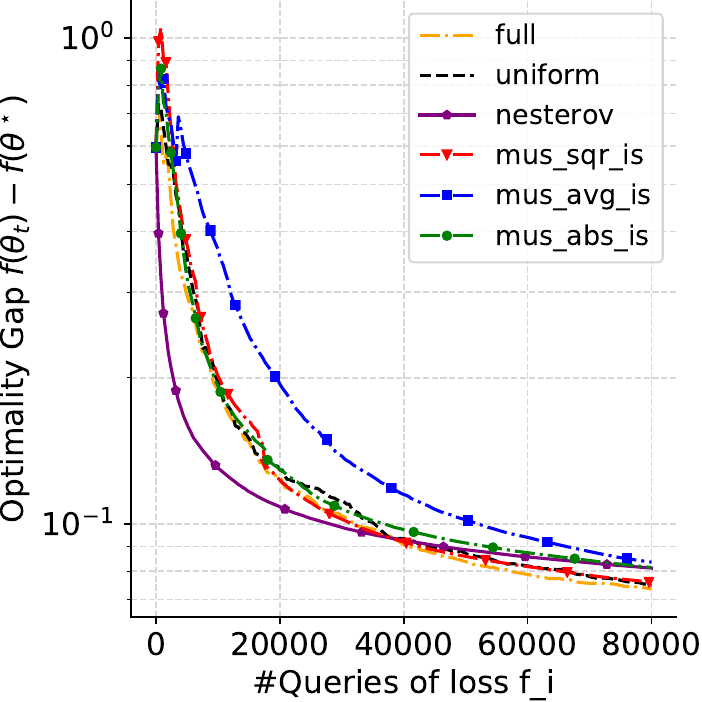}}
  \caption{$[f(\theta_t)-f^\star]$ for Logistic Regression with $n=2000$ and $p=50,200$}

  \centering

  \subfigure[$n=5000, p=50$]{
  \includegraphics[scale=0.26]{graph/graphs_zo_appendix/zo_log_n5000_p50.pdf}
  \includegraphics[scale=0.26]{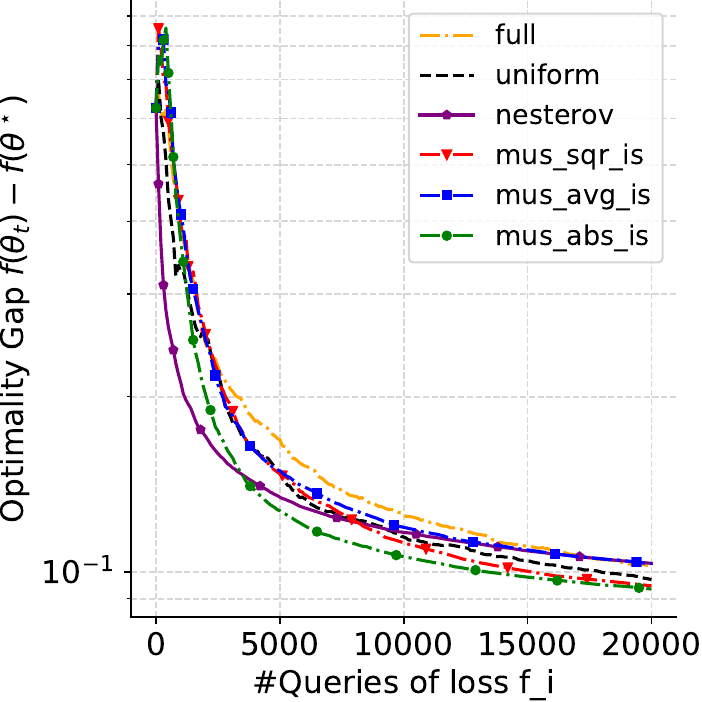}}
  \subfigure[$n=5000, p=200$]{
  \includegraphics[scale=0.26]{graph/graphs_zo_appendix/zo_log_n5000_p200.pdf}
  \includegraphics[scale=0.26]{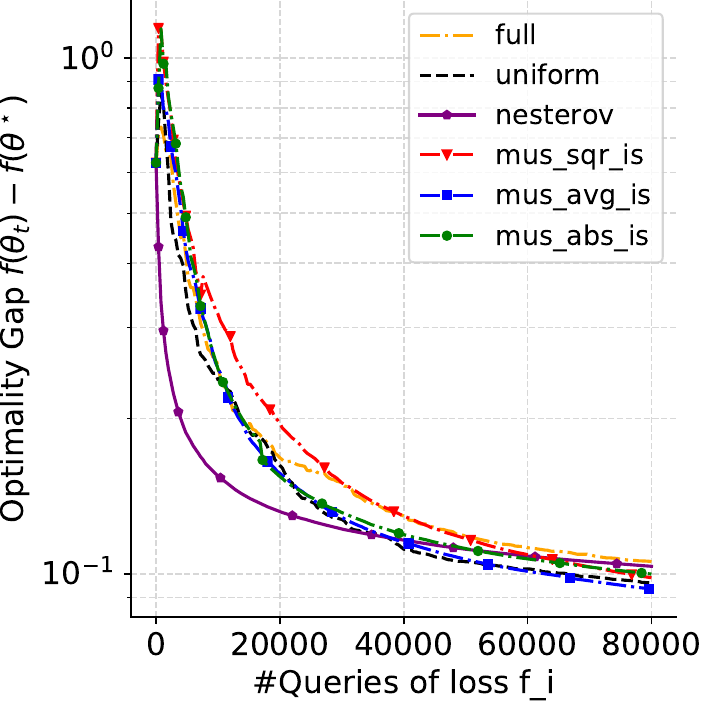}}
  \caption{$[f(\theta_t)-f^\star]$ for Logistic Regression with $n=5000$ and $p=50,200$}

\end{figure}

\newpage
\section{Further Experiments with stochastic first order methods} \label{sec:more_num}

\subsection{Comparing learning rates}

This section investigates the effect of different learning rates $\gamma_k = \gamma/k$ with $\gamma \in \{0.5; 1; 1.5; 2\}$. It reveals a safe behavior of MUSKETEER as it performs better than the other methods in all configurations with a stronger difference when dealing with small values of $\gamma$. We consider the Ridge regression problem with regularization parameter $\mu=1/n$ and run several experiments in the setting $n=5,000$ samples and dimension $p \in \{20;100;200\}$. We endow the data matrix $X$ with a block structure. The columns are drawn as $X[:,k] \sim \mathcal{N}(0,\sigma_k^2 I_n)$ with $\sigma_k^2 = k^{-\alpha}$ for all $k\in \llbracket 1,p \rrbracket$. The parameter $\alpha$ of block structure is $\alpha=8$. The gradient estimate $g$ is computed using mini-batches of size $4$. The different Figures below present the evolution of the optimality gap $t \mapsto[f(\theta_t)-f^\star]$ averaged over $20$ independent runs for $N=100$ iterations with normalized passes over coordinates.

\begin{figure}[h]
  \centering
  \subfigure[$p=20, \gamma=0.5$]{
  \includegraphics[scale=0.26]{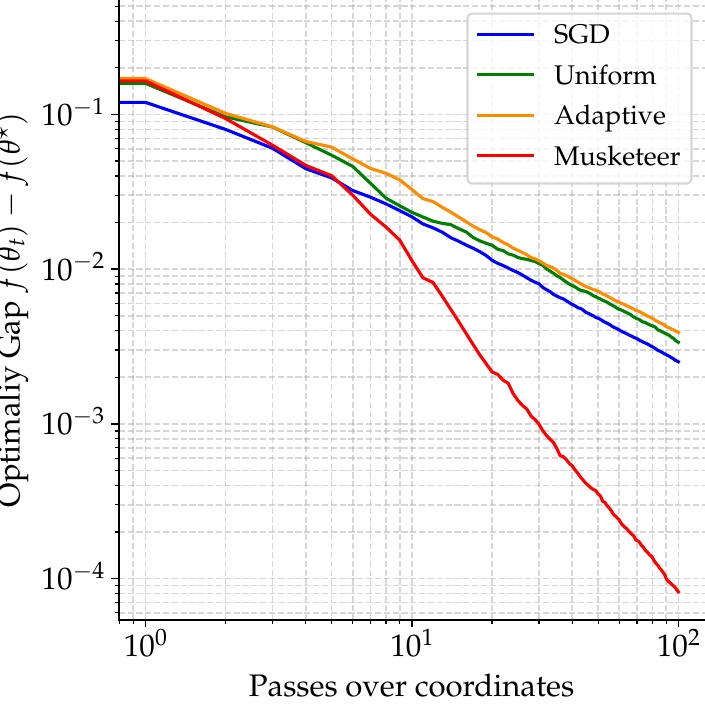}
  \label{fig:ridge_a05_p20}}
  \subfigure[$p=20, \gamma=1$]{
  \includegraphics[scale=0.26]{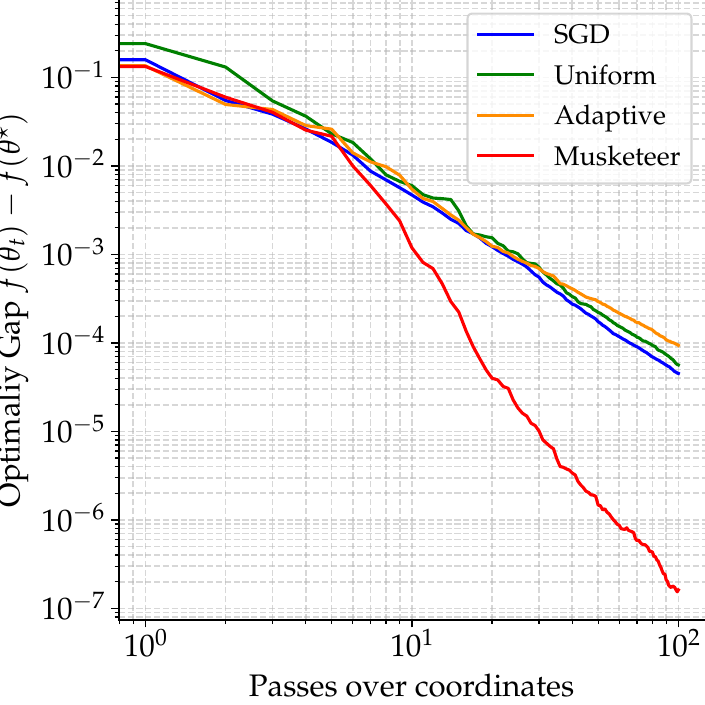}\label{fig:ridge_a1_p20}}
  \subfigure[$p=20, \gamma=1.5$]{
  \includegraphics[scale=0.26]{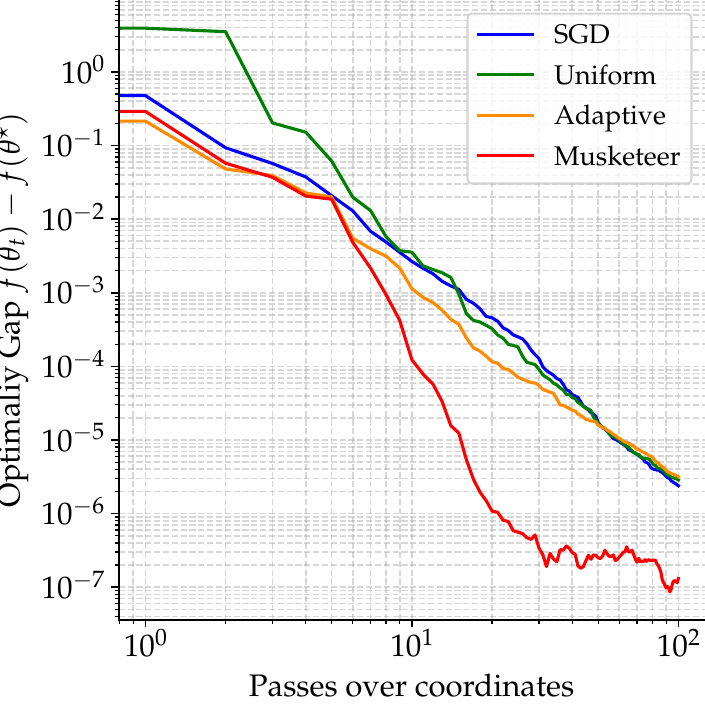}\label{ridge_a15_p20}}
  \subfigure[$p=20, \gamma=2$]{
  \includegraphics[scale=0.26]{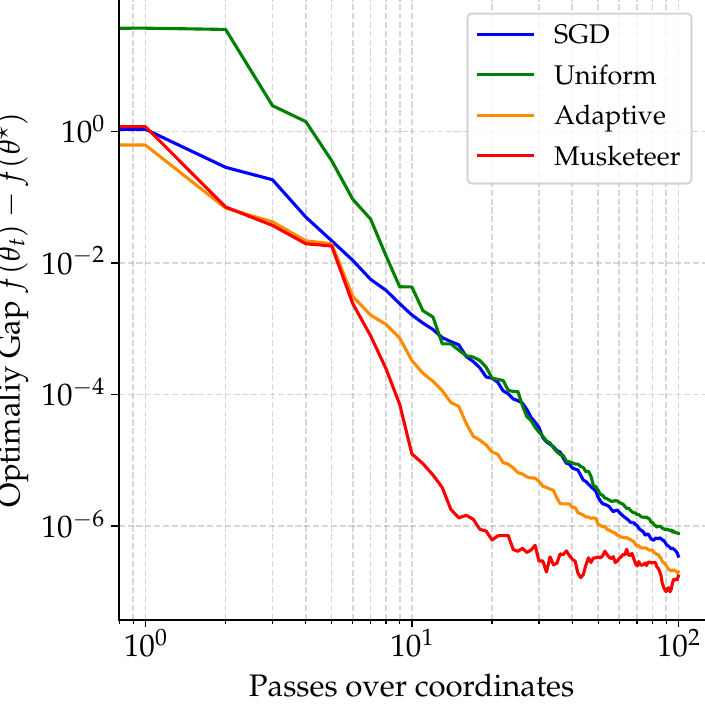}\label{ridge_a2_p20}}
  \caption{$[f(\theta_t)-f^\star]$ for Ridge Regression with $p=20$ and $\gamma \in \{0.5; 1; 1.5; 2\}$}
\label{fig:ridge_p20}

  \centering
  \subfigure[$p=100, \gamma=0.5$]{
  \includegraphics[scale=0.26]{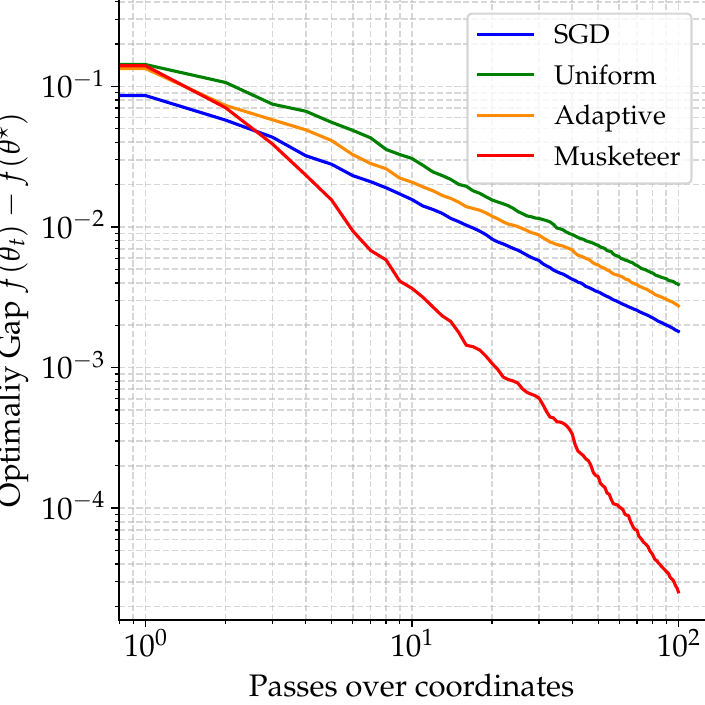}
  \label{fig:ridge_a05_p100}}
  \subfigure[$p=100, \gamma=1$]{
  \includegraphics[scale=0.26]{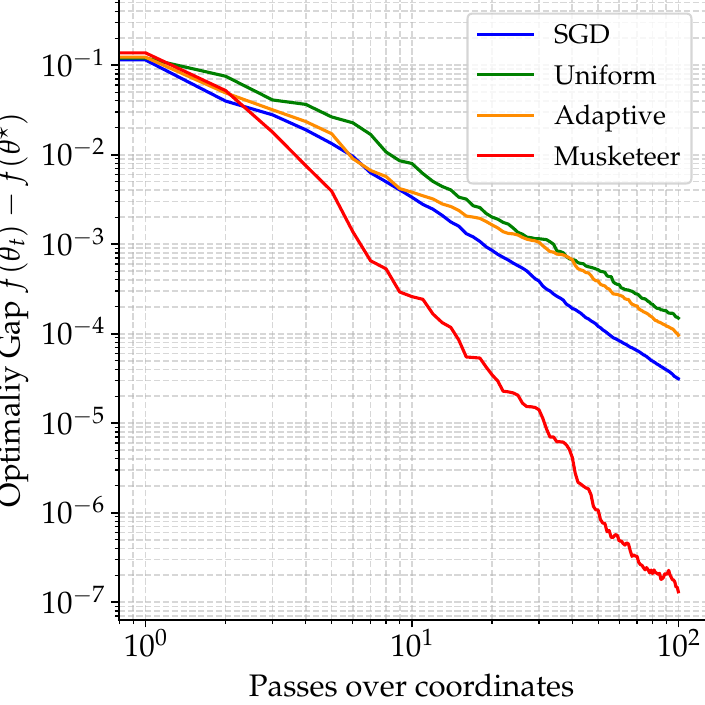}\label{fig:ridge_a1_p100}}
  \subfigure[$p=100, \gamma=1.5$]{
  \includegraphics[scale=0.26]{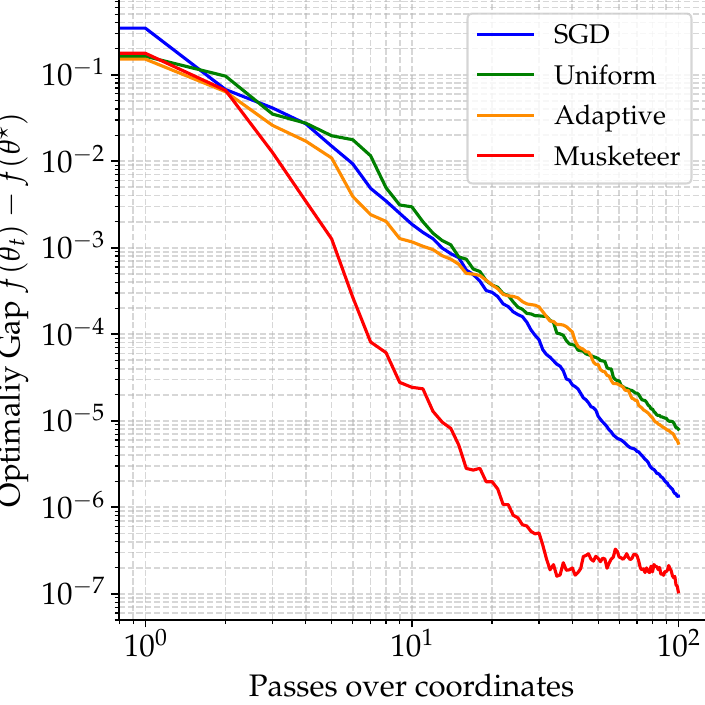}\label{ridge_a15_p100}}
  \subfigure[$p=100, \gamma=2$]{
  \includegraphics[scale=0.26]{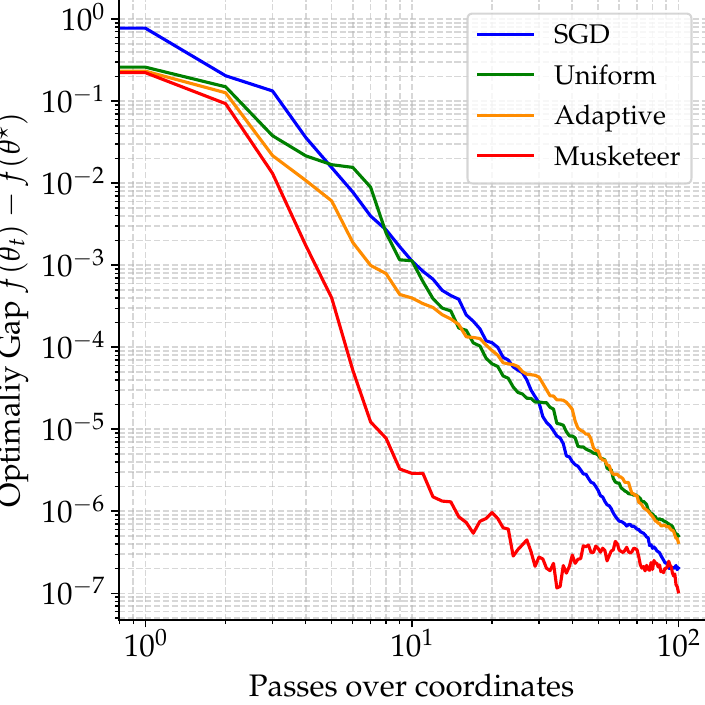}\label{ridge_a2_p100}}
  \caption{ $[f(\theta_t)-f^\star]$ for Ridge Regression with $p=100$ and $\gamma \in \{0.5; 1; 1.5; 2\}$}
\label{fig:ridge_p100}

  \centering
  \subfigure[$p=200, \gamma=0.5$]{
  \includegraphics[scale=0.26]{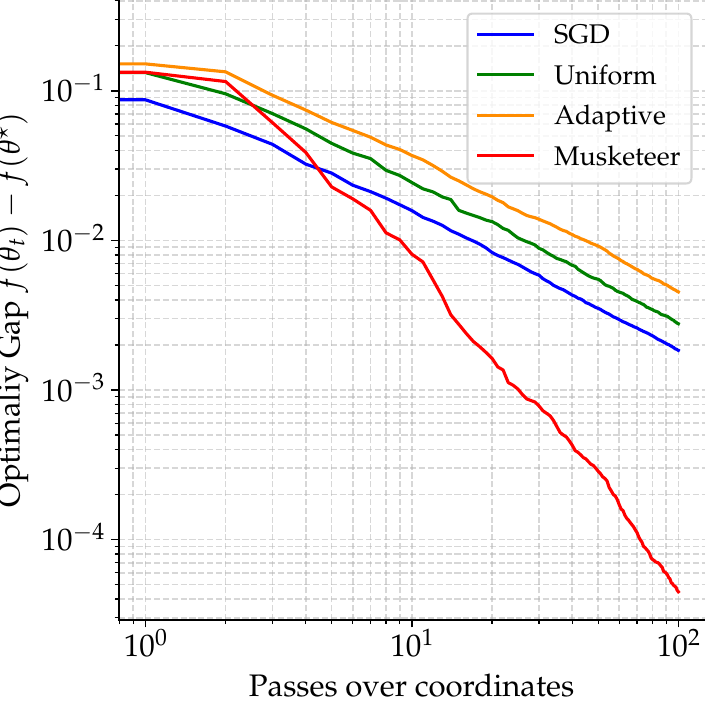}
  \label{fig:ridge_a05_p200}}
  \subfigure[$p=200, \gamma=1$]{
  \includegraphics[scale=0.26]{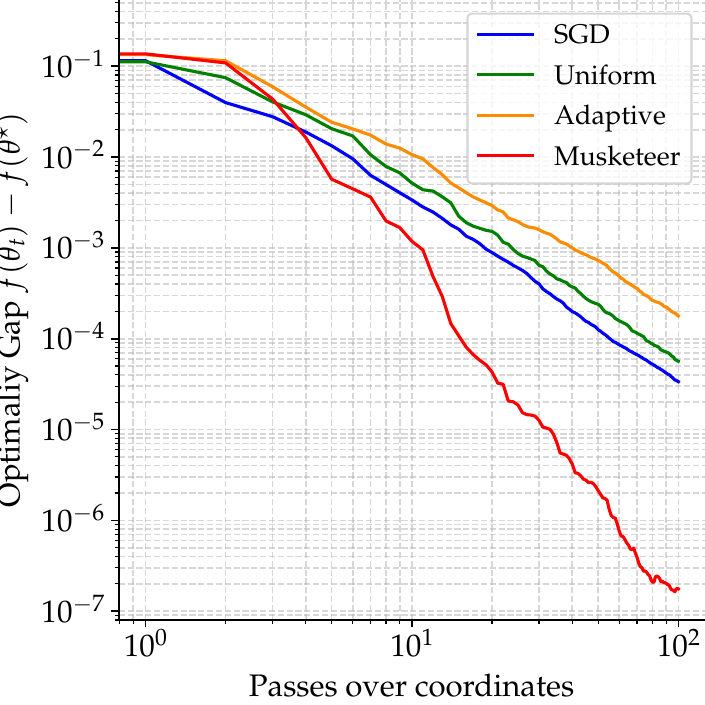}\label{fig:ridge_a1_p200}}
  \subfigure[$p=200, \gamma=1.5$]{
  \includegraphics[scale=0.26]{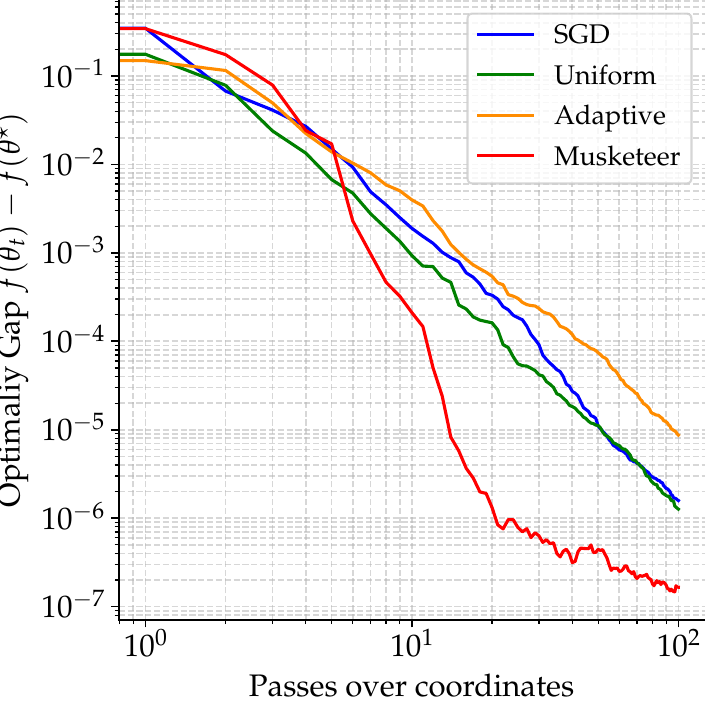}\label{ridge_a15_p200}}
  \subfigure[$p=200, \gamma=2$]{
  \includegraphics[scale=0.26]{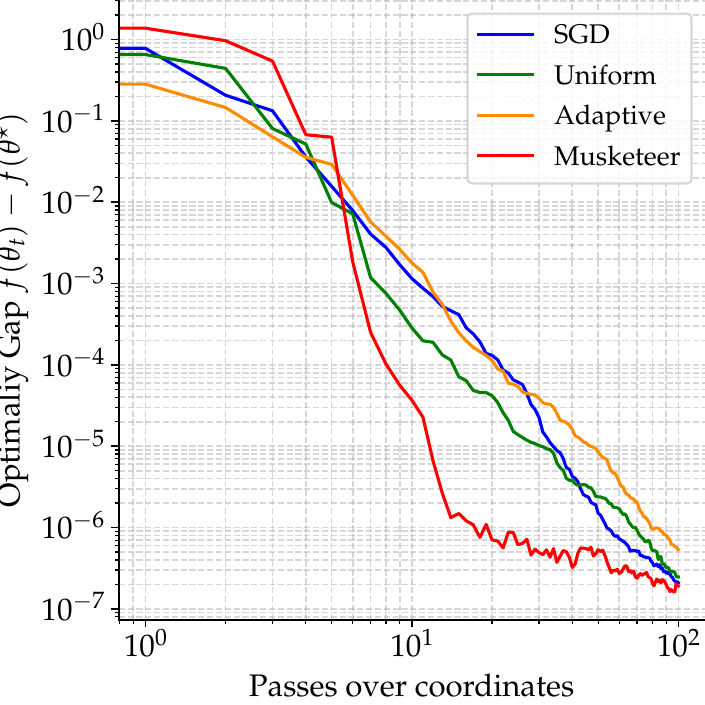}\label{ridge_a2_p200}}
  \caption{$[f(\theta_t)-f^\star]$ for Ridge Regression with $p=200$ and $\gamma \in \{0.5; 1; 1.5; 2\}$}
\label{fig:ridge_p200}
\end{figure}

\newpage
\subsection{Ridge Regression with different settings of $(n,p)$}
We consider the Ridge regression problem with the classical regularization parameter value $\mu=1/n$ and run several experiments in various settings of $(n,p)$. We endow the data matrix $X$ with a block structure. The columns are drawn as $X[:,kB+1:kB+B] \sim \mathcal{N}(0,\sigma_k^2 I_n)$ with $\sigma_k^2 = k^{-\alpha}$ for all $k\in \llbracket 0,(p/B)-1 \rrbracket$. The parameter $B$ is the block-size and is set to $B=5$ for the Ridge regression. The parameter $\alpha$ represents the block structure and is set to $\alpha=10$. The data sampling process $\xi$ of gradient estimate $g$ is computed using mini-batches of size $8$. The different Figures below present the evolution of the optimality gap $t \mapsto[f(\theta_t)-f^\star]$ averaged over $20$ independent runs for $N=1000$ iterations with normalized passes over coordinates. The learning rates is the same for all methods, fixed to $\gamma_k = 1/k$. The different settings are: number of samples $n \in \{1,000;2,000;5,000\}$ and dimension $p \in \{20;50;100;200\}$.

\begin{figure}[h]
  \centering
  \subfigure[$n=1000, p=20$]{
  \includegraphics[scale=0.26]{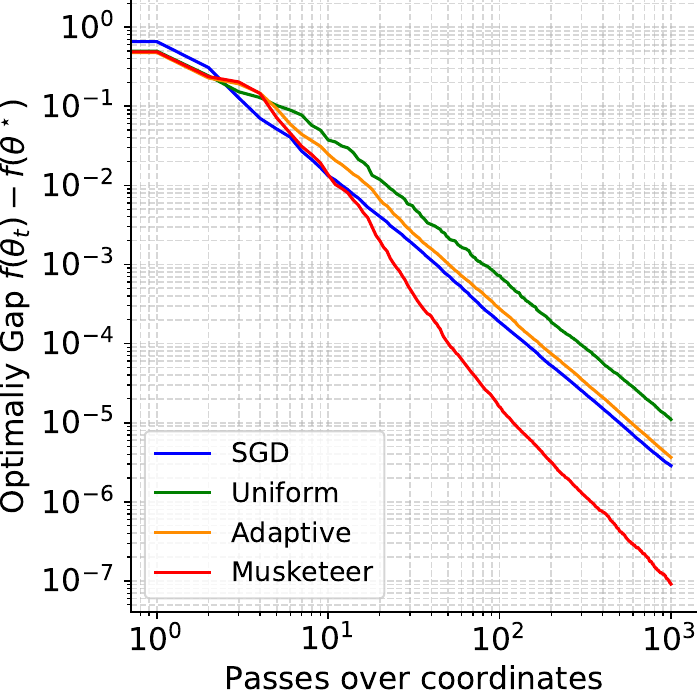}
  \label{fig:ridge_n1000_p20}}
  \subfigure[$n=1000, p=50$]{
  \includegraphics[scale=0.26]{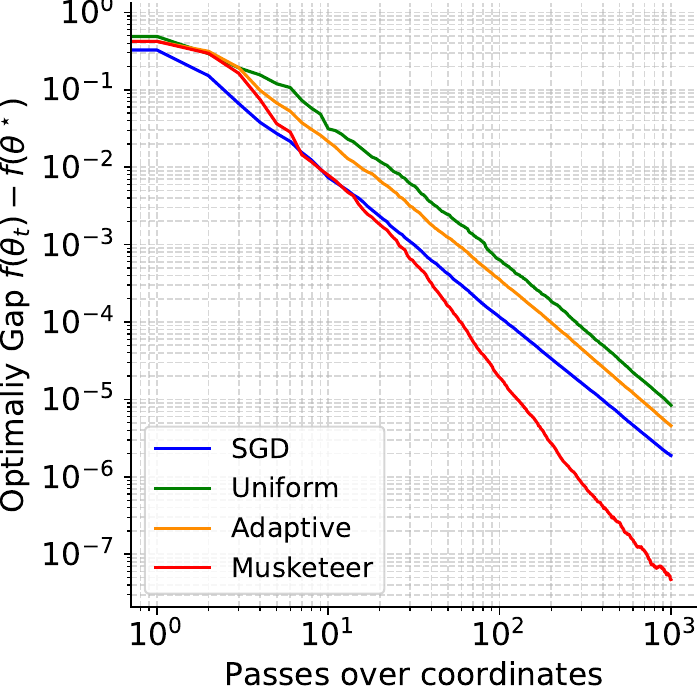}\label{fig:ridge_n1000_p50}}
  \subfigure[$n=1000, p=100$]{
  \includegraphics[scale=0.26]{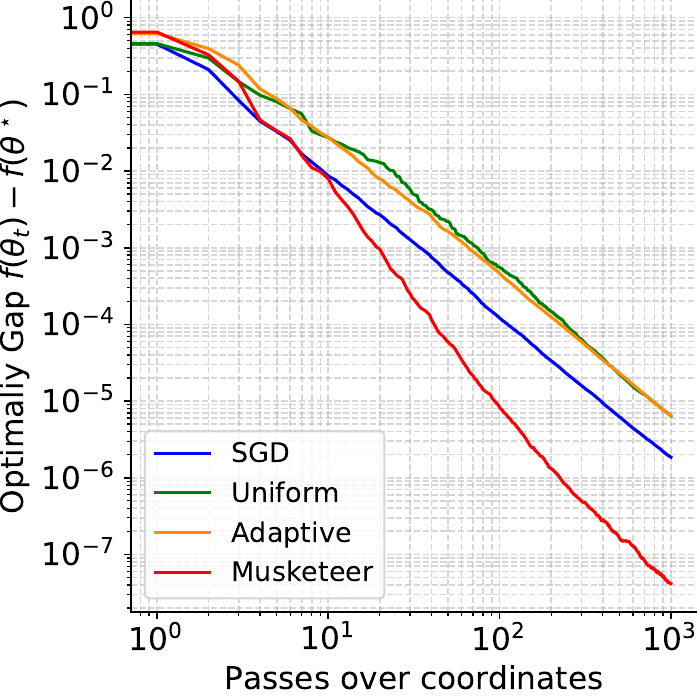}\label{ridge_n1000_p100}}
  \subfigure[$n=1000, p=200$]{
  \includegraphics[scale=0.26]{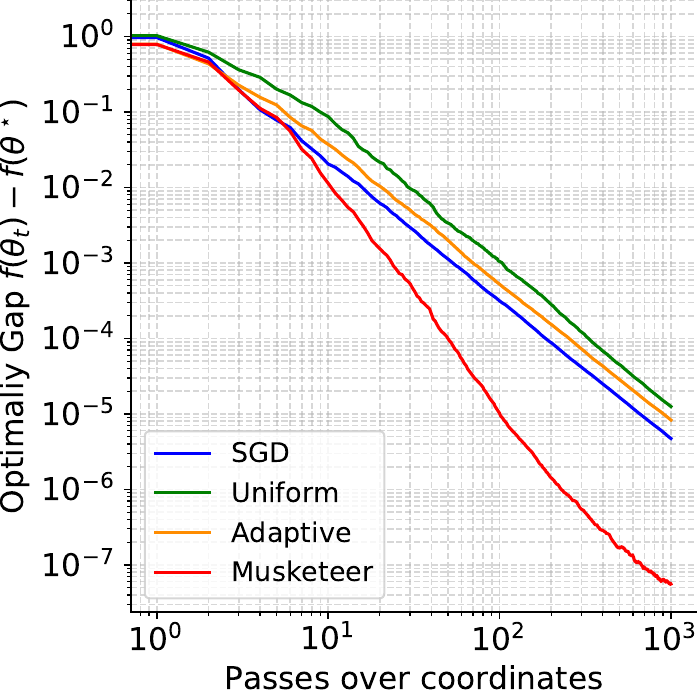}\label{ridge_n1000_p200}}
  \caption{$[f(\theta_t)-f^\star]$ for Ridge Regression with $n=1000$ and $p=20,50,100,200$}
\label{fig:ridge_n1000}

  \centering
  \subfigure[$n=2000, p=20$]{
  \includegraphics[scale=0.26]{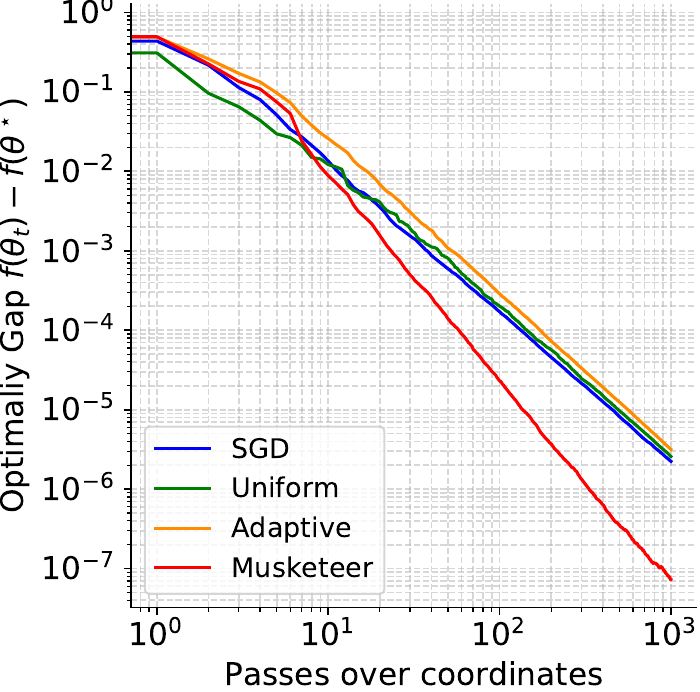}
  \label{fig:ridge_n2000_p20}}
  \subfigure[$n=2000, p=50$]{
  \includegraphics[scale=0.26]{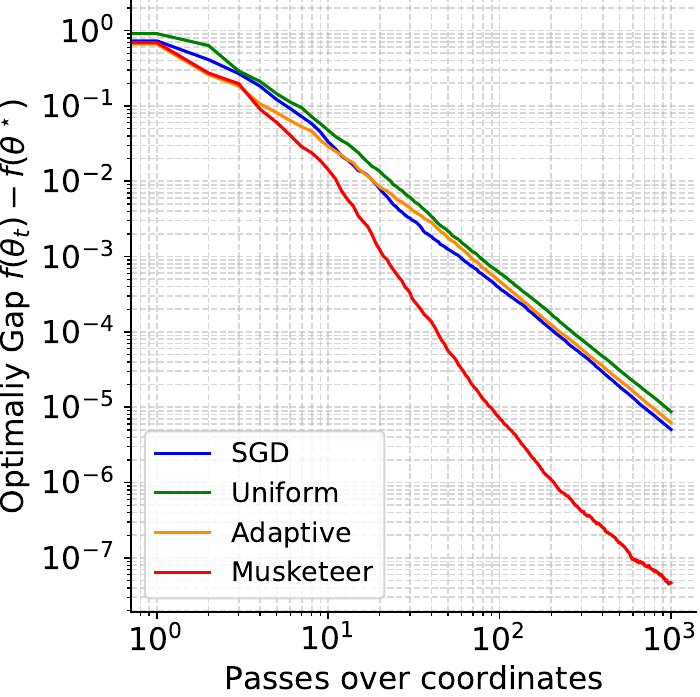}\label{fig:ridge_n2000_p50}}
  \subfigure[$n=2000, p=100$]{
  \includegraphics[scale=0.26]{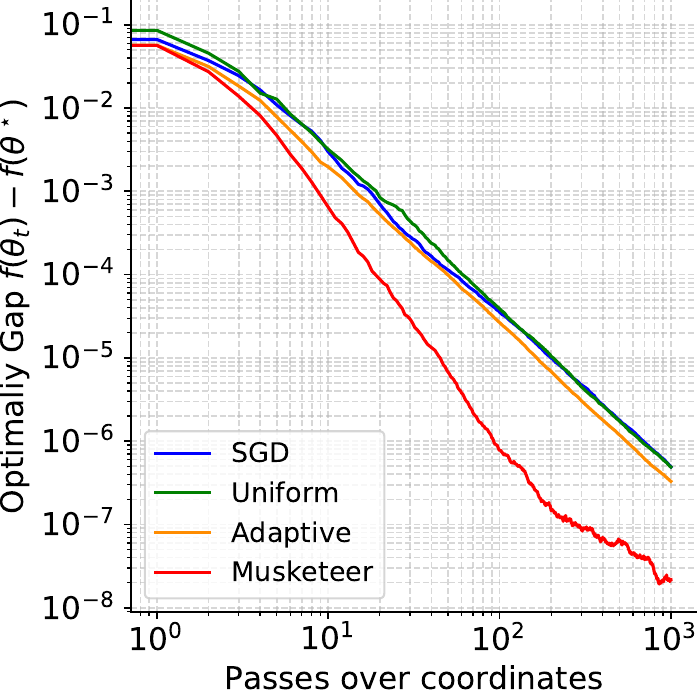}\label{ridge_n2000_p100}}
  \subfigure[$n=2000, p=200$]{
  \includegraphics[scale=0.26]{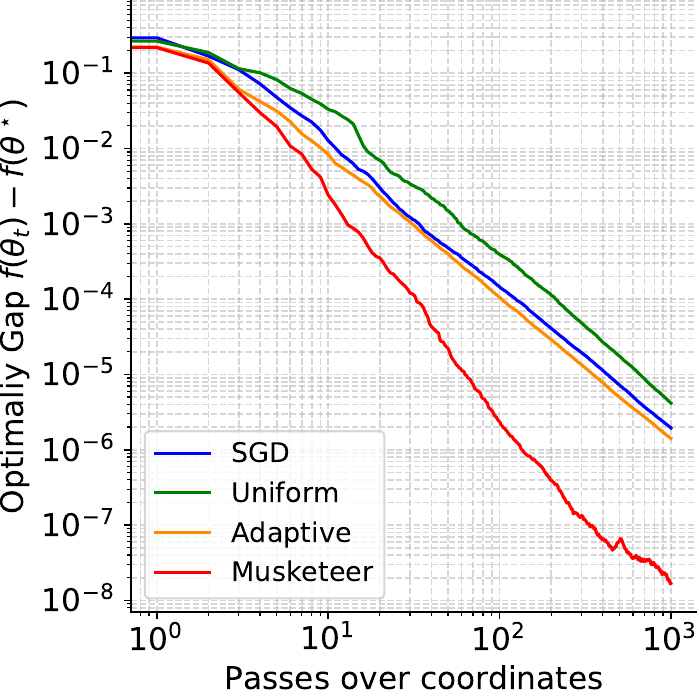}\label{ridge_n2000_p200}}
  \caption{$[f(\theta_t)-f^\star]$ for Ridge Regression with $n=2000$ and $p=20,50,100,200$}
\label{fig:ridge_n2000}

  \centering
  \subfigure[$n=5000, p=20$]{
  \includegraphics[scale=0.26]{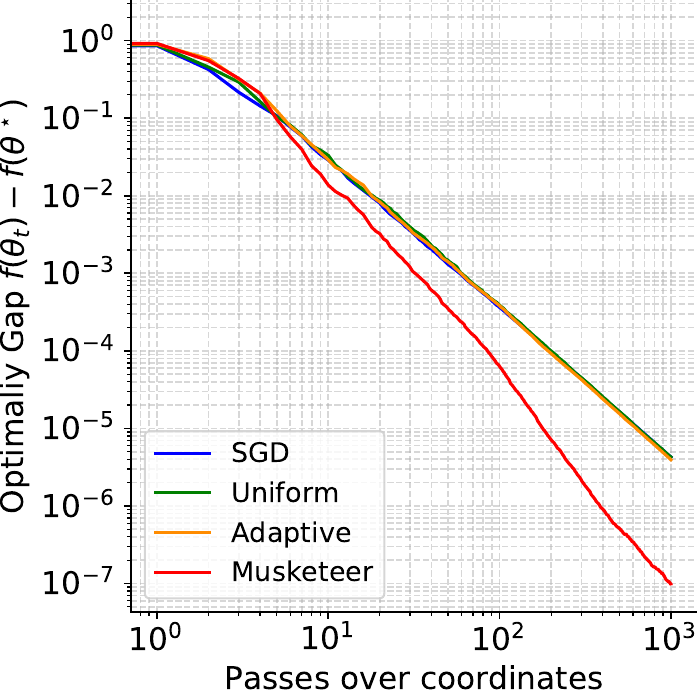}
  \label{fig:ridge_n5000_p20}}
  \subfigure[$n=5000, p=50$]{
  \includegraphics[scale=0.26]{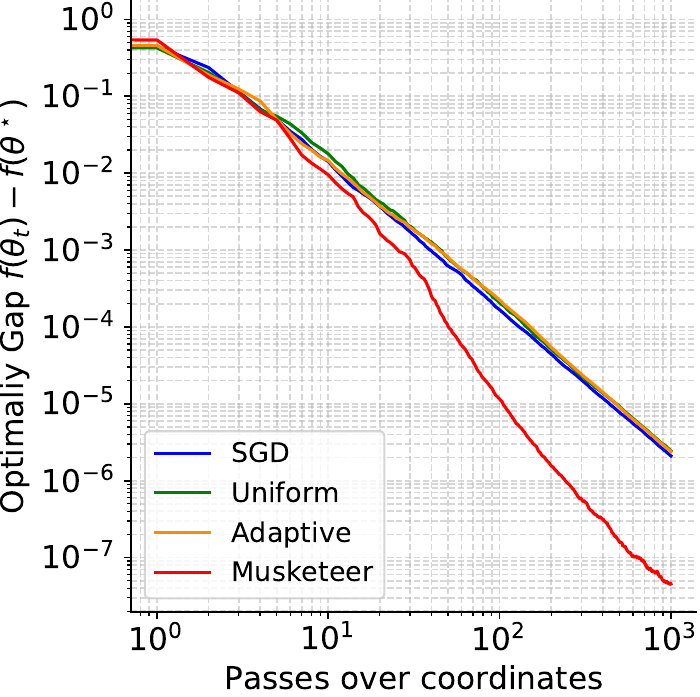}\label{fig:ridge_n5000_p50}}
  \subfigure[$n=5000, p=100$]{
  \includegraphics[scale=0.26]{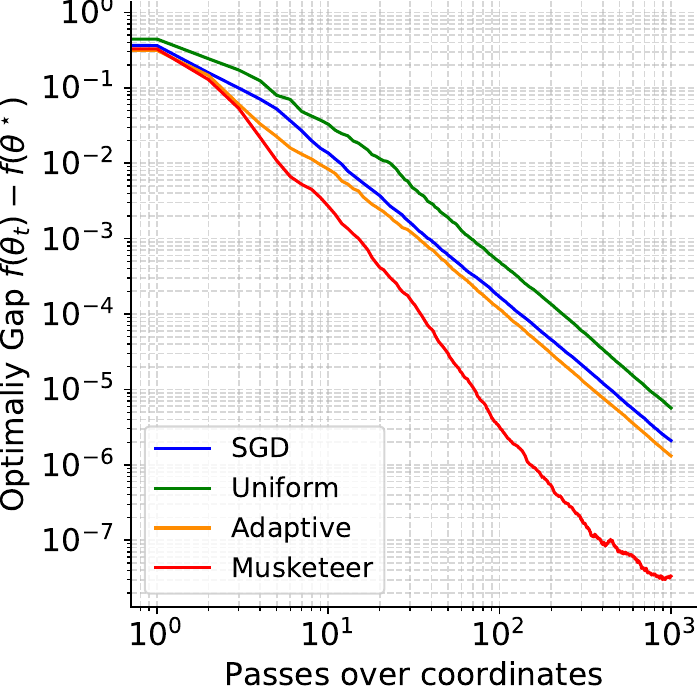}\label{ridge_n5000_p100}}
  \subfigure[$n=5000, p=200$]{
  \includegraphics[scale=0.26]{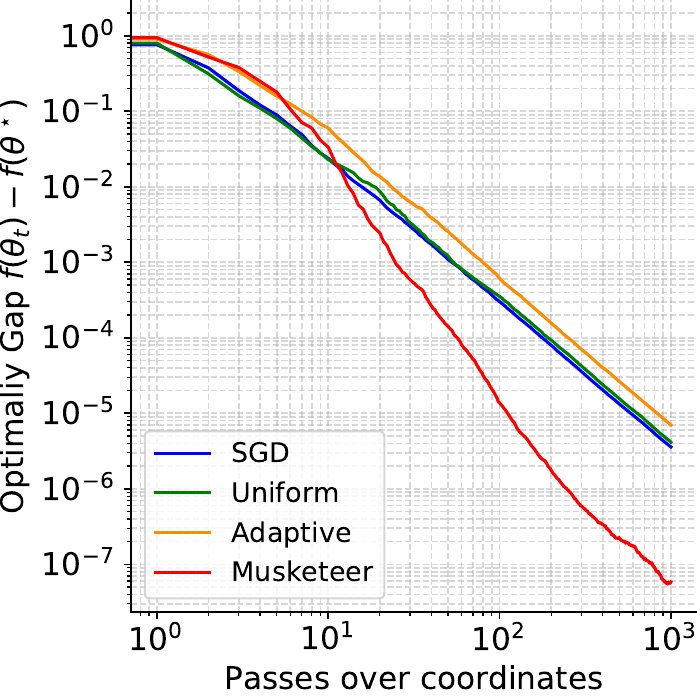}\label{ridge_n5000_p200}}
  \caption{$[f(\theta_t)-f^\star]$ for Ridge Regression with $n=5000$ and $p=20,50,100,200$}
\label{fig:ridge_n5000}
\end{figure}

\newpage
\subsection{Logistic Regression with different settings of $(n,p)$}

We consider the $\ell_2$-Logistic regression problem with the classical regularization parameter value $\mu=1/n$ and run several experiments in various settings of $(n,p)$. We endow the data matrix $X$ with a block structure. The columns are drawn as $X[:,kB+1:kB+B] \sim \mathcal{N}(0,\sigma_k^2 I_n)$ with $\sigma_k^2 = k^{-\alpha}$ for all $k\in \llbracket 1,(p/B)-1 \rrbracket$. The parameter $B$ is the block-size and is set to $B=2$ for the Logistic regression. The parameter $\alpha$ represents the block structure and is set to $\alpha=5$. The data sampling process $\xi$ of gradient estimate $g$ is computed using mini-batches of size $32$. The different Figures below present the evolution of the optimality gap $t \mapsto[f(\theta_t)-f^\star]$ averaged over $20$ independent runs for $N=1000$ iterations with normalized passes over coordinates. The learning rates is the same for all methods, fixed to $\gamma_k = 1/k$. The different settings are: number of samples $n \in \{1,000;2,000;5,000\}$ and dimension $p \in \{20;50;100;200\}$.

\begin{figure}[h]
  \centering
  \subfigure[$n=1000, p=20$]{
  \includegraphics[scale=0.26]{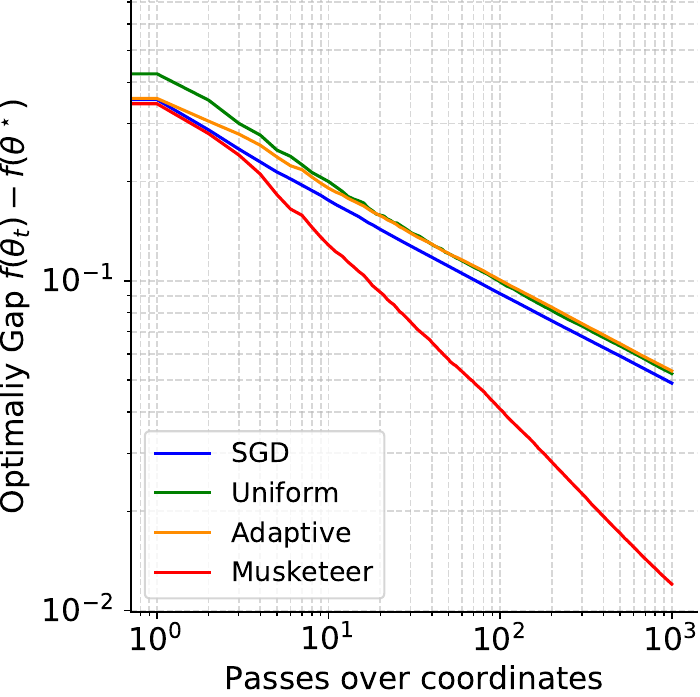}
  \label{fig:logistic_n1000_p20}}
  \subfigure[$n=1000, p=50$]{
  \includegraphics[scale=0.26]{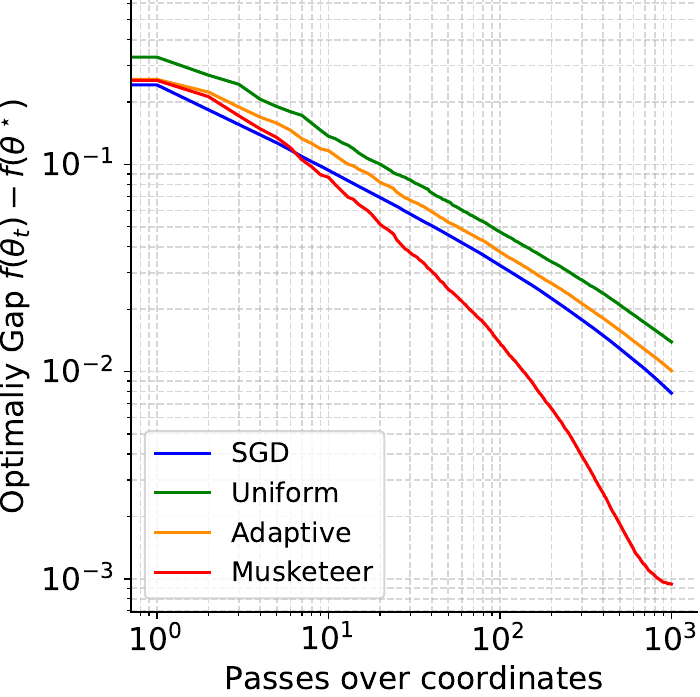}\label{fig:logistic_n1000_p50}}
  \subfigure[$n=1000, p=100$]{
  \includegraphics[scale=0.26]{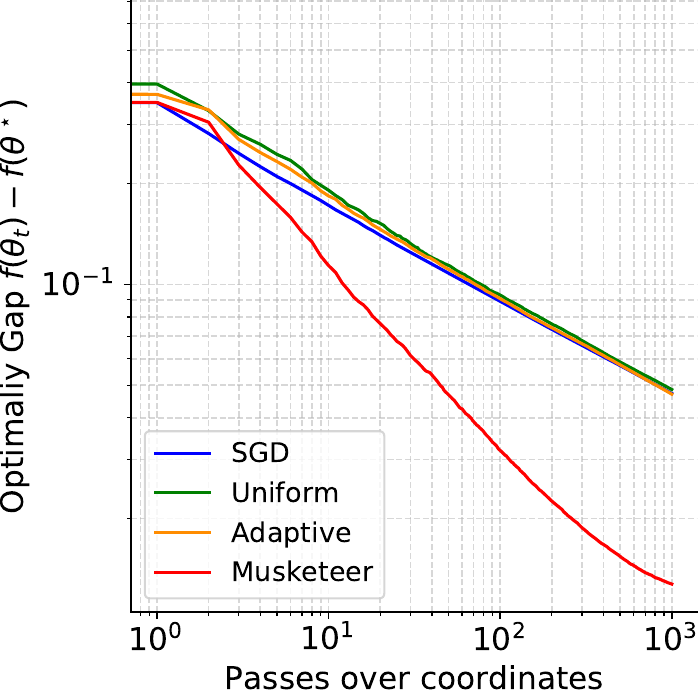}\label{logistic_n1000_p100}}
  \subfigure[$n=1000, p=200$]{
  \includegraphics[scale=0.26]{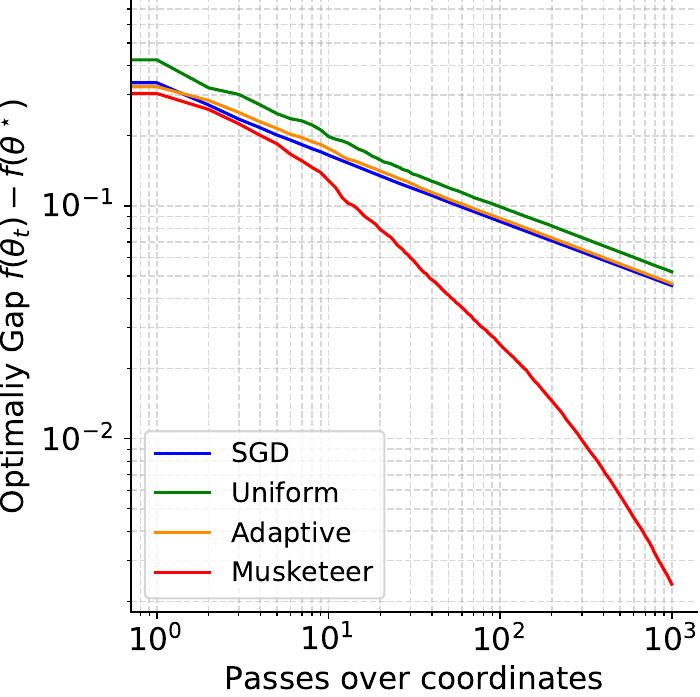}\label{logistic_n1000_p200}}
  \caption{$[f(\theta_t)-f^\star]$ for logistic Regression with $n=1000$ and $p=20,50,100,200$}
\label{fig:logistic_n1000}

  \centering
  \subfigure[$n=2000, p=20$]{
  \includegraphics[scale=0.26]{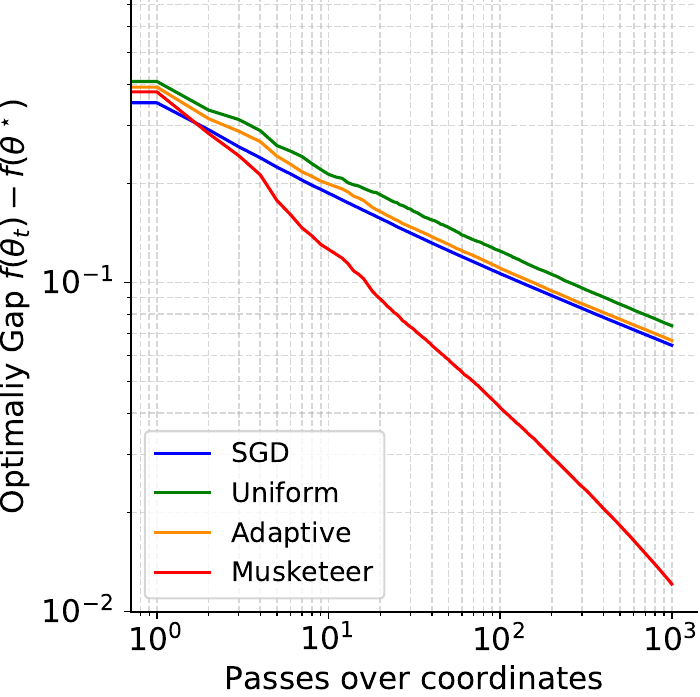}
  \label{fig:logistic_n2000_p20}}
  \subfigure[$n=2000, p=50$]{
  \includegraphics[scale=0.26]{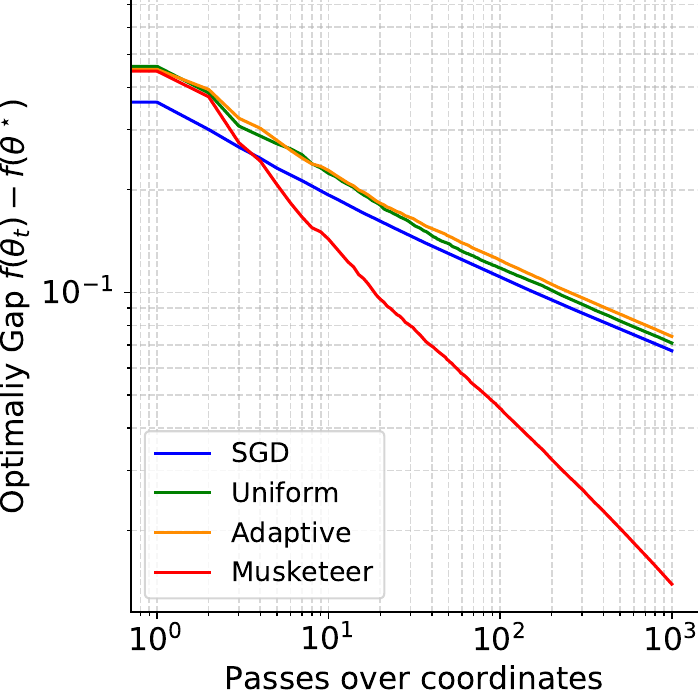}\label{fig:logistic_n2000_p50}}
  \subfigure[$n=2000, p=100$]{
  \includegraphics[scale=0.26]{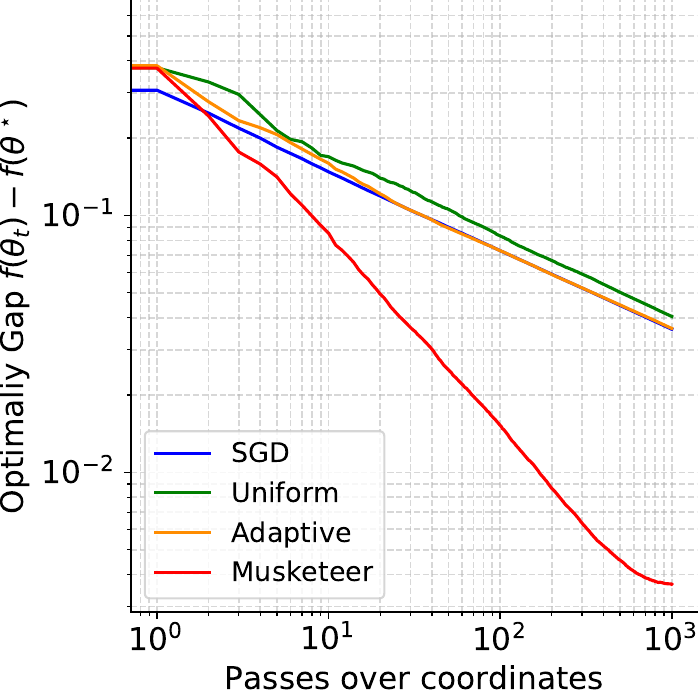}\label{logistic_n2000_p100}}
  \subfigure[$n=2000, p=200$]{
  \includegraphics[scale=0.26]{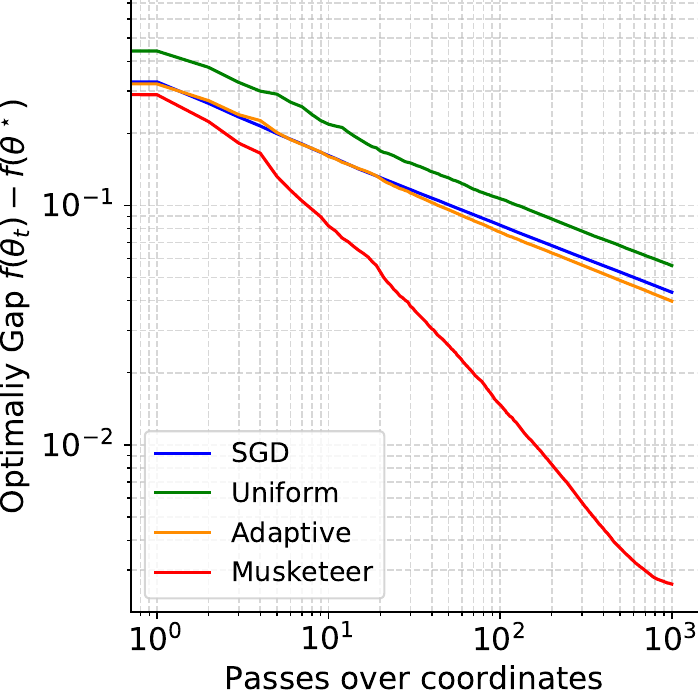}\label{logistic_n2000_p200}}
  \caption{ $[f(\theta_t)-f^\star]$ for logistic Regression with $n=2000$ and $p=20,50,100,200$}
\label{fig:logistic_n2000}

  \centering
  \subfigure[$n=5000, p=20$]{
  \includegraphics[scale=0.26]{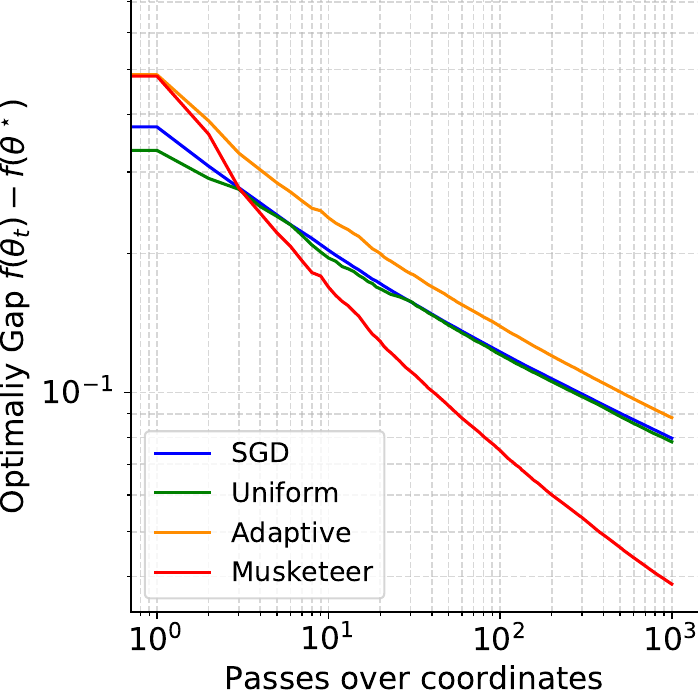}
  \label{fig:logistic_n5000_p20}}
  \subfigure[$n=5000, p=50$]{
  \includegraphics[scale=0.26]{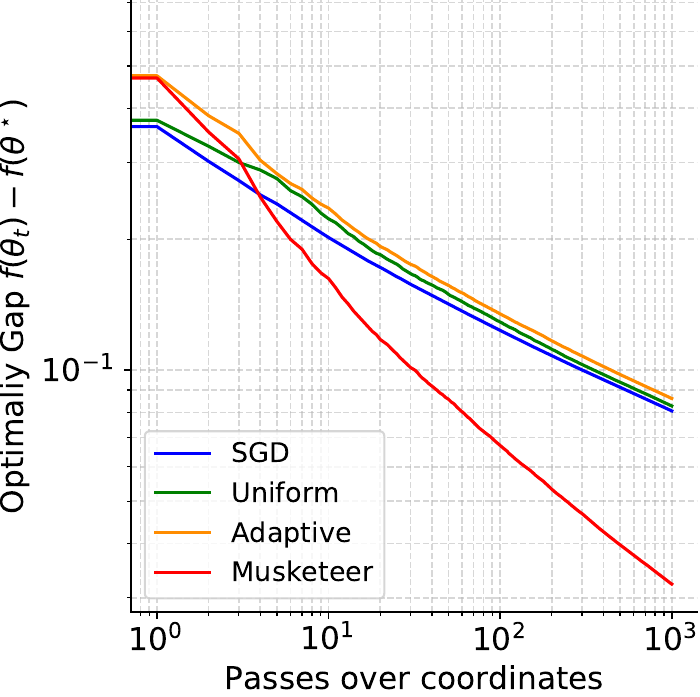}\label{fig:logistic_n5000_p50}}
  \subfigure[$n=5000, p=100$]{
  \includegraphics[scale=0.26]{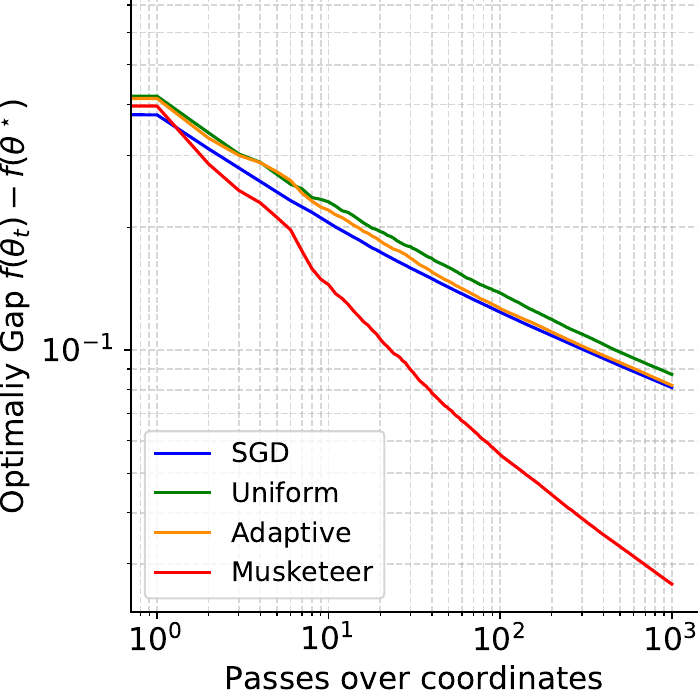}\label{logistic_n5000_p100}}
  \subfigure[$n=5000, p=200$]{
  \includegraphics[scale=0.26]{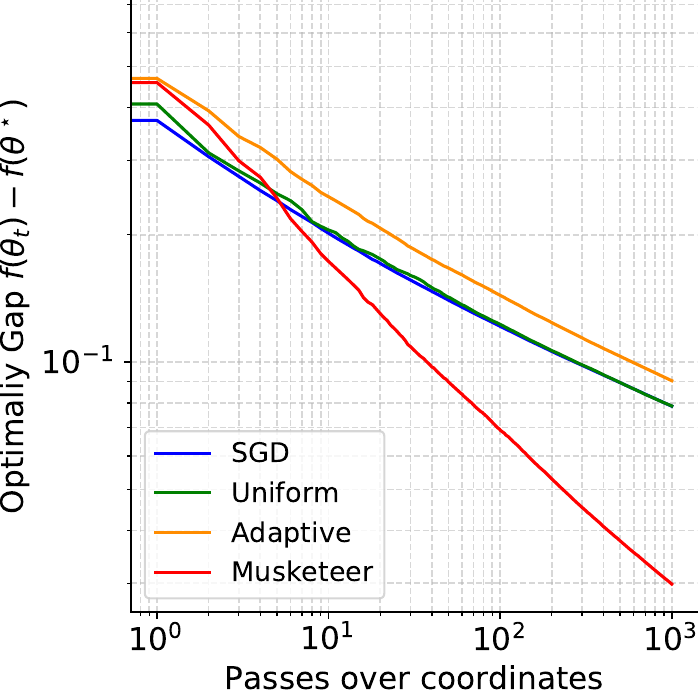}\label{logistic_n5000_p200}}
  \caption{$[f(\theta_t)-f^\star]$ for logistic Regression with $n=5000$ and $p=20,50,100,200$}
\label{fig:logistic_n5000}
\end{figure}

\newpage
\bibliography{ref}

\begin{thebibliography}{65}
\providecommand{\natexlab}[1]{#1}
\providecommand{\url}[1]{\texttt{#1}}
\expandafter\ifx\csname urlstyle\endcsname\relax
  \providecommand{\doi}[1]{doi: #1}\else
  \providecommand{\doi}{doi: \begingroup \urlstyle{rm}\Url}\fi

\bibitem[Agarwal et~al.(2012)Agarwal, Bartlett, Ravikumar, and
  Wainwright]{agarwal2012information}
Alekh Agarwal, Peter~L Bartlett, Pradeep Ravikumar, and Martin~J Wainwright.
\newblock Information-theoretic lower bounds on the oracle complexity of
  stochastic convex optimization.
\newblock \emph{IEEE Transactions on Information Theory}, 58\penalty0
  (5):\penalty0 3235--3249, 2012.

\bibitem[Ajalloeian and Stich(2020)]{ajalloeian2020analysis}
Ahmad Ajalloeian and Sebastian~U Stich.
\newblock Analysis of sgd with biased gradient estimators.
\newblock \emph{arXiv preprint arXiv:2008.00051}, 2020.

\bibitem[Alistarh et~al.(2017)Alistarh, Grubic, Li, Tomioka, and
  Vojnovic]{alistarh2017qsgd}
Dan Alistarh, Demjan Grubic, Jerry Li, Ryota Tomioka, and Milan Vojnovic.
\newblock Qsgd: Communication-efficient sgd via gradient quantization and
  encoding.
\newblock In \emph{Advances in Neural Information Processing Systems}, pages
  1709--1720, 2017.

\bibitem[Alistarh et~al.(2018)Alistarh, Hoefler, Johansson, Konstantinov,
  Khirirat, and Renggli]{alistarh2018convergence}
Dan Alistarh, Torsten Hoefler, Mikael Johansson, Nikola Konstantinov, Sarit
  Khirirat, and C{\'e}dric Renggli.
\newblock The convergence of sparsified gradient methods.
\newblock In \emph{Advances in Neural Information Processing Systems}, pages
  5973--5983, 2018.

\bibitem[Allen-Zhu et~al.(2016)Allen-Zhu, Qu, Richt{\'a}rik, and
  Yuan]{allen2016even}
Zeyuan Allen-Zhu, Zheng Qu, Peter Richt{\'a}rik, and Yang Yuan.
\newblock Even faster accelerated coordinate descent using non-uniform
  sampling.
\newblock In \emph{International Conference on Machine Learning}, pages
  1110--1119, 2016.

\bibitem[Auer et~al.(2002{\natexlab{a}})Auer, Cesa-Bianchi, and
  Fischer]{auer2002finite}
Peter Auer, Nicolo Cesa-Bianchi, and Paul Fischer.
\newblock Finite-time analysis of the multiarmed bandit problem.
\newblock \emph{Machine learning}, 47\penalty0 (2-3):\penalty0 235--256,
  2002{\natexlab{a}}.

\bibitem[Auer et~al.(2002{\natexlab{b}})Auer, Cesa-Bianchi, Freund, and
  Schapire]{auer2002nonstochastic}
Peter Auer, Nicolo Cesa-Bianchi, Yoav Freund, and Robert~E Schapire.
\newblock The nonstochastic multiarmed bandit problem.
\newblock \emph{SIAM journal on computing}, 32\penalty0 (1):\penalty0 48--77,
  2002{\natexlab{b}}.

\bibitem[Beck and Tetruashvili(2013)]{beck2013convergence}
Amir Beck and Luba Tetruashvili.
\newblock On the convergence of block coordinate descent type methods.
\newblock \emph{SIAM journal on Optimization}, 23\penalty0 (4):\penalty0
  2037--2060, 2013.

\bibitem[Bercu et~al.(2015)Bercu, Delyon, and Rio]{bercu2015concentration}
Bernard Bercu, Bernard Delyon, and Emmanuel Rio.
\newblock \emph{Concentration inequalities for sums and martingales}.
\newblock Springer, 2015.

\bibitem[Bertsekas and Tsitsiklis(2000)]{bertsekas2000gradient}
Dimitri~P Bertsekas and John~N Tsitsiklis.
\newblock Gradient convergence in gradient methods with errors.
\newblock \emph{SIAM Journal on Optimization}, 10\penalty0 (3):\penalty0
  627--642, 2000.

\bibitem[Bottou et~al.(2018)Bottou, Curtis, and
  Nocedal]{bottou2018optimization}
L{\'e}on Bottou, Frank~E Curtis, and Jorge Nocedal.
\newblock Optimization methods for large-scale machine learning.
\newblock \emph{Siam Review}, 60\penalty0 (2):\penalty0 223--311, 2018.

\bibitem[Boyer and Godichon-Baggioni(2020)]{boyer2020asymptotic}
Claire Boyer and Antoine Godichon-Baggioni.
\newblock On the asymptotic rate of convergence of stochastic newton algorithms
  and their weighted averaged versions.
\newblock \emph{arXiv preprint arXiv:2011.09706}, 2020.

\bibitem[Csiba et~al.(2015)Csiba, Qu, and Richt{\'a}rik]{csiba2015stochastic}
Dominik Csiba, Zheng Qu, and Peter Richt{\'a}rik.
\newblock Stochastic dual coordinate ascent with adaptive probabilities.
\newblock In \emph{International Conference on Machine Learning}, pages
  674--683, 2015.

\bibitem[Delyon and Portier(2018)]{delyon+p:2018}
Bernard Delyon and Fran{\c{c}}ois Portier.
\newblock Asymptotic optimality of adaptive importance sampling.
\newblock In \emph{Proceedings of the 32nd International Conference on Neural
  Information Processing Systems}, pages 3138--3148. Curran Associates Inc.,
  2018.

\bibitem[Deng(2012)]{deng2012mnist}
Li~Deng.
\newblock The mnist database of handwritten digit images for machine learning
  research [best of the web].
\newblock \emph{IEEE Signal Processing Magazine}, 29\penalty0 (6):\penalty0
  141--142, 2012.

\bibitem[Duchi et~al.(2012)Duchi, Bartlett, and
  Wainwright]{duchi2012randomized}
John~C Duchi, Peter~L Bartlett, and Martin~J Wainwright.
\newblock Randomized smoothing for stochastic optimization.
\newblock \emph{SIAM Journal on Optimization}, 22\penalty0 (2):\penalty0
  674--701, 2012.

\bibitem[Fercoq and Richt{\'a}rik(2015)]{fercoq2015accelerated}
Olivier Fercoq and Peter Richt{\'a}rik.
\newblock Accelerated, parallel, and proximal coordinate descent.
\newblock \emph{SIAM Journal on Optimization}, 25\penalty0 (4):\penalty0
  1997--2023, 2015.

\bibitem[Fercoq et~al.(2014)Fercoq, Qu, Richt{\'a}rik, and
  Tak{\'a}{\v{c}}]{fercoq2014fast}
Olivier Fercoq, Zheng Qu, Peter Richt{\'a}rik, and Martin Tak{\'a}{\v{c}}.
\newblock Fast distributed coordinate descent for non-strongly convex losses.
\newblock In \emph{2014 IEEE International Workshop on Machine Learning for
  Signal Processing (MLSP)}, pages 1--6. IEEE, 2014.

\bibitem[Flaxman et~al.(2005)Flaxman, Kalai, and McMahan]{flaxman2005}
Abraham~D. Flaxman, Adam~Tauman Kalai, and H.~Brendan McMahan.
\newblock Online convex optimization in the bandit setting: Gradient descent
  without a gradient.
\newblock In \emph{Proceedings of the Sixteenth Annual ACM-SIAM Symposium on
  Discrete Algorithms}, SODA '05, page 385–394, USA, 2005. Society for
  Industrial and Applied Mathematics.
\newblock ISBN 0898715857.

\bibitem[Friedman et~al.(2010)Friedman, Hastie, and
  Tibshirani]{friedman2010regularization}
Jerome Friedman, Trevor Hastie, and Rob Tibshirani.
\newblock Regularization paths for generalized linear models via coordinate
  descent.
\newblock \emph{Journal of statistical software}, 33\penalty0 (1):\penalty0 1,
  2010.

\bibitem[Gadat et~al.(2018)Gadat, Panloup, Saadane,
  et~al.]{gadat2018stochastic}
S{\'e}bastien Gadat, Fabien Panloup, Sofiane Saadane, et~al.
\newblock Stochastic heavy ball.
\newblock \emph{Electronic Journal of Statistics}, 12\penalty0 (1):\penalty0
  461--529, 2018.

\bibitem[Gazagnadou et~al.(2019)Gazagnadou, Gower, and
  Salmon]{gazagnadou2019optimal}
Nidham Gazagnadou, Robert Gower, and Joseph Salmon.
\newblock Optimal mini-batch and step sizes for saga.
\newblock In \emph{International conference on machine learning}, pages
  2142--2150. PMLR, 2019.

\bibitem[Genevay et~al.(2016)Genevay, Cuturi, Peyr{\'e}, and
  Bach]{genevay2016stochastic}
Aude Genevay, Marco Cuturi, Gabriel Peyr{\'e}, and Francis Bach.
\newblock Stochastic optimization for large-scale optimal transport.
\newblock \emph{Advances in neural information processing systems}, 29, 2016.

\bibitem[Ghadimi and Lan(2013)]{ghadimi2013stochastic}
Saeed Ghadimi and Guanghui Lan.
\newblock Stochastic first-and zeroth-order methods for nonconvex stochastic
  programming.
\newblock \emph{SIAM Journal on Optimization}, 23\penalty0 (4):\penalty0
  2341--2368, 2013.

\bibitem[Glasmachers and Dogan(2013)]{glasmachers2013accelerated}
Tobias Glasmachers and Urun Dogan.
\newblock Accelerated coordinate descent with adaptive coordinate frequencies.
\newblock In \emph{Asian Conference on Machine Learning}, pages 72--86, 2013.

\bibitem[Gower et~al.(2021)Gower, Richt{\'a}rik, and Bach]{gower2021stochastic}
Robert~M Gower, Peter Richt{\'a}rik, and Francis Bach.
\newblock Stochastic quasi-gradient methods: Variance reduction via jacobian
  sketching.
\newblock \emph{Mathematical Programming}, 188\penalty0 (1):\penalty0 135--192,
  2021.

\bibitem[Gower et~al.(2019)Gower, Loizou, Qian, Sailanbayev, Shulgin, and
  Richt{\'a}rik]{gower2019sgd}
Robert~Mansel Gower, Nicolas Loizou, Xun Qian, Alibek Sailanbayev, Egor
  Shulgin, and Peter Richt{\'a}rik.
\newblock Sgd: General analysis and improved rates.
\newblock In \emph{International Conference on Machine Learning}, pages
  5200--5209. PMLR, 2019.

\bibitem[Hall and Heyde(1980)]{hall1980martingale}
P.~Hall and C.C. Heyde.
\newblock \emph{Martingale Limit Theory and Its Application}.
\newblock Probability and mathematical statistics. Academic Press, 1980.
\newblock ISBN 9781483240244.
\newblock URL \url{https://books.google.fr/books?id=wdLajgEACAAJ}.

\bibitem[Hanna et~al.(2019)Hanna, Niekum, and Stone]{hanna2019importance}
Josiah Hanna, Scott Niekum, and Peter Stone.
\newblock Importance sampling policy evaluation with an estimated behavior
  policy.
\newblock In \emph{International Conference on Machine Learning}, pages
  2605--2613. PMLR, 2019.

\bibitem[Kiefer et~al.(1952)Kiefer, Wolfowitz, et~al.]{kiefer1952stochastic}
Jack Kiefer, Jacob Wolfowitz, et~al.
\newblock Stochastic estimation of the maximum of a regression function.
\newblock \emph{The Annals of Mathematical Statistics}, 23\penalty0
  (3):\penalty0 462--466, 1952.

\bibitem[Krizhevsky et~al.(2009)Krizhevsky, Hinton,
  et~al.]{krizhevsky2009learning}
Alex Krizhevsky, Geoffrey Hinton, et~al.
\newblock Learning multiple layers of features from tiny images.
\newblock 2009.

\bibitem[Lee and Sidford(2013)]{lee2013efficient}
Yin~Tat Lee and Aaron Sidford.
\newblock Efficient accelerated coordinate descent methods and faster
  algorithms for solving linear systems.
\newblock In \emph{2013 IEEE 54th Annual Symposium on Foundations of Computer
  Science}, pages 147--156. IEEE, 2013.

\bibitem[Leluc and Portier(2020)]{leluc2020towards}
R{\'e}mi Leluc and Fran{\c{c}}ois Portier.
\newblock Towards asymptotic optimality with conditioned stochastic gradient
  descent.
\newblock \emph{arXiv preprint arXiv:2006.02745}, 2020.

\bibitem[Lian et~al.(2016)Lian, Zhang, Hsieh, Huang, and
  Liu]{lian2016comprehensive}
Xiangru Lian, Huan Zhang, Cho-Jui Hsieh, Yijun Huang, and Ji~Liu.
\newblock A comprehensive linear speedup analysis for asynchronous stochastic
  parallel optimization from zeroth-order to first-order.
\newblock In D.~Lee, M.~Sugiyama, U.~Luxburg, I.~Guyon, and R.~Garnett,
  editors, \emph{Advances in Neural Information Processing Systems}, volume~29.
  Curran Associates, Inc., 2016.

\bibitem[Liu et~al.(2020)Liu, Chen, Kailkhura, Zhang, Hero~III, and
  Varshney]{liu2020primer}
Sijia Liu, Pin-Yu Chen, Bhavya Kailkhura, Gaoyuan Zhang, Alfred~O Hero~III, and
  Pramod~K Varshney.
\newblock A primer on zeroth-order optimization in signal processing and
  machine learning: Principals, recent advances, and applications.
\newblock \emph{IEEE Signal Processing Magazine}, 37\penalty0 (5):\penalty0
  43--54, 2020.

\bibitem[Loshchilov et~al.(2011)Loshchilov, Schoenauer, and
  Sebag]{loshchilov2011adaptive}
Ilya Loshchilov, Marc Schoenauer, and Michele Sebag.
\newblock Adaptive coordinate descent.
\newblock In \emph{Proceedings of the 13th annual conference on Genetic and
  evolutionary computation}, pages 885--892, 2011.

\bibitem[Lu and Xiao(2015)]{lu2015complexity}
Zhaosong Lu and Lin Xiao.
\newblock On the complexity analysis of randomized block-coordinate descent
  methods.
\newblock \emph{Mathematical Programming}, 152\penalty0 (1-2):\penalty0
  615--642, 2015.

\bibitem[Moulines and Bach(2011)]{moulines2011non}
Eric Moulines and Francis~R Bach.
\newblock Non-asymptotic analysis of stochastic approximation algorithms for
  machine learning.
\newblock In \emph{Advances in Neural Information Processing Systems}, pages
  451--459, 2011.

\bibitem[Namkoong et~al.(2017)Namkoong, Sinha, Yadlowsky, and
  Duchi]{namkoong2017adaptive}
Hongseok Namkoong, Aman Sinha, Steve Yadlowsky, and John~C Duchi.
\newblock Adaptive sampling probabilities for non-smooth optimization.
\newblock In \emph{International Conference on Machine Learning}, pages
  2574--2583, 2017.

\bibitem[Necoara et~al.(2014)Necoara, Nesterov, and Glineur]{necoara2014random}
I~Necoara, Y~Nesterov, and F~Glineur.
\newblock A random coordinate descent method on large-scale optimization
  problems with linear constraints.
\newblock Technical report, Technical Report, 2014.

\bibitem[Needell et~al.(2014)Needell, Ward, and Srebro]{needell2014stochastic}
Deanna Needell, Rachel Ward, and Nati Srebro.
\newblock Stochastic gradient descent, weighted sampling, and the randomized
  kaczmarz algorithm.
\newblock In \emph{Advances in neural information processing systems}, pages
  1017--1025, 2014.

\bibitem[Nemirovski et~al.(2009)Nemirovski, Juditsky, Lan, and
  Shapiro]{nemirovski2009robust}
Arkadi Nemirovski, Anatoli Juditsky, Guanghui Lan, and Alexander Shapiro.
\newblock Robust stochastic approximation approach to stochastic programming.
\newblock \emph{SIAM Journal on optimization}, 19\penalty0 (4):\penalty0
  1574--1609, 2009.

\bibitem[Nemirovski and Yudin(1983)]{nemirovsky1983problem}
Arkadi~Semenovich Nemirovski and David~Borisovich Yudin.
\newblock Problem complexity and method efficiency in optimization.
\newblock 1983.

\bibitem[Nesterov(2012)]{nesterov2012efficiency}
Yu~Nesterov.
\newblock Efficiency of coordinate descent methods on huge-scale optimization
  problems.
\newblock \emph{SIAM Journal on Optimization}, 22\penalty0 (2):\penalty0
  341--362, 2012.

\bibitem[Nesterov and Spokoiny(2017)]{nesterov2017random}
Yurii Nesterov and Vladimir Spokoiny.
\newblock Random gradient-free minimization of convex functions.
\newblock \emph{Foundations of Computational Mathematics}, 17\penalty0
  (2):\penalty0 527--566, 2017.

\bibitem[Nguyen et~al.(2018)Nguyen, Nguyen, Dijk, Richt{\'a}rik, Scheinberg,
  and Tak{\'a}c]{nguyen2018sgd}
Lam Nguyen, Phuong~Ha Nguyen, Marten Dijk, Peter Richt{\'a}rik, Katya
  Scheinberg, and Martin Tak{\'a}c.
\newblock Sgd and hogwild! convergence without the bounded gradients
  assumption.
\newblock In \emph{International Conference on Machine Learning}, pages
  3750--3758. PMLR, 2018.

\bibitem[Nutini et~al.(2015)Nutini, Schmidt, Laradji, Friedlander, and
  Koepke]{nutini2015coordinate}
Julie Nutini, Mark Schmidt, Issam Laradji, Michael Friedlander, and Hoyt
  Koepke.
\newblock Coordinate descent converges faster with the gauss-southwell rule
  than random selection.
\newblock In \emph{International Conference on Machine Learning}, pages
  1632--1641, 2015.

\bibitem[Papa et~al.(2015)Papa, Bianchi, and
  Cl{\'e}men{\c{c}}on]{papa2015adaptive}
Guillaume Papa, Pascal Bianchi, and St{\'e}phan Cl{\'e}men{\c{c}}on.
\newblock Adaptive sampling for incremental optimization using stochastic
  gradient descent.
\newblock In \emph{International Conference on Algorithmic Learning Theory},
  pages 317--331. Springer, 2015.

\bibitem[Patel and Dieuleveut(2019)]{patel2019communication}
Kumar~Kshitij Patel and Aymeric Dieuleveut.
\newblock Communication trade-offs for local-sgd with large step size.
\newblock \emph{Advances In Neural Information Processing Systems 32 (Nips
  2019)}, 32\penalty0 (CONF), 2019.

\bibitem[Perekrestenko et~al.(2017)Perekrestenko, Cevher, and
  Jaggi]{perekrestenko2017faster}
Dmytro Perekrestenko, Volkan Cevher, and Martin Jaggi.
\newblock Faster coordinate descent via adaptive importance sampling.
\newblock In \emph{Artificial Intelligence and Statistics}, pages 869--877.
  PMLR, 2017.

\bibitem[Polyak(1963)]{polyak1963gradient}
Boris~Teodorovich Polyak.
\newblock Gradient methods for minimizing functionals.
\newblock \emph{Zhurnal Vychislitel'noi Matematiki i Matematicheskoi Fiziki},
  3\penalty0 (4):\penalty0 643--653, 1963.

\bibitem[Qu and Richt{\'a}rik(2016)]{qu2016coordinate}
Zheng Qu and Peter Richt{\'a}rik.
\newblock Coordinate descent with arbitrary sampling i: Algorithms and
  complexity.
\newblock \emph{Optimization Methods and Software}, 31\penalty0 (5):\penalty0
  829--857, 2016.

\bibitem[Qu et~al.(2015)Qu, Richt{\'a}rik, and Zhang]{qu2015quartz}
Zheng Qu, Peter Richt{\'a}rik, and Tong Zhang.
\newblock Quartz: Randomized dual coordinate ascent with arbitrary sampling.
\newblock In \emph{Advances in neural information processing systems}, pages
  865--873, 2015.

\bibitem[Richt{\'a}rik and Tak{\'a}{\v{c}}(2014)]{richtarik2014iteration}
Peter Richt{\'a}rik and Martin Tak{\'a}{\v{c}}.
\newblock Iteration complexity of randomized block-coordinate descent methods
  for minimizing a composite function.
\newblock \emph{Mathematical Programming}, 144\penalty0 (1-2):\penalty0 1--38,
  2014.

\bibitem[Richt{\'a}rik and
  Tak{\'a}{\v{c}}(2016{\natexlab{a}})]{richtarik2016optimal}
Peter Richt{\'a}rik and Martin Tak{\'a}{\v{c}}.
\newblock On optimal probabilities in stochastic coordinate descent methods.
\newblock \emph{Optimization Letters}, 10\penalty0 (6):\penalty0 1233--1243,
  2016{\natexlab{a}}.

\bibitem[Richt{\'a}rik and
  Tak{\'a}{\v{c}}(2016{\natexlab{b}})]{richtarik2016parallel}
Peter Richt{\'a}rik and Martin Tak{\'a}{\v{c}}.
\newblock Parallel coordinate descent methods for big data optimization.
\newblock \emph{Mathematical Programming}, 156\penalty0 (1-2):\penalty0
  433--484, 2016{\natexlab{b}}.

\bibitem[Robbins and Monro(1951)]{robbins1951stochastic}
Herbert Robbins and Sutton Monro.
\newblock A stochastic approximation method.
\newblock \emph{The annals of mathematical statistics}, pages 400--407, 1951.

\bibitem[Robbins and Siegmund(1971)]{robbins1971convergence}
Herbert Robbins and David Siegmund.
\newblock A convergence theorem for non negative almost supermartingales and
  some applications.
\newblock In \emph{Optimizing methods in statistics}, pages 233--257. Elsevier,
  1971.

\bibitem[Shalev-Shwartz and Zhang(2013)]{shalev2013stochastic}
Shai Shalev-Shwartz and Tong Zhang.
\newblock Stochastic dual coordinate ascent methods for regularized loss
  minimization.
\newblock \emph{Journal of Machine Learning Research}, 14\penalty0
  (Feb):\penalty0 567--599, 2013.

\bibitem[Shalev-Shwartz et~al.(2011)Shalev-Shwartz, Singer, Srebro, and
  Cotter]{shalev2011pegasos}
Shai Shalev-Shwartz, Yoram Singer, Nathan Srebro, and Andrew Cotter.
\newblock Pegasos: Primal estimated sub-gradient solver for svm.
\newblock \emph{Mathematical programming}, 127\penalty0 (1):\penalty0 3--30,
  2011.

\bibitem[Shamir(2017)]{shamir2017optimal}
Ohad Shamir.
\newblock An optimal algorithm for bandit and zero-order convex optimization
  with two-point feedback.
\newblock \emph{The Journal of Machine Learning Research}, 18\penalty0
  (1):\penalty0 1703--1713, 2017.

\bibitem[Wang et~al.(2018)Wang, Du, Balakrishnan, and Singh]{pmlr-v84-wang18e}
Yining Wang, Simon Du, Sivaraman Balakrishnan, and Aarti Singh.
\newblock Stochastic zeroth-order optimization in high dimensions.
\newblock In Amos Storkey and Fernando Perez-Cruz, editors, \emph{Proceedings
  of the Twenty-First International Conference on Artificial Intelligence and
  Statistics}, volume~84 of \emph{Proceedings of Machine Learning Research},
  pages 1356--1365. PMLR, 09--11 Apr 2018.
\newblock URL \url{https://proceedings.mlr.press/v84/wang18e.html}.

\bibitem[Wangni et~al.(2018)Wangni, Wang, Liu, and Zhang]{wangni2018gradient}
Jianqiao Wangni, Jialei Wang, Ji~Liu, and Tong Zhang.
\newblock Gradient sparsification for communication-efficient distributed
  optimization.
\newblock In \emph{Advances in Neural Information Processing Systems}, pages
  1299--1309, 2018.

\bibitem[Wu et~al.(2008)Wu, Lange, et~al.]{wu2008coordinate}
Tong~Tong Wu, Kenneth Lange, et~al.
\newblock Coordinate descent algorithms for lasso penalized regression.
\newblock \emph{Annals of Applied Statistics}, 2\penalty0 (1):\penalty0
  224--244, 2008.

\bibitem[Xiao et~al.(2017)Xiao, Rasul, and Vollgraf]{xiao2017fashion}
Han Xiao, Kashif Rasul, and Roland Vollgraf.
\newblock Fashion-mnist: a novel image dataset for benchmarking machine
  learning algorithms.
\newblock \emph{arXiv preprint arXiv:1708.07747}, 2017.

\end{thebibliography}

\end{document}